\pdfoutput=1
\documentclass{article}

\PassOptionsToPackage{numbers, compress}{natbib}
\usepackage[preprint]{neurips_2026}

\usepackage{amsmath,amsfonts,bm}

\def\eqref#1{equation~\ref{#1}}

\def\1{\bm{1}}

\DeclareMathAlphabet{\mathsfit}{\encodingdefault}{\sfdefault}{m}{sl}
\SetMathAlphabet{\mathsfit}{bold}{\encodingdefault}{\sfdefault}{bx}{n}

\def\gC{{\mathcal{C}}}
\def\gD{{\mathcal{D}}}

\def\gJ{{\mathcal{J}}}

\def\gS{{\mathcal{S}}}

\newcommand{\E}{\mathbb{E}}

\newcommand{\R}{\mathbb{R}}

\DeclareMathOperator{\Tr}{Tr}

\usepackage[utf8]{inputenc} 
\usepackage[T1]{fontenc}    
\usepackage{hyperref}       
\usepackage{url}            
\usepackage{booktabs}       
\usepackage{amsfonts}       
\usepackage{graphicx}       
\usepackage{amsthm}         
\usepackage{nicefrac}       
\usepackage{microtype}      
\usepackage{xcolor}         
\usepackage{thmtools,thm-restate}
\usepackage{amssymb} 
\usepackage{enumitem}
\usepackage{subcaption}
\usepackage[most]{tcolorbox}
\usepackage{wrapfig}

\newtheorem{theorem}{Theorem}[section]

\newtheorem{prop}[theorem]{Proposition}

\newtheorem{remark}[theorem]{Remark}

\newcommand{\norm}[1]{\left\lVert#1\right\rVert}

\newcommand{\N}{\mathbb{N}}

\def\RMS{\mathop{\rm RMS}\nolimits}

\def\id{\mathop{\rm id}\nolimits}
\def\lin{\mathop{\rm lin}\nolimits}
\def\block{\mathop{\rm B}\nolimits}

\def\attn{\mathop{\rm ATTN}\nolimits}

\def\Tr{\mathop{\rm tr}\nolimits}
\newcommand{\topl}{^{(l)}}

\usepackage{xspace}
\makeatletter
\DeclareRobustCommand\onedot{\futurelet\@let@token\@onedot}
\def\@onedot{\ifx\@let@token.\else.\null\fi\xspace}

\def\ie{{i.e}\onedot}

\makeatother

\renewcommand{\eqref}[1]{(\ref{#1})} 

\title{Sink vs. diagonal patterns as mechanisms\\ for attention switch and oversmoothing prevention}

\author{%
  Peter Súkeník\\
  Institute of Science and Technology (ISTA)\\
  Austria\\
  \texttt{peter.sukenik@ista.ac.at} \\
  \And
  Cristina López Amado\\
  Institute of Science and Technology (ISTA)\\
  Austria\\
  \texttt{cristina.lopezamado@ista.ac.at} \\
  \AND
  Christoph H. Lampert\thanks{Equal contribution} \\
  Institute of Science and Technology (ISTA)\\
  Austria\\
  \texttt{chl@ist.ac.at} \\
  \And
  Marco Mondelli\footnotemark[1] \\
  Institute of Science and Technology (ISTA)\\
  Austria\\
  \texttt{marco.mondelli@ista.ac.at}
  }

\begin{document}

\looseness=-1

\maketitle

\begin{abstract}
This paper studies the role of sinks and diagonal patterns as attention switch and anti-oversmoothing mechanisms. We analyze geometric conditions under which sinks can be represented, showing a necessary alignment between the embedding of the sink and all other embeddings. Next, we refine the current understanding of the role of sinks in oversmoothing prevention: we specify the conditions under which dense attention is expected to smooth more than sparse attention, and empirically verify that such conditions are often satisfied in practice. We further prove an equivalence between sinks and hard attention switch, in which the output of the attention is identically 0. Finally, we relax the hard attention switch by allowing token self-communication: we provide a quantitative comparison of the costs of representing sinks vs.\ diagonal patterns, showing why sinks are favored in pretrained transformers. The introduction and analysis of diagonal patterns and the generalization of the attention switch close the gap between what oversmoothing prevention requires and what sinks provide, while also establishing when and why attention layers act like MLPs if token communication is not necessary. 
\end{abstract}

\section{Introduction}

Softmax self-attention is the key 
component of transformers \cite{vaswani2017attention} that enables communication between tokens and, hence, makes it possible for non-recurrent architectures to solve sequential modeling tasks. 
Unlike in GNNs \cite{scarselli2008graph}, the full-graph
structure of attention allows for much denser and long-range communication patterns, a core feature that made transformers outperform vanilla RNN-like architectures \cite{bengio1970recurrent} in language modeling. Thus, it came as a surprise when \cite{xiao2024efficient} revealed that many tokens in many attention heads do not communicate at all. More precisely, \cite{xiao2024efficient} showed heads in which all tokens in a sequence idly attended to the first, beginning-of-sequence (BOS) token. 
This pattern was termed \textit{attention sinks} due to the 
resemblance of visualized attention maps (a vertical line) to sinks. Sinks have since been reproduced 
in different language settings \cite{gu2025when}, vision transformers \cite{darcet2024registers}, vision-language models \cite{luo2025sink}, multi-modal models \cite{kang2025see} and diffusion-based models \cite{wen2025vdit, rulli2025dlm}.

Sinks have received significant attention from both practitioners and theoreticians. Practitioners primarily studied 
how to take practical advantage of their existence or how to mitigate their emergence. \cite{xiao2024efficient} used sinks explicitly to generate long-form content without any fine-tuning and suggested to assign a dedicated sink token during pre-training. Similar concepts were discussed in the follow-ups \cite{li2024streamingdialogue, cai2024pyramidkv}. \cite{yu2024unveiling} presented mixed results where not all sinks improve the accuracy. On the negative side, sinks have been a target of blame in the context of resource effectiveness and interpretability \cite{kang2025see}, hallucination \cite{tu2025attentionreallocation, wang2025mirage}, quantization issues \cite{bondarenko2023quantizable}, and training stability \cite{qiu2025gated}. 
Thus, some works \cite{qiu2025gated, gu2025when} suggested methods to mitigate them, mostly by 
relaxing the softmax-normalization constraint. 

Theoreticians 
focused on explaining \textit{why} attention sinks emerge (see Section~\ref{sec:related} for an overview). 
Some works \cite{barbero2024glasses, barbero2025llms} claimed that sinks help transformers reduce oversmoothing and oversquashing. Others \cite{guo2024active, ranmilo2026necessary} connected sinks with the active-dormant mechanism (which we call attention switch), where the head is 
inactive on some tokens and outputs zero. 
However, existing oversmoothing results do not cover all architectural components of transformers, 
and they 
do not exclude other sparse attention patterns, such as diagonal attention. 
Similarly, \cite{guo2024active} neither proves that sinks are the only mechanism for 
attention switch, nor it 
argues why switching off attention 
means that it outputs 0, instead of e.g.\ acting as an MLP on idle tokens. Finally, it remains unclear how 
heads 
exhibit the sink from a geometric point of view. Our work closes the outlined gaps via the following contributions:

\vspace{-.5em}


\begin{itemize}[leftmargin=*]
    \item Section \ref{sec:how_sinks_represented} studies how sinks are represented. We show that, for the attention layer to be able to represent the sink, it 
    suffices that all token embeddings are aligned  (either positively or negatively) 
    with the sink embedding. We empirically verify this is the case in pretrained transformers. 
    \item Section \ref{sec:oversmoothing} provides the first quantitative analysis of oversmoothing considering all architectural components. 
    We analyze a single 
    attention step 
    and derive 
    conditions under which increasing attention density increases or decreases 
    token cosine similarity. 
    Our analysis shows that, in principle, skip connections and value-output transformations may counteract the mixing induced by attention, even under uniform attention. However, measurements on trained LLMs show that this 
    is rare in practice: 
    in most heads, token similarity increases as attention becomes more uniform. 
    \item While earlier work \cite{guo2024active, ranmilo2026necessary} defines an attention switch when the output on the residual stream is 0, we regard this as a hard attention switch and define a relaxation, dubbed the \textit{soft attention switch}, which only prohibits inter-token communication, but doesn't force zero output. 
    Soft switches allow for diagonal attention, which we empirically verify to be common in transformers.

    \item Section \ref{sec:sink_switch}
shows that, under mild conditions (which we verify in pretrained LLMs), 
a head is a hard attention switch if and only if it is a sink. This equivalence goes beyond the results of \cite{guo2024active}, which suggest sinks only as a possible mechanism for hard attention switch and show the emergence of sinks 
close to convergence 
(i.e., when attention already exhibits some amount of sink behavior). 

\item Section \ref{sec:sink_diag} focuses on the soft attention switch, and it establishes a trade-off between sink and diagonal attention in terms of the $\ell_2$ cost necessary to represent them. 
Because sinks are provably cheaper in a wide range of hyperparameters and distributions, this supports their larger relative prevalence w.r.t.\ diagonal attention in pretrained transformers.

\end{itemize}

\vspace{-.5em}

\section{Related work}\label{sec:related}

\vspace{-.5em}

\textbf{Attention sinks and diagonal patterns.} 
\cite{guo2024active} argues that sinks 
implement the active-dormant mechanism 
and justifies the emergence of sinks in a near-convergence regime on a bigram-backcopy task. \cite{zhang2025attention} claims that sinks play a more active role in computation by tagging future tokens with a shared embedding which can be exploited geometrically in subsequent layers, e.g., for averaging. Similarly, \cite{behrens2024understanding} shows that sinks are sometimes necessary for token counting and empirical histogram tasks. \cite{barbero2024glasses, barbero2025llms} study sinks from the representational viewpoint and show that they help prevent oversmoothing and oversquashing, while \cite{queipo2025attention} connects sinks with compression valleys arguing they are both manifestations of massive activations in the residual stream. \cite{ruscio2025you} proposes a geometric interpretation in which sinks act as reference-frame anchors, simplifying the geometry and subsequent attention allocation. \cite{qiu2026unified} 
argues that attention and residual sinks 
rescale and refine non-outlier features and attention weights. The 
contemporaneous
work \cite{ranmilo2026necessary} proves 
that sinks are necessary in 
an adaptation of the bigram-backcopy task from \cite{guo2024active}. 
We note that this paper 
relies on the following assumption: if a pair of token embeddings $x_t, x_q$ appears at positions $t, q$ in a sequence with non-vanishing probability, then the same pair must appear with non-vanishing probability precisely at positions $2,3$. 
Our theory in Section \ref{sec:sink_switch} 
explicitly takes into account positional embeddings, and it does not require such a strong assumption. Moreover, the main theorem of \cite{ranmilo2026necessary} can be recovered from the more general version of our Proposition \ref{thm:sink_switch_cppaste_generic} presented in Appendix~\ref{appx:proofs6} as we prove in Appendix~\ref{app:ranmilo}.
Finally, while diagonal patterns have been discussed e.g.\ in \cite{hankemeier2026stochastic, zhai2026exclusive}, to the best of our knowledge we are the first to explicitly consider them in the context of oversmoothing and sinks.

\vspace{-.5em}

\paragraph{Oversmoothing.} 

\cite{dong2021attention} proves that rank collapse occurs doubly exponentially with depth  in multi-head transformers without skip connections, layer normalization, or MLPs. Furthermore, while the role of skip connections and MLPs is considered, \cite{dong2021attention} does not analyze their joint effect. 
\cite{geshkovski2023emergence} shows that token representations converge to a clustered configuration for single-head attention 
without layer normalization or MLPs, assuming that query, key and value matrices are shared across layers. 
\cite{dovonon2025setting} analyzes a similar setting without such assumptions on learnable matrices, showing that oversmoothing can be avoided for some 
configurations. Oversmoothing at initialization is considered in 
\cite{noci2022signal,giorlandino2025two}. 
\cite{wu2024role} shows 
that 
oversmoothing occurs in single-head self-attention 
without skip connections regardless of the attention mask, and layer normalization may prevent it. 
\cite{karagodin2024clustering} extends the clustering results of \cite{geshkovski2023emergence} to causal masking, while still making strong assumptions on learnable matrices.
Related work studies the loss of token information as context length increases (rather than as depth increases): \cite{saada2024mind} proves rank collapse at initialization 
for single-head attention without skip connections, layer normalization, or MLPs;  \cite{barbero2024glasses,barbero2025llms} study oversquashing, i.e., the 
phenomenon in which 
the output is less sensitive to specific input tokens. 
Overall, 
prior work lacks a characterization of causal self-attention beyond initialization that jointly accounts for normalization and skip connections under weak assumptions on learnable matrices, and our work addresses this gap. 

\vspace{-.5em}

\section{Problem setup}


\vspace{-.5em}

\textbf{Model and task.}
We use the shorthand $[n]:=\{1,\dots,n\}.$ Let $\norm{\cdot}_2$ and $\norm{\cdot}_F$ be Euclidean and Frobenius norms. We denote the all-one vector of length $T$ by $\mathbf{1}\in\R^T$. For a matrix $M$, we denote its $i$-th row by $M_{i,:}$ and its $j$-th column by $M_{:,j}$ or $M_j$.
We consider 
single-head causal self-attention transformers $f_\Theta=\lin_{L+1}\circ \RMS \circ \block_L\circ \block_{L-1}\circ \dots \circ \block_1$. Each block is the composition of self-attention and MLP sub-blocks with residual connections: $\block_l=(\id+\lin_{l,2}\circ \sigma\circ\lin_{l,1}\circ \RMS)\circ (\id + \attn_l \circ \RMS),$ with $\RMS$ denoting the RMSNorm layer (in our theoretical results we assume it has non-learnable parameters), $\id$ the identity mapping, $\lin$ a fully-connected layer, $\sigma$ the ReLU layer and $\attn$ the single-head attention defined below. For an input embedding matrix $X^{(1)}=X_e+X_p\in\R^{d\times (T+1)},$ with $X_e, X_p$ token and positional embeddings, $X^{(l)}$ is the output of the first $l-1$ blocks and $Z^{(l)}=\RMS (X^{(l)}).$ An extension to relative positional embeddings is discussed in Section \ref{sec:sink_diag}. The attention layer $\attn_l$ computes the (row-wise) attention score matrix $A\topl_{ij}=\frac{\operatorname{exp}(S\topl_{ij})}{\sum_{k=0}^i\operatorname{exp}(S\topl_{ik})}$ if $j\le i$ and $0$ otherwise, 
where $S^{(l)}=(W^{(l)}_QZ^{(l)})^\top(W^{(l)}_KZ^{(l)})/\sqrt{d}$ is the score matrix and $W_Q,W_K\in\R^{d\times d}$ are query and key matrices. 
Then, $\attn_l(Z\topl)= W_{VO}^{(l)}Z^{(l)}(A^{(l)})^\top,$   
where $W_{VO}^{(l)}\in\R^{d\times d}$ is the multiplication of the value and the output matrix. When focusing on a single layer, we drop layer-wise indexing. 
We consider transformers pretrained by minimizing $\mathbb{E}_{\rho(X)}\mathcal{L}(f_\Theta(X), Y)+\lambda \norm{\Theta}^2_2$, where $\rho(X)$ is the training data empirical distribution and $\mathcal{L}$ is a loss (e.g., next-token cross-entropy). 

\textbf{Attention sink, diagonal attention and attention switch mechanisms.} 
%
A 
head 
exhibits \emph{attention sink}, if a fraction of 
context embeddings 
attends exclusively to BOS. 
This does not need to be global for the entire sequence. Instead, it suffices that 
individual tokens within the sequence attend to BOS. 
A 
head 
exhibits \emph{diagonal attention}, if a fraction of 
context embeddings 
attends exclusively to themselves. As for the sink, this property is also defined on the individual token level as opposed to the sequence level. 
A 
head 
exhibits a \emph{hard attention switch}, if its output on a fraction of context embeddings is zero. A head exhibits a \emph{soft attention switch}, if its output on a fraction of context embeddings is either zero or equal to the 
context embeddings themselves. 

\textbf{Oversmoothing measure.} We quantify token similarity 
via average cosine similarity. For a token matrix $X\in\R^{d\times T}$, the similarity between 
$X_i$ and $X_j$ is 
$\rho_{ij}(X):=\frac{(X_{i})^\top X_j}{\norm{X_i}_2\norm{X_j}_2}$, 
and the average cosine similarity of $X$ is $\rho(X)=\frac{2}{T(T-1)}\sum_{i=1}^T\sum_{j=i+1}^T\rho_{ij}$. Other notions of oversmoothing studied in the literature include input similarity \cite{wang2022antioversmoothing,wu2024role} and rank collapse \cite{shi2022revisiting,noci2022signal,giorlandino2025two} (see \cite{dovonon2025setting} for details).

\textbf{Empirical head pattern labeling.}
We empirically 
classify a head as a certain type when the relative attention mass on the corresponding pattern exceeds a threshold 
(e.g., a head is a \textit{sink} if BOS attention weights make up $\ge 40\%$ of the total mass). For dual-pattern heads, e.g. sink--lower-diagonal heads, we require both the joint pattern mass to exceed the threshold and the sub-patterns to be relatively balanced (e.g., the lower-diagonal mass is $\ge 10\%$ of the joint sink--lower-diagonal mass, and the joint mass is $\ge 60\%$ of the total mass).

\vspace{-.6em}

\section{Warm-up: 
How transformers represent sinks
}\label{sec:how_sinks_represented}

\vspace{-.6em}


While sinks frequently appear in many heads per block \cite{ivanitskiy2025motifs, gu2025when}, it is not a priori clear when the sink can be geometrically represented (
e.g., if two embedding vectors are opposite to each other, they cannot be both positively aligned with the sink's key). To address this, the result below provides conditions under which a group of tokens attends to the sink. The proof is deferred to Appendix \ref{appx:theorySec4}.

\begin{restatable}{prop}{representatbilityofsink}
\label{thm:representability_sink}
Let $Z\in \R^{d\times (T+1)}$ be a matrix of input embeddings $z_0, z_1, \dots, z_T$, and let $\mathcal{J}\subset [T]$ be a set of indices. Then, there exists $W \in \R^{d\times d}$ such that $z_j^T W z_0 > z_j^TWz_i$ for all $j\in \mathcal{J}$ and $i\in [T]$ if and only if the two sets of vectors $\{z_0-z_i, \, i \in [T]\}$ and $\{z_j, \, j \in \mathcal{J}\}$ lie in half-spaces.
\end{restatable}

The existence of $W$ s.t.\ $z_j^T W z_0 > z_j^TWz_i$ for all $j\in \mathcal J$ and $i\in [T]$ is equivalent to the head being able to represent a sink. Thus, Proposition \ref{thm:representability_sink} shows that, if the sink is representable, one can project the sink furthest away from 
other embeddings. Moreover, 
context embeddings attending to the sink must lie in a cone. A simple sufficient condition is $z_0^Tz_i \le 0$ for all $i\in[T]$, i.e., context embeddings are jointly rotated against the BOS embedding. 

In Figure~\ref{fig:bos_cos}, we verify empirically that this is indeed the case for all middle layers of 
three pretrained transformers (gpt2-small, Llama 3.1-8B and Gemma 7B). We measure the cosine similarity between BOS embedding and all other embeddings across multiple sequences and datasets. The maximum over all cosine similarities does not cross 0 in the inner layers of the models, allowing each attention head to produce any sink pattern. 


\begin{figure}
    \centering
    \includegraphics[width=0.9\linewidth]{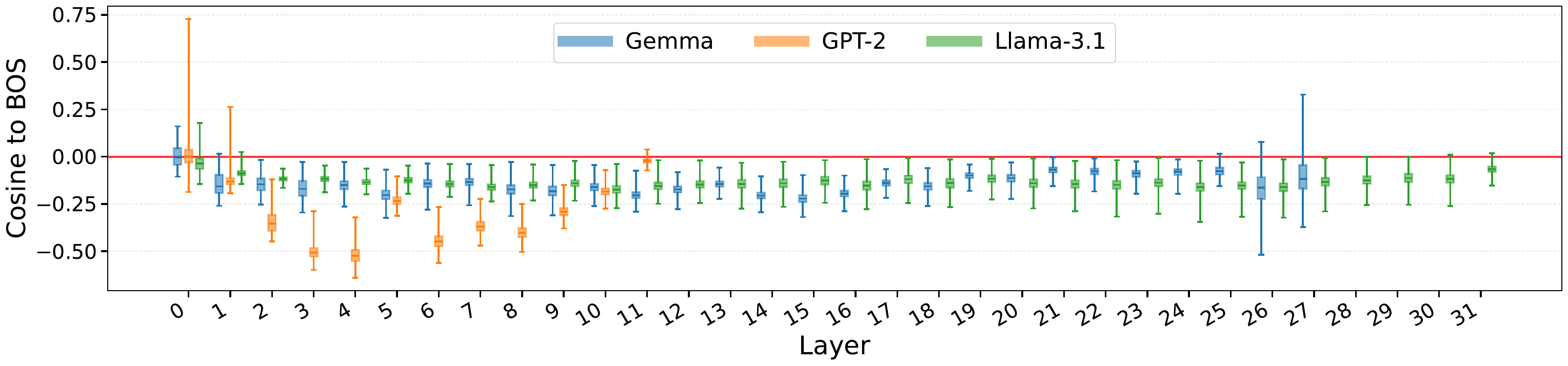}
    \caption{Cosine similarity between BOS and all other 256 token embeddings at the input to the attention layer, pooled across heads and 50 sequences across 4 language datasets (\texttt{TinyStories} \protect\cite{eldan2023tinystories}, \texttt{tinyshakespeare} \protect\cite{karpathy2015char-rnn}, \texttt{WikiText} \protect\cite{merity2017pointer}, and \texttt{CodeSearchNet-Python} \protect\cite{husain2019codesearchnet}). The whiskers of the box plot point to absolute minimum and maximum.}
    \vspace{-1em}
    \label{fig:bos_cos}
\end{figure}

    \vspace{-.5em}

\section{Sparse attention patterns as an oversmoothing prevention mechanism}\label{sec:oversmoothing}

    \vspace{-.5em}

In this section, we investigate whether sparse attention is necessary to avoid oversmoothing, 
or instead other transformer components can counteract the effect of near-uniform attention patterns. As shown in Appendix~\ref{app:existence}, in the presence of skip connections and layer normalization, for a fixed token matrix, one can choose $W_{VO}\topl$ to prevent oversmoothing 
even under uniform attention. 
However, the key question is whether a single choice of $W_{VO}\topl$ prevents oversmoothing over the full range of token matrices arising from the data distribution, rather than only for particular inputs. 


\textbf{Theoretical analysis.}
We study how token similarity changes after a single step of attention, under a distributional assumption on token representations. For simplicity, we omit the layer index $l$ ($X_i^{(l)}\equiv X_i$, $W_{VO}^{(l)}\equiv W$). 
As token representations tend to exhibit positive cosine similarity across layers 
(see Figure~\ref{fig:avgcossim-across-layers} in  Appendix~\ref{appx:exp-oversmoothing}), we model the normalized token representations as \(Z_i = \bar{z} + \epsilon_i\). Here, $\bar{z}$ is the average token representation and $\epsilon_i$ is such that 
$\E[\epsilon_i] = 0$. In addition, we assume that each unnormalized token can be written as \(X_i=\beta Z_i.\)
To study how the sparsity of the attention matrix affects oversmoothing, for  $\lambda\in[0,1]$ we define 
$Y(\lambda)= \beta Z+ WZ((1-\lambda)I+\lambda A_{u})^\top$, 
where $A_u$ is the $T\times T$ uniform causal attention matrix, i.e., $(A_u)_{ij}=1/i$ for $j\leq i$. The matrix $(1-\lambda)I+\lambda A_{u}$ interpolates between the identity ($\lambda=0$) and uniform attention ($\lambda=1$). Accordingly, $Y(\lambda)$ represents the token matrix after one attention update with skip connection, when the attention pattern varies 
from no mixing to uniform causal mixing, while keeping token representations and weight matrix fixed. This  isolates the effect of attention sparsity on token mixing. We deliberately choose diagonal attention rather than sink as the sparse endpoint: assuming a zero sink value, interpolating sinks with uniform attention only rescales the uniform attention contribution. Diagonal attention instead prevents inter-token communication while retaining nonzero updates.
We approximate the expected pairwise cosine similarity of the token matrix $Y(\lambda)$ by 
\begin{equation}\label{eq:approx-cos-sim}
     \E[\rho_{ij}(Y(\lambda))]\approx\hat\rho_{ij}(Y(\lambda)):=\frac{\E[Y_i(\lambda)^\top Y_j(\lambda)]}{\sqrt{\E[\norm{Y_i(\lambda)}^2]}\sqrt{\E[\norm{Y_j(\lambda)}^2]}}.
\end{equation}
This yields the approximation $\hat\rho(Y(\lambda))=\frac{2}{T(T-1)}\sum_{i=1}^T\sum_{j=i+1}^T\hat\rho_{ij}(Y(\lambda))$ for the average cosine similarity. This approximation is expected to be accurate in high dimensions when the bilinear terms $Y_i(\lambda)^\top Y_j(\lambda),\norm{Y_i(\lambda)}^2,\norm{Y_j(\lambda)}^2$ are sufficiently concentrated around their means. 

The result below provides sufficient conditions under which $\hat\rho(Y(\lambda))$ increases/decreases with $\lambda$, i.e., as the attention pattern becomes denser. The proof is deferred to Appendix~\ref{app:pf:thm:incr-lambda-skip}.
\begin{restatable}{theorem}{thIncrLambdaSkip}\label{thm:incr-lambda-skip}
    Assume 
    $\E[\epsilon_i\epsilon_i^\top]=\Sigma_V$ for all $i\in[T]$ and $\E[\epsilon_j\epsilon_i^\top]=\Sigma_C$ for all $i\neq j$. Let $B:=\Sigma_V-\Sigma_C$ and $C=\bar z\bar z^\top+\Sigma_C$. Then, for all $\lambda\in[0,1]$ and $i<j$, $k\in [T]$, we have
       \vspace{-.5em} \begin{equation}\label{eq:approx}
        \begin{split}
    \E[Y_i(\lambda)^\top Y_j(\lambda)]&=
    \beta^2\Tr(C)+2\beta\Tr(CW)+\Tr(CW^\top W)+\frac{\lambda}{j}\left(\beta\Tr(BW)+\Tr(BW^\top W)\right),\\
\E[\norm{Y_k(\lambda)}^2]&=\beta^2\Tr(B+C)+2\beta\Tr((B+C)W)+\Tr((B+C)W^\top W)
    \\&-2\left(1-\frac{1}{k}\right)(\beta\Tr(BW)+\Tr(BW^\top W))\lambda+\left(1-\frac{1}{k}\right)\Tr(BW^\top W)\lambda^2.
    \end{split}
    \end{equation}
     Thus, assuming $\E[Y_i(\lambda)^\top Y_j(\lambda)]>0$ $\forall i,j\in[T]$, we have that \emph{(i)} if $\beta\Tr(BW)> 0$, $\hat\rho(Y(\lambda))$ is strictly increasing in  $\lambda$; and \emph{(ii)} if $\beta\Tr(BW)+\Tr(BW^\top W)<0$, $\hat\rho(Y(\lambda))$ is strictly decreasing.
\end{restatable}
This result shows that denser attention does not, by itself, imply stronger oversmoothing, and the behavior depends on 
$W$. 
Compared to \cite{karagodin2024clustering}, which proves a clustering result for $W=I$, we show that 
$W$ can 
counteract the mixing effect of attention. Theorem \ref{thm:incr-lambda-skip} also sheds light on the role of skip connections in mitigating oversmoothing, which has been discussed in the literature \cite{dong2021attention,dovonon2025setting,noci2022signal}. 
This effect depends jointly on the skip-connection strength, the token distribution, and the transformation $W$: in the absence of skip connections ($\beta=0$), condition (ii) (i.e., $\beta\Tr(BW)+\Tr(BW^\top W)<0$ implying $\hat\rho(Y(\lambda))$  decreasing) never holds; furthermore, when $\Tr(BW)<0$, strengthening the skip connection decreases $\beta\Tr(BW)+\Tr(BW^\top W)$, which favors $\hat\rho(Y(\lambda))$ to be decreasing. A further interpretation of these conditions is in Appendix~\ref{app:pf:thm:incr-lambda-skip} (Remark~\ref{rem:trace-conditions-theorem}). 
We note that the requirement $\E[Y_i(\lambda)^\top Y_j(\lambda)]>0$ is not restrictive, since token representations typically have positive cosine similarity across layers. While generalizations to 
position-dependent scaling ($X_i=\beta_i Z_i$) and 
general covariances ($\E[\epsilon_j\epsilon_i^\top]=\Sigma_{ij}$) are possible, 
the experiments discussed below show that the approximation obtained under the current assumptions is already very close to the empirical quantity. 

While Theorem~\ref{thm:incr-lambda-skip} characterizes a single attention step, if its sufficient conditions hold consistently across layers of a transformer without MLP blocks, the layer-wise effects on the surrogate would accumulate, leading to more or less oversmoothing with depth as attention becomes denser.

\textbf{Experimental analysis.} 
We empirically 
assess whether the approximation in \eqref{eq:approx} 
matches the empirical average cosine similarity. Then, we study whether there are weight matrices $W$ that, even under uniform attention, counteract the mixing induced by attention, as part \textit{ii} of Theorem~\ref{thm:incr-lambda-skip} suggests. 
We consider 
four models (LLaMA3-8B, Gemma-7B, GPT2-XL, Mistral-7B), average across heads and over a batch of 50 sequences. We report results for the C4 dataset 
\cite{raffel2020exploring}, and show in Figure~\ref{fig:avgcossim-approx-datasets} in Appendix~\ref{appx:exp-oversmoothing} that the same conclusions hold for other datasets. 
For each attention head, $W$ is the composition of learnable normalization scaling, value matrix, and output projection matrix. 
    
\begin{figure}
    \centering
    \begin{subfigure}[b]{0.49\textwidth}
        \centering
        \includegraphics[width=\textwidth]{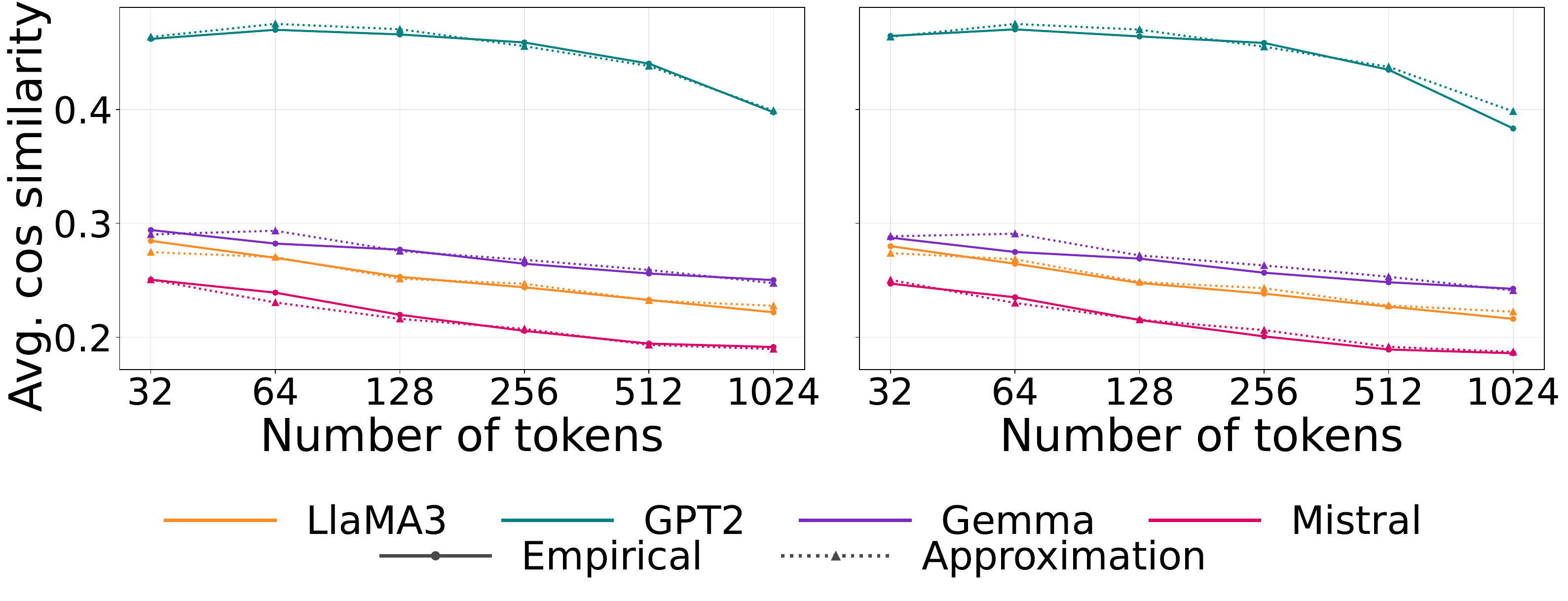}
    \end{subfigure}
    \hfill 
    \begin{subfigure}[b]{0.49\textwidth}
        \centering
        \includegraphics[width=\textwidth]{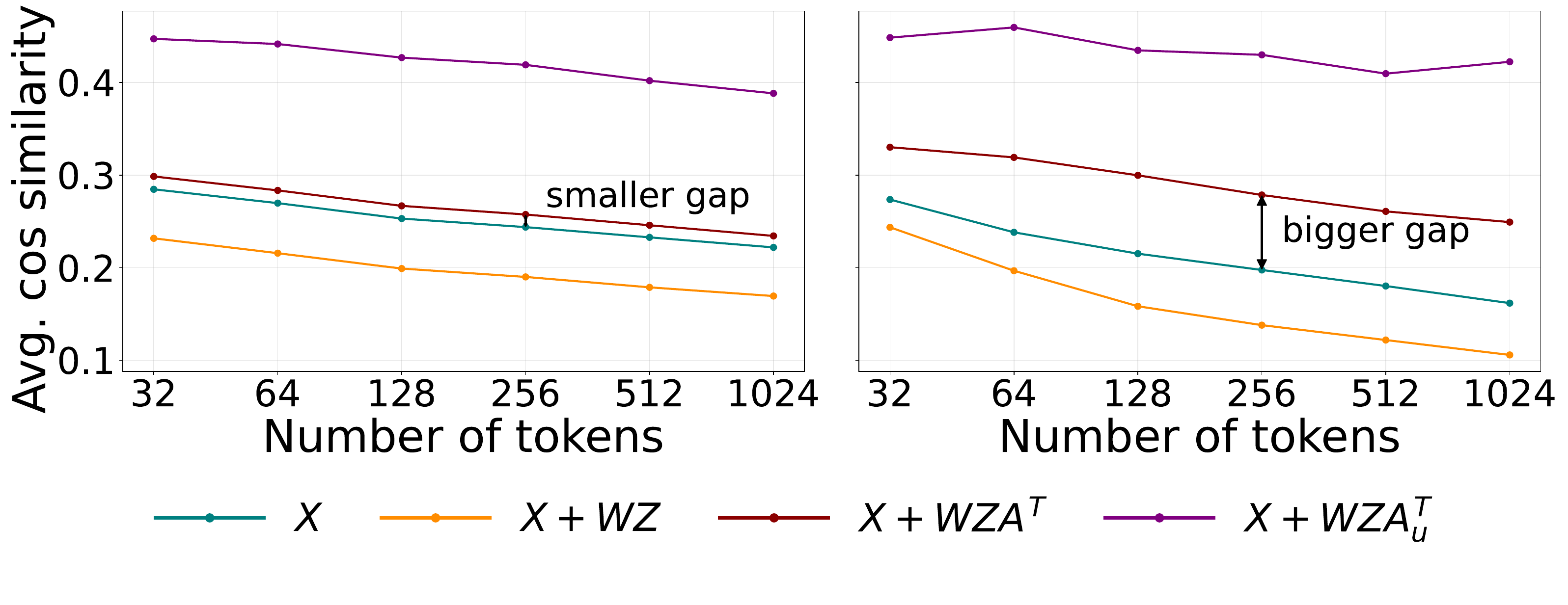}
    \end{subfigure}

    \caption{\textbf{Left:} empirical average cosine similarity (solid) and theoretical approximation \eqref{eq:approx} (dashed) for $X$ (first plot) and $X+WZ$ (second plot), across models and context lengths. \textbf{Right:} empirical average cosine similarity in LLaMA3-8B for different components of the attention step, averaging over all heads (first plot) and over heads with uniformity coefficient (see \eqref{eq:unifcoeff} in Appendix~\ref{appx:exp-oversmoothing}) larger than 0.6 (second plot). Cosine similarity generally increases as the attention matrix becomes denser.}
    \label{fig:avgcossim-theoryvsempirical}
    \vspace{-1em}
\end{figure}

The left two plots of Figure~\ref{fig:avgcossim-theoryvsempirical} show an excellent agreement between empirical average cosine similarity 
and the  approximation \eqref{eq:approx} 
. This holds both for $X$ and $X+WZ$, 
across several models 
and context lengths. 
In fact, the approximation remains accurate not only after averaging across heads, but also for most individual attention heads, see Figure \ref{fig:cossim-approx-ind-heads} in Appendix \ref{appx:exp-oversmoothing}. 

The right two plots of Figure~\ref{fig:avgcossim-theoryvsempirical} show the 
empirical average cosine similarity of token representations in LLaMA3-8B 
for different attention components. 
The attention contribution is rescaled to match the magnitude of the original token matrix, since in practice contributions of different heads are summed within each layer (see \eqref{eq:normATTN} in Appendix~\ref{appx:exp-oversmoothing}).
Comparing curves corresponding to $X$, $X+WZ$ and $X+WZA^\top$, we note that the application of $W$ decreases the average cosine similarity, which is then increased by the attention matrix. On average, the effect of the attention matrix is slightly stronger than that of $W$, i.e., the cosine similarity for $X+WZA^\top$ is larger than that for $X$. Comparing third and fourth plots, we note that the effect of the attention matrix is more pronounced (i.e., the gap between the curves for $X+WZA^\top$ and $X$ is bigger), when averaging only over heads closer to being uniform. Finally, the curve for $X+WZA_u^\top$ is obtained by replacing the attention matrix with a uniform one, which leads to a significant increase in cosine similarity. Figure~\ref{fig:avgcossim-empirical-models-datasets} in Appendix~\ref{appx:exp-oversmoothing} shows that this behavior is consistent across models and datasets, and Figure~\ref{fig:unif-vs-relChange} provides a finer head-wise view of this phenomenon.  
Figures~\ref{fig:trace-cond-theorem}-\ref{fig:trace-cond-theorem-2} 
show that \emph{(a)} the monotonicity of $\hat\rho(Y(\lambda))$ predicted by Theorem~\ref{thm:incr-lambda-skip} replicates well in practice, and that \emph{(b)} 
roughly half of the heads 
satisfy condition (i) 
(which leads to $\hat\rho(Y(\lambda))$ increasing), whereas 
few heads satisfy condition (ii) (which leads to  $\hat\rho(Y(\lambda))$ decreasing). This suggests that 
oversmoothing increases as attention becomes more uniform. 

\textbf{Summary.} Sparse attention is not the only way to prevent oversmoothing: skip connections and value-output transformations can counteract the effect of attention. In trained transformers, however, this rarely happens, and denser attention typically increases oversmoothing.



\vspace{-.5em}

\section{Equivalence between sink and hard attention switch
}\label{sec:sink_switch}

\vspace{-.5em}

In this section, we prove that in globally optimal transformers, 
the sink is the \textit{only} possible mechanism for hard attention switch. We generalize the input-output distribution of \cite{guo2024active}: for each context position $t$, let the set of possible input context embeddings $x_t$ be $\gS_t=\gD_t \cup \gC_t$, where embeddings in $\gD_t$ are dormant (i.e., the output of the head on them is 0) and embeddings in $\gC_t$ are copy-paste (i.e., their output on the head is any past context embedding $x_{s}$ except BOS). We assume $W_{VO}x_0=0$ as in \cite{guo2024active} which also supports it empirically 
(see Figure 1(b) therein). 
%
%
While \cite{guo2024active} claims that sink is \textit{the} solution for hard attention switch, there could, in theory, be other solutions that also implement it. In particular, a context embedding that has zero update in the attention layer could: \emph{(i)} attend to a convex combination of past context embeddings that sums up to zero, or \emph{(ii)} attend to itself while having zero value embedding. 
In the result below, we identify 
geometric conditions 
guaranteeing that neither of these two options is possible and, therefore, the only solution to switch off the head is to use the sink. 



\begin{restatable}{prop}{attnsinkattnswitchcppastegeneric}
\label{thm:sink_switch_cppaste_generic}
Consider the input-output distribution above. Assume that, for all input embedding sequences, 
either the sequence is full column-rank or it 
lies in a strict half-space. Assume also that the rank of the span of all output embeddings equals the rank of the span of all input embeddings. Then, in order to realize the hard attention switch, dormant tokens have to attend fully to BOS.
\end{restatable}

This result reveals 
why the hard attention switch has to be implemented via sinks: tokens cannot mix up value embeddings of past tokens or themselves to zero, because their own context embeddings might be used in the future. 
Sink is then the only option if either \emph{(i)} input embeddings lie in a strict half-space, or \emph{(ii)} both input embeddings and the value matrix on the span of input embeddings are full rank. Remarkably, both these conditions empirically hold in pretrained transformers: the first is verified in Figure \ref{fig:bos_cos}, and the second in Figure \ref{fig:value_ranks} in Appendix \ref{appx:exp-value_ranks}. 
We also note that a sufficient condition for the weight matrix to be full rank under the statement's conditions is that all embedding dimensions are useful (i.e., copy-pasted) for some future context tokens. Finally, while Proposition \ref{thm:sink_switch_cppaste_generic} shows that the hard attention switch must be implemented via the sink, the converse statement is obvious: in optimal transformers, if the head exhibits an attention sink (with zero BOS value embedding), then it must implement the attention switch. In fact, the only alternative is that all output embeddings of the head are zero, which contradicts the optimality of the transformer.
 
Proposition \ref{thm:sink_switch_cppaste_generic} follows from the more general Theorem \ref{thm:sink_switch_generic}  which does not require a precise copy-paste pattern, and it  directly implies that the sink is the only solution for the distribution of \cite{guo2024active}, see Corollary \ref{thm:copypaste_sink_concrete}. Both Theorem \ref{thm:sink_switch_generic} and Corollary \ref{thm:copypaste_sink_concrete} are in Appendix \ref{appx:proofs6}. There, we discuss the expected sample complexity due to the span equality condition (Remark \ref{rem:dataset_size}), and we show that heads implementing the behavior predicted by our theory are ubiquitous in pretrained LLMs (Remark \ref{rmk:exp}).

\textbf{Summary.}  Sinks are the only mechanism that can implement the hard attention switch under mild geometric assumptions (which empirically hold in pretrained transformers) and for a rather general input-output distribution (which includes the ones from  \cite{guo2024active, ranmilo2026necessary}).


\vspace{-.5em}

\section{Sink vs.\ diagonal patterns for soft attention switch
}
\label{sec:sink_diag}

\vspace{-.5em}


Section~\ref{sec:sink_switch} proves that sinks are necessary for \textit{hard} attention switch. 
However, Section~\ref{sec:oversmoothing} only gives evidence that \textit{dense} attention patterns (and, hence, significant inter-token communication) lead to oversmoothing. This leaves a gap between controlling oversmoothing and the current notion of attention switch. Note that oversmoothing would be prevented if the transformer implemented a ResNet by setting attention weights on the diagonal to $1$ and the rest to $0$. This diagonal pattern can be interpreted as an attention switch-off, since there is no inter-token communication.

Similarly to sinks, diagonal attention is common in pretrained LLMs. Across the sequences and datasets considered in Figure \ref{fig:bos_cos}, we find that $6\%, 0.6\%$ and $9.5\%$ of the  heads of gpt2-medium, Llama 3.1-8B and Gemma 7B have $\ge 40\%$ of the attention mass on the main diagonal (in Llama the number represents $>10\%$ of all non-sink heads). Thus, we introduce a notion of \emph{soft} attention switch allowing for both sink and diagonal attention, and we study the trade-off between these two patterns. 
The first aspect in this trade-off is the 
$\ell_2$ norm of the matrix  $W_{QK}\equiv W_Q^T W_K$ representing either sink or diagonal attention. For simplicity, 
assume the sequence is the identity 
$I_{T+1}$. To represent the sink, it suffices to set $W_{QK}=\kappa \mathbf{1}_{T+1}e_1^T$, with $e_1$ the first element of the canonical basis and $\kappa$ large enough to make attention weights approximately $1$ after the softmax. 
This matrix has a single non-zero singular value given by $\kappa \sqrt{T+1}$. To represent the diagonal attention, one needs to set $W_{QK}=\kappa I_{T+1}$, which has $T+1$ singular values equal to $\kappa$. As the regularization cost of $W_{QK}$ is its nuclear norm \cite{kobayashi2024weight}, the costs of sink and diagonal attention are respectively $\kappa \sqrt{T+1}$ and $\kappa (T+1)$. Hence, the sink is cheaper due to its low-rank structure. 


The second aspect of the trade-off is the re-use of resources in fully-connected layers. If the attention pattern is diagonal, the switched-off tokens receive a non-trivial update in the residual stream, which 
is missing in the sink case. There are then two options. \emph{(i)} The update is \textit{useful} for tokens in the following layers. Here the sink, which ignores the update, compensates via the MLP and incurs an extra cost. 
Hence, there is a non-trivial trade-off between the two patterns, depending on the token geometry and the role of the head. \emph{(ii)} The update is \emph{not useful} as it is inconsistent with messages sent to the following tokens. Here the diagonal attention is not desirable even when it is cheaper, as it harms the next-token prediction capabilities. Hence, there is a clear preference towards the sink. 


In the following, we focus on the case of \emph{useful} updates where there is no clear preference towards the sink. We
%
analyze 
two settings with different 
generality and compare the $l_2$ costs of a single-head attention block that represents the same task via either the diagonal or the sink attention patterns. 


\textbf{Large context lengths favor sinks for the backcopy task.}
We adapt the bigram-backcopy task to a single-block optimality analysis, keeping
the backcopy part of the distribution and dropping the bigram part. The context
embeddings are split into the dormant set $\mathcal D$ and the copy-paste set
$\mathcal C$, together with the BOS, indexed by $0$. Let the token directions
$\{h_0\}\cup\{h_i:i\in\mathcal C\cup\mathcal D\}$ be orthonormal, and let
$p_1,\ldots,p_T$ be unit positional vectors orthogonal to all token directions.
We write $E=\operatorname{span}\{p_1,\ldots,p_T\}$ and $r=\dim E$, so that
$r+|\mathcal C|+|\mathcal D|+1\le d$, while $T$ may exceed $d$.
The input embeddings are $x_{00}=\sqrt d\,h_0$ and
$x_{ti}=\sqrt{d/2}(h_i+p_t)$ for $t\ge1$. The first token after BOS is
dormant; subsequent tokens can belong to either group, and every sequence
of this form has positive probability. The block leaves BOS unchanged,
maps a dormant input $x_{ti}$ to $2x_{ti}$, and maps a copy-paste input
$x_{ti}$ to $x_{ti}+x_{t-1,j}$, where $j$ is the preceding token.

Unlike orthogonal token embeddings, the positional vectors can share directions.
We assume that they remain spread across their span: the matrix $P$ with rows
$p_t^\top$ satisfies $\|P\|_{\mathrm{op}}^2\le BT/r$ for some $B\ge1$.
We also assume that one linear map $R:E\to E$ can distinguish each
immediate predecessor. More precisely, for some $\gamma\in(0,1]$ and $L\ge1$,
$p_{t-1}^\top Rp_t=1, \,\, p_s^\top Rp_t\le1-\gamma, \,\, (2\le t\le T,\ 1\le s\le t,\ s\ne t-1)$ while also $p_1^\top Rp_1\le1,\,\,\|R\|_*\le Lr.$
Thus, a shared query-key map can select the previous position without
requiring orthogonality of position embeddings. Harmonic
positional embeddings, closely related to RoPEs, satisfying these conditions are constructed in
Appendix~\ref{app:backcopy-frame}.

Both attention patterns send copy-paste tokens to their immediate predecessor.
They differ on dormant tokens: a sink sends them to BOS, whereas diagonal
attention sends them to themselves. As above, we use the hard-attention
approximation with logit margin $\kappa>0$ and the standard $1/\sqrt d$
score normalization. Cost is measured as sum of squared Frobenius norms of matrix parameters.

\begin{restatable}{theorem}{sinkdiagcppaste}
\label{thm:backcopy-frame}
Let $T\ge3$, and let a single-head transformer block have an MLP hidden
width of at least $r+|\mathcal D|+1$. Under the distribution and positional
assumptions above, put $C=|\mathcal C|$, $D=|\mathcal D|$, and
$m=\min\{C,D\}$. There is a sink representation of cost at most
\begin{equation}
 U_{\mathrm{sink}}
 =\frac{2\kappa}{\sqrt d}
   \left[\frac{2Lr}{\gamma}
     +\sqrt{2D}\left(1+\frac1\gamma\right)\right]
   +(5+25C)(r+D)+27C,
\label{eq:backcopy-upper-main}
\end{equation}
whereas every diagonal representation has cost at least
\begin{equation}
 L_{\mathrm{diag}}
 =2\sqrt{\frac23}\,\kappa
   \sqrt{\frac{mr(T-1)(T-2)}{dB}}.
\label{eq:backcopy-lower-main}
\end{equation}
If $U_{\mathrm{sink}}<L_{\mathrm{diag}}$, it is cheaper to
represent this mapping with a sink than with diagonal attention.
\end{restatable}

This result, proved in Appendix~\ref{app:backcopy-frame}, gives an analytic
condition for a sink to incur lower cost. For fixed token-set sizes, the
value/output projections and MLP in our sink construction have cost $\mathcal{O}(r)$,
and the total sink upper bound is $\mathcal{O}(r+\kappa Lr/(\sqrt d\,\gamma))$.
The diagonal lower bound scales as $\kappa T\sqrt{mr/(dB)}$.
Thus, the comparison depends on both the positional rank and the separation
achieved by the shared map; their dependence on $T$ must be accounted for.
Appendix~\ref{app:backcopy-logarithmic-example} shows that harmonic frames
can be chosen with $r,d=\Theta(\log T)$ and bounded
$B,L,\gamma^{-1}$. For fixed $\kappa$, this gives a sink cost of
$\mathcal{O}(\log T)$ and an optimal diagonal cost of $\Theta(T)$.
There are also fixed-dimensional families with sink cost $\mathcal{O}(T^\varepsilon)$
and optimal diagonal cost $\Theta(T)$ for every prescribed
$\varepsilon>0$, with the fixed dimension chosen depending on
$\varepsilon$; see Appendix~\ref{app:backcopy-fixed-d-example}.
Finally, at fixed positional geometry and token-set sizes, increasing $d$
makes the sufficient inequality harder to satisfy: it reduces the
attention-cost terms by $1/\sqrt d$ while leaving the value/output and
MLP costs unchanged.

\begin{figure}
\centering
\includegraphics[width=0.8\textwidth]{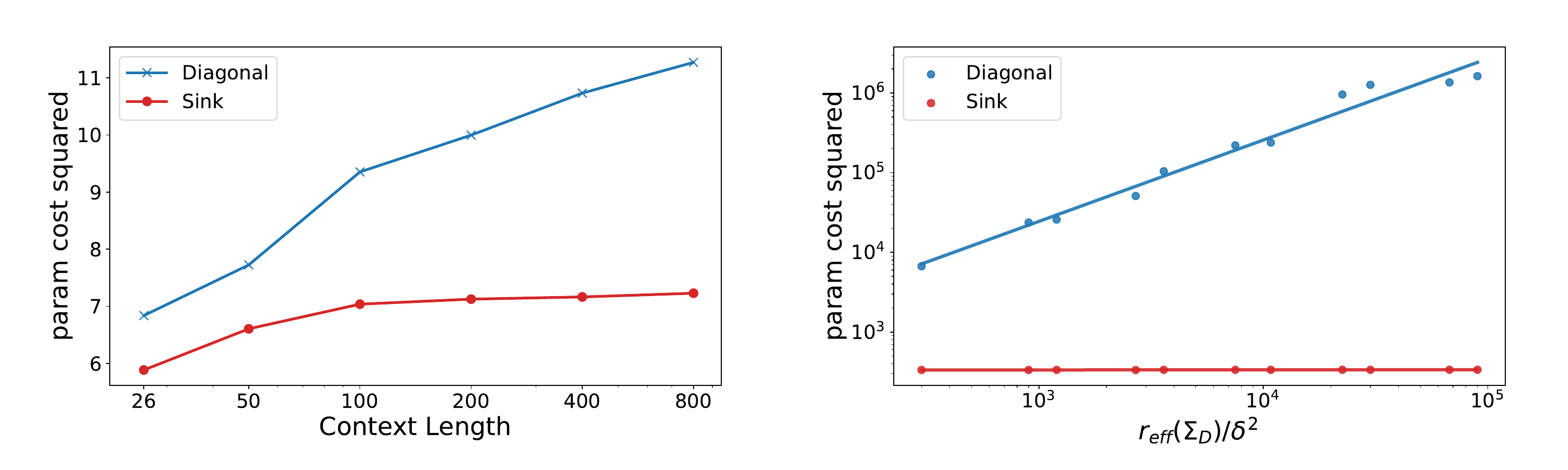}
\vspace{-.5em}
\caption{Frobenius squared cost of a single transformer block if the attention pattern is either sink or diagonal. \textbf{Left:} Backcopy task of Theorem \ref{thm:backcopy-frame} and cost as a function of the context length. \textbf{Right:} Generic task of Theorem \ref{thm:sink_diag_generic} and cost as a function of $r_{\text{eff}}(\Sigma_D)/\delta^2.$}
\label{fig:sink_diag_cppaste}
\vspace{-1em}
\end{figure}


Figure~\ref{fig:sink_diag_cppaste} compares the total cost of a single transformer block exhibiting either diagonal or sink attention patterns. The positional embedding rank first tracks dimension until $d=100,$ then remains fixed ($D=C=5)$. The estimated polynomial growth rates are: 0.85 for sink solutions with $T < 100$, 0.09 for sink solutions with $T > 100$. For diagonal, it is 1.87 for $T < 100$ and 0.94 for $T > 100$. These results are agree with the rates presented in Theorem~\ref{thm:backcopy-frame}. The slight deviations from precise numbers (0.94 vs 1, 0.09 vs. 0 and 0.85 vs. 1) are best explained by finiteness of the parameters and thus the asymptotics being only approximate, as well as experimental confounds (training with a fixed number of steps and learning rate might affect different $T$ values differently).

\textbf{Trade-off 
for a generic input-output mapping.}
We next consider a generalized distribution, leading to a more nuanced trade-off between sink and diagonal attention. Let $C\ge 1$ be the number of copy-paste groups. Let
$\gS$ denote the set of $d$-dimensional vectors $\{b_{\mathrm{BOS}}, \bar d_1,\dots,\bar d_C, c_1,\dots,c_C\}$, where 
$b_{\mathrm{BOS}}$ is the BOS token, $\bar d_1,\dots,\bar d_C$ are the dormant group directions, and $c_1,\dots,c_C$ the copy-paste context embeddings. Vectors in $\mathcal S$ have the same length (i.e., $\|b_{\mathrm{BOS}}\|_2=\|\bar d_c\|_2=\|c_c\|_2=\sqrt d$ for every $c$) and are near-orthogonal (i.e., $|\langle x,y\rangle|\le \phi d$, for any two distinct vectors $x, y$ in $\mathcal{S}$). Each dormant context in group $c$ has the form
$d_{ci}=\lambda \bar d_c+\eta_{ci}$, where $\eta_{ci}\perp \mathcal S$, 
$\|\eta_{ci}\|_2=\delta\sqrt d$ for $\delta\in[0,1)$ and
$\lambda:=\sqrt{1-\delta^2}$. 
For each group $c$, let $\mathcal P_c$ be the set of unordered pairs $\{i,j\}$ s.t.\ both relative orders of $d_{ci}$ and $d_{cj}$ occur in the data. Denote
$\Sigma_D:=\sum_{c=1}^C\sum_{\{i,j\}\in\mathcal P_c}(d_{ci}-d_{cj})(d_{ci}-d_{cj})^\top$, $r_{\mathrm{eff}}(\Sigma_D):=\frac{\operatorname{tr}(\Sigma_D)}{\|\Sigma_D\|_2}$ and $r_\eta$ as the rank of all noise vectors.

The output of the block on the BOS token is itself; 
the output on every dormant context $d_{ci}$ is $(1+\phi_{ci})d_{ci}$, with $\phi_{ci}\in[1-\epsilon, 1+\epsilon]$ and $\epsilon\in[0,1)$ a tolerance parameter; the output on a copy-paste context $c_c$ is $c_c$ if there is no dormant context from group $c$ before it in the sequence, and otherwise
$c_c+\phi_c\frac{1}{|R_c|}\sum_{\tilde c\in R_c} \tilde c$, where $R_c$ is the set of preceding dormant contexts from group $c$ and $\phi_{c}\in[1-\epsilon, 1+\epsilon]$. Note that $\phi_c, \phi_{ci}$ 
might 
depend on the particular sequence and position of the copy-paste context.
We say that the block implements a sink if: \emph{(i)} the BOS embedding attends only to itself; \emph{(ii)} every dormant embedding attends to $b_{\mathrm{BOS}}$; and \emph{(iii)} every copy-paste embedding attends uniformly to the preceding dormant embedding from its own group, if any such embedding is present, and otherwise attends to $b_{\mathrm{BOS}}$. We say that the block implements a diagonal attention if: \emph{(i)} the BOS embedding attends only to itself; \emph{(ii)} every dormant embedding attends to itself; and \emph{(iii)} every copy-paste embedding attends uniformly to the preceding dormant embedding from its own group, if any such embedding is present, and otherwise attends to itself.
Finally, we assume that, 
for every copy-paste token $c_c$ and every dormant key $d_{ki}$, the score depends only on the centroid component of the key: since
$d_{ki}=\lambda\bar d_k+\eta_{ki}$ and $\eta_{ki}\perp \mathcal S$, we write this as $c_c^\top W_{QK}d_{ki}=\lambda\,c_c^\top W_{QK}\bar d_k.$ Thus, the query-key mechanism only tests whether a dormant key has the correct centroid component.

In summary, there are $C$ context collectors that check whether past contexts have the correct semantic meaning (the average direction $\bar d_c$), and then they collect both this meaning as well as the nuanced low-level meaning ($\eta_{ci}$). As an example, tokens that are ends of some words can try to collect the previous tokens of that word by querying the ``high-level'' word membership feature and then copy-pasting the nuanced content of the tokens belonging to the same word. Note that heads of this type occur in pretrained transformers, as displayed in Figure~\ref{fig:sink_diag_attention_maps} in Appendix \ref{app:patterns}.

\begin{restatable}{theorem}{sinkdiaggeneric}\label{thm:sink_diag_generic}
Let $2C\phi\le \frac14, \phi\le \frac1{20}, \delta\le \frac12$ and $\frac1{20}\le \epsilon\le \frac14.$
If the following holds: \begin{equation}\label{eq:sink_diag_generic}
12\,\frac{\kappa C}{\sqrt d}+5(r_\eta+C) < \frac{\kappa}{\delta^2\sqrt d}\,r_{\mathrm{eff}}(\Sigma_D),
\end{equation}
then it is cheaper to represent this mapping 
with a sink rather than a diagonal attention. 
\end{restatable}

This result (proved in Appendix~\ref{app:proof-sinkdiaggeneric}) provides a sufficient condition for a sink to incur lower cost than a diagonal attention, when the input-output distribution is described via a general copy-paste mechanism. The idea is to show that the RHS of \eqref{eq:sink_diag_generic} lower bounds the cost of diagonal attention, and the LHS upper bounds the sink cost. We now provide an interpretation of these two bounds. \emph{(i)} As for the diagonal lower bound, when $d$ grows, the relative cost of $W_{QK}$ decreases and so does the RHS of \eqref{eq:sink_diag_generic}. 
The crucial term is $r_{\mathrm{eff}}(\Sigma_D)/\delta^2$: if the nuanced signal is small compared to the high-level content, then it is expensive for the diagonal pattern to separate the nuanced part, so dormant embeddings attend to themselves instead of other same-group embeddings. 
Similarly, the higher the effective rank $r_{\mathrm{eff}}(\Sigma_D)$ of the nuanced part is, the more dimensions $W_{QK}$ has to  blow up to guarantee the separation of the same-group dormant tokens.  
\emph{(ii)} As for the sink upper bound, 
the term $\kappa C / \sqrt{d}$ in the LHS of \eqref{eq:sink_diag_generic} ensures that copy-paste tokens attend according to the task description. 
The term $r_\eta+C$ upper bounds the number of dimensions that $W_{VO}$ and the MLP need to operate on for dormant embeddings. This is larger than the equivalent cost of a diagonal solution, since copy-paste and dormant embeddings do not share the resources to perform the same update. 


To obtain conditions under which diagonal attention is cheaper than sink, one needs to upper bound the cost of diagonal attention and lower bound the sink cost. Theorem \ref{thm:umbrella_bounds_relaxed_simple} in Appendix~\ref{app:proof-sinkdiaggeneric} contains \emph{(i)} an upper bound on diagonal attention of order $\frac{\kappa C}{\sqrt d}+\frac{\kappa}{\sqrt d\delta^2}r_\eta+r_\eta+C$ (with smaller multiplicative constant on $r_\eta+C$ compared to the sink upper bound in the LHS of \eqref{eq:sink_diag_generic}), and \emph{(ii)} a lower bound on sink of order $\kappa C / \sqrt{d}$. We conjecture that a lower bound on sink of order $r_\eta+C$ holds as well. 

Figure~\ref{fig:sink_diag_cppaste} compares parameter costs of a single block trained to implement the task of Theorem \ref{thm:sink_diag_generic} with $\phi=0, C=3, d=100, T=50$, as a function of the crucial quantity $r_{\mathrm{eff}}(\Sigma_D)/\delta^2$ computed over a sweep of $\delta$ values and number of dormant tokens per group. 
The plot shows that, while the sink cost is independent of $r_{\mathrm{eff}}(\Sigma_D)/\delta^2$, the cost of the diagonal solution grows log-linearly with it.



\paragraph{Extension to relative positional embeddings (RoPE).} Our analysis is formulated for absolute positional embeddings, but a similar geometric perspective can be developed for RoPE. Under RoPE, queries and keys are rotated according to their positions, $q_i=\Phi_iW_Qx_i$ and $k_j=\Phi_jW_Kx_j$, yielding attention scores $S_{ij}=\frac{1}{\sqrt{d}}x_i^\top W_Q^\top\Phi_i^\top\Phi_jW_Kx_j$. For a fixed context length, we can interpret this through a fixed position-indexed geometry by defining $z_{a,i}=\Phi_i x_a$ for an embedding $x_a$ occurring at position $i$. Their inner products satisfy $z_{a,i}^\top z_{b,j}=x_a^\top\Phi_i^\top\Phi_jx_b$, so similarity depends on both the embeddings and their positions.

Importantly, this geometry can be related directly to the query-key computation. Defining $\widetilde W_{Q,i}=\Phi_iW_Q\Phi_i^\top$ and $\widetilde W_{K,i}=\Phi_iW_K\Phi_i^\top$, we obtain $q_i=\widetilde W_{Q,i}z_{a,i}$ and $k_j=\widetilde W_{K,j}z_{b,j}$. Moreover, orthogonality of the rotations implies $\|\widetilde W_{Q,i}\|_F=\|W_Q\|_F$ and $\|\widetilde W_{K,i}\|_F=\|W_K\|_F$. Thus, the RoPE scores can be expressed exactly using position-indexed embeddings and transformed query-key matrices, without changing their individual Frobenius costs. However, unlike in the original setting, these transformed matrices depend on position and must all originate from the same shared $W_Q,W_K$.

This motivates defining dormant groups using the position-indexed query-key geometry: the same embedding at different positions may belong to different groups, while different embeddings may belong to the same group if their position-indexed representations are sufficiently close. The query-key and residual-stream parts of our analysis are then affected differently. The sink construction becomes a problem of satisfying attention-score constraints over position-indexed dormant groups. Its cost must additionally account for the requirement that all position-dependent query-key matrices arise from a single shared pair $W_Q,W_K$. The value-output and MLP constructions, however, remain unchanged since they do not need to distinguish individual dormant groups.

Similarly, the diagonal lower bound depends only on query-key score constraints, which can be aggregated into a position-indexed generalization of $\Sigma_D$. We discuss these extensions in Appendix~\ref{app:rope}. A fully rigorous RoPE cost comparison, including the shared-parameter constraints and generalized effective-rank bounds, remains an interesting direction for future work.

\textbf{Summary.} Diagonal attention switches off token communication, thereby helping to prevent oversmoothing. 
However, its representation is often more expensive than that of the sink, which can explain the prevalence of the latter in pretrained transformers. 

\section{Concluding remarks}
In this paper, we study attention sinks and compare them with alternative ways to shut down token communication. We close several gaps in the literature by showing that: \emph{(i)} in order to represent sinks, transformers rotate the embeddings against that of the sink; \emph{(ii)} sparse attention is needed to prevent oversmoothing despite existing theoretical mechanisms that could avoid it; \emph{(iii)} sinks are theoretically and empirically necessary to switch attention off fully; \emph{(iv)} there are other mechanisms to prevent token communication beyond sinks, such as diagonal attention; and, finally, \emph{(v)} sinks are prevalent in practice since they have provably lower cost than diagonal attention for a wide range of data distributions. 

Several limitations motivate open questions for future work: In which precise module is the alignment between sink and other embeddings established? Can we characterize the depth-wise progression of oversmoothing in multi-head transformers and under alternative oversmoothing measures? How do the relative positional embeddings change the geometry of heads? How to formally characterize the usefulness of diagonal attention from the next-token prediction viewpoint? 



\begin{ack}
M. M.\ and P. S.\ are funded by the European Union (ERC, INF$^2$, project number 101161364). Views and opinions expressed are however
those of the author(s) only and do not necessarily reflect those of the European Union or the
European Research Council Executive Agency. Neither the European Union nor the granting
authority can be held responsible for them.
This research was supported by the Scientific Service Units (SSU) of ISTA through resources provided by Scientific Computing (SciComp).
\end{ack}



\bibliographystyle{plain}
\bibliography{neurips_2026}


\clearpage
\appendix
\section{Proofs for Section \ref{sec:how_sinks_represented}}\label{appx:theorySec4}

\representatbilityofsink*

\begin{proof}
We first rewrite the conditions:
\begin{align*}
    &z_j^T W z_0 > z_j^TWz_i \\
  \Longleftrightarrow\qquad & z_j^T W (z_0-z_i) > 0 \\ 
  \Longleftrightarrow\qquad &  \Tr(z_j^T W (z_0-z_i)) > 0 \\
  \Longleftrightarrow\qquad &  \Tr(W(z_0-z_i)z_j^T) > 0 \\
  \Longleftrightarrow\qquad &  \langle W, z_j(z_0-z_i)^T\rangle_F > 0.
\end{align*}
Thus, we have a system of strict linear inequalities. Using Farkas' lemma, we know that the existence of the solution ($W$) is equivalent to the non-existence of a non-zero vector $\alpha=(\alpha_{ij})_{i, j}$ such that $\alpha_{ij} \ge 0$ and $\sum_{i\in[T],j\in \mathcal{J}} \alpha_{ij} z_j(z_0-z_i)^T =0.$ Geometrically, if we treat the matrices $z_j(z_0-z_i)^T$ as vectors, this is equivalent to these vectors lying in a half space. We will show that a sufficient and necessary condition for this is what is claimed in the theorem statement.

\paragraph{The sufficiency.} Assume there exists $\alpha=(\alpha_{ij})_{i, j}$ such that $\alpha_{ij} \ge 0$ and $\sum_{i\in[T],j\in \mathcal{J}} \alpha_{ij} z_j(z_0-z_i)^T =0.$ And assume by contradiction that both $(z_0-z_i)$ and $z_j$ lie in a half-space. Then there exist vectors $u, v$ such that $u^Tz_j>0 \,\, \forall j\in \gJ$ and $(z_0-z_i)^Tv>0 \,\, \forall i\in [T].$ Multiply the involved equation from left by $u^T$ and from right by $v.$ We get: $\sum_{i\in[T],j\in \mathcal{J}} \alpha_{ij} u^Tz_j(z_0-z_i)^Tv =0.$ Because all the involved terms except $\alpha_{ij}$ are strictly positive, all $\alpha_{ij}$ must be zero, which is a contradiction. 

\paragraph{The necessity.} Assume now that the set of vectors $z_j(z_0-z_i)^T$ does lie in a half-space and assume by contradiction that either $(z_0-z_i)$ or $z_j^T$ do not lie in a half-space. W.l.o.g. let that be the $z_j$'s. Then, there exists a non-zero, non-negative vector $\alpha_j$ such that $\sum_{j}\alpha_{j} z_j = 0.$ Hence,  we can build $\alpha$ by simply repeating this sub-vector along the $i$ index and readily get that $\sum_{i\in[T],j\in \mathcal{J}} \alpha_{ij} z_j(z_0-z_i)^T=0$ by considering the inner sum over $j.$ 

\end{proof}

\section{Proofs for Section \ref{sec:oversmoothing}}\label{appx:theory-oversmoothing}

\subsection{Existence of $W_{VO}\topl$ preventing oversmoothing}\label{app:existence}

We provide examples showing that, even under uniform attention, when skip connections and layer normalization are present, for every full-column-rank token matrix, there exist configurations of matrices $W_{VO}\topl$ that can counteract the homogenizing effect of attention. To that end, we study the evolution as the number of layers increases of
\begin{equation}\label{eq:att-prelaynorm-skip}
    X^{(l+1)}= X^{(l)}+W_{VO}^{(l)}\operatorname{RMS}(X^{(l)})(A^{(l)})^\top=X^{(l)}+W_{VO}^{(l)}X^{(l)}D^{(l)}(A^{(l)})^\top,
\end{equation}
where $D\topl=\operatorname{diag}(d\topl_1,\dots,d\topl_T)$, with $d_i\topl=\frac{\sqrt{d}}{\norm{X_{i}\topl}_2}$ for all $i\in[T]$.

For simplicity, we consider causal uniform attention at every layer, \ie, for every $l\in\N_0$, $A\topl\equiv A$, $A_{ij}=1/i$ for $j\leq i$ and $A_{ij}=0$ otherwise. We adopt this choice because it represents a regime in which token representations are mixed as strongly as possible. However, as will become clear from the proofs, similar results can be easily obtained for different causal attention schemes. 

First, we provide examples where the direction of each token is preserved at every attention step. Consequently, if the initial token representations have distinct directions, oversmoothing is avoided.
\begin{prop}\label{prop:rankN-conv-scale}
    Let $d\geq T$ and $X^{(0)}$ a full-column-rank matrix. Let us denote $Z\topl_i=\operatorname{RMS}(X^{(l)}_{i})$, for $i\in [T]$. For a fixed $\lambda> 0$, consider a matrix $W_{VO}^{(l)}=W_{VO}$, such that
    \begin{align*}
        &W_{VO}Z^{(0)}_{1}=\lambda Z^{(0)}_{1},\\
        &W_{VO}Z^{(0)}_{i}=\lambda (Z^{(0)}_{i}-Z^{(0)}_{i-1}),\qquad i\in\{2,\dots,T\}.
    \end{align*}
    
Then, with the dynamics defined in~\eqref{eq:att-prelaynorm-skip}, for every token $i\in [T]$ and every layer $l\in\N_0$, $Z^{(l)}_{i}=Z^{(0)}_{i}$.
\end{prop}
 \begin{proof}
    Let us denote $v_i:=Z^{(0)}_{i}$, for $i=\in [T]$.
     We will prove that $Z^{(l)}_{i}=v_i$ for all $l\in\N_0 $ by induction on $i$.

    (i) Base case. Let us prove $Z^{(l)}_{1}=v_1$ for all $l\in\N_0 $. We have 
    \begin{align*}
        X_{1}^{(1)}=W_{VO}Z^{(0)}_{1}+X_{1}^{(0)}=\lambda v_1+\frac{\norm{X_{1}^{(0)}}}{\sqrt{d}}v_1=\left(\lambda +\frac{\norm{X_{1}^{(0)}}}{\sqrt{d}}\right)v_1.
    \end{align*}
    Therefore, $Z^{(1)}_{1}=v_1$.

    For a fixed $l$, assume $Z^{(l)}_{1}=v_1$. Then,
    \[X_{1}^{(l+1)}=W_{VO}Z^{(l)}_{1}+X_{1}^{(l)}=\lambda v_1+\frac{\norm{X_{1}^{(l)}}}{\sqrt{d}}v_1=\left(\lambda +\frac{\norm{X_{1}^{(l)}}}{\sqrt{d}}\right)v_1.
    \]
    Therefore, $Z^{(l)}_{1}=v_1$ for all $l\in\N_0 $.
    
    (ii) Induction. Assume $Z^{(l)}_{k}=v_k$ for all $l\in\N_0$ and for all $k<i$.
    Then, 
    \begin{align*}
        X_{i}^{(1)}&=\frac{1}{i}W_{VO}Z^{(0)}_{1}+\frac{1}{i}W_{VO}Z^{(0)}_{2}+\dots+\frac{1}{i}W_{VO}Z^{(0)}_{i-1}+\frac{1}{i}W_{VO}Z^{(0)}_{i}+X_{i}^{(0)}\\
        &=\frac{1}{i}\lambda v_1+\frac{1}{i}\lambda (v_2-v_1)+\dots+\frac{1}{i}\lambda (v_{i-1}-v_{i-2})+\frac{1}{i}\lambda (v_i-v_{i-1})+\frac{\norm{X_{i}^{(0)}}}{\sqrt{d}}v_i\\
        &=\left(\frac{\lambda }{i}+\frac{\norm{X_{i}^{(0)}}}{\sqrt{d}}\right)v_i.
    \end{align*}
    Therefore, $Z^{(1)}_{i}=v_i$. 

    For a fixed $l$, assume $Z^{(l)}_{i}=v_i$. Then, since we also have $Z^{(l)}_{k}=v_k$ for all $k<i$ (induction hypothesis),
    \begin{align*}
        X_{i}^{(l+1)}&=\frac{1}{i}W_{VO}Z^{(l)}_{1}+\frac{1}{i}W_{VO}Z^{(l)}_{2}+\dots+\frac{1}{i}W_{VO}Z^{(l)}_{i-1}+\frac{1}{i}W_{VO}Z^{(l)}_{i}+X_{i}^{(l)}\\
        &=\frac{1}{i}W_{VO}v_1+\frac{1}{i}W_{VO}v_2+\dots+\frac{1}{i}W_{VO}v_{i-1}+\frac{1}{i}W_{VO}v_i+X_{i}^{(l)}\\
        &=\frac{1}{i}\lambda v_1+\frac{1}{i}\lambda (v_2-v_1)+\dots+\frac{1}{i}\lambda (v_{i-1}-v_{i-2})+\frac{1}{i}\lambda (v_i-v_{i-1})+\frac{\norm{X_{i}\topl}}{\sqrt{d}}v_i\\
        &=\left(\frac{\lambda }{i}+\frac{\norm{X_{i}\topl}}{\sqrt{d}}\right)v_i.
    \end{align*}
    Therefore, we have proven $Z^{(l)}_{i}=v_i$ for all $l\in\N_0 $.
\end{proof}
\begin{remark}\label{rem:existence-problem-WV}
    Note that the matrix $W_{VO}$ defined in Proposition~\ref{prop:rankN-conv-scale} always exists. First, let us denote $v_i:=Z^{(0)}_{i}$, for $i\in [T]$ and let $V=\operatorname{RMS}(X^{(0)})\in\R^{d\times T}$. We can extend $\{v_1,\dots,v_T\}$ to a basis of $\R^d$, $\mathcal{B}=\{v_1,\dots,v_T,u_{T+1},\dots,u_d\}$. The conditions in Proposition~\ref{prop:rankN-conv-scale} determine the image of the first $T$ elements of the basis, and the image of the remaining $d-T$ basis vectors can be chosen arbitrarily.
    $W_{VO}$ is a matrix such that
    \[
    W_{VO}V=\lambda \begin{pmatrix}
    \vline & \vline & & \vline \\
    v_1 & v_2 - v_1 & \dots & v_T - v_{T-1} \\
    \vline & \vline & & \vline
\end{pmatrix}=
    \lambda VL , \hspace{0.2cm} L=\begin{pmatrix}
        1 & -1 & 0 & \dots & 0& 0\\
        0 & 1 & -1 & \dots & 0& 0\\
        &&&\dots&\\
        &&&\dots&\\
        0 & 0 & 0 & \dots & 1& -1\\
        0 & 0 & 0 & \dots & 0& 1\\
    \end{pmatrix}.
    \]
    Since $V$ has full column rank, it has a left inverse $V^+=(V^\top V)^{-1}V^\top$, so one possible expression for $W_{VO}$ is
    \[W_{VO}=\lambda VL(V^\top V)^{-1}V^\top .\]
\end{remark}

Although such a configuration is possible, it is not realistic to expect a trained transformer to converge to such a configuration, since this would mean that the attention step reduces to a rescaling of each token. The following result establishes a much stronger statement: for any full-column-rank token matrix, one can choose value matrices so that each layer applies any prescribed scaled orthogonal transformation that preserves the current token span.

\begin{prop}\label{prop:rankN-moregen-conv}
    Let $d\geq T$ and let $X^{(0)}\in\mathbb{R}^{d\times T}$ have full column rank. 
    For each layer $l\in\mathbb N_0$, let 
    $U^{(l)}=[U_1^{(l)}\ U_2^{(l)}]\in\mathbb{R}^{d\times d}$ be an orthogonal matrix such that the columns of $U_1^{(l)}$ span $\operatorname{col}(X^{(l)})$ and the columns of $U_2^{(l)}$ span $\operatorname{col}(X^{(l)})^\perp$. Let $Q_1^{(l)}\in\mathbb{R}^{T\times T}$ and $Q_2^{(l)}\in\mathbb{R}^{(d-T)\times(d-T)}$ be arbitrary orthogonal matrices and let $\sigma\topl>0$. Define
    \[
        Q^{(l)}
        =
        U^{(l)}
        \begin{pmatrix}
            Q_1^{(l)} & 0\\
            0 & Q_2^{(l)}
        \end{pmatrix}
        (U^{(l)})^\top.
    \]
    Then it is possible to choose matrices $W_{VO}^{(l)}$ such that, with the dynamics defined in~\eqref{eq:att-prelaynorm-skip}, 
    \[
        X^{(l+1)}=\sigma\topl Q^{(l)}X^{(l)}.
    \]
    Consequently, $X^{(l+1)}$ has full column rank, and the pairwise cosine similarities between tokens are preserved across layers.
\end{prop}
 \begin{proof}

        For any fixed matrix $B\topl\in\R^{T\times T}$, we can set  
         \[W_{VO}\topl =X\topl (B\topl -I)((A\topl)^\top )^{-1}(D\topl )^{-1}(X\topl )^{+},\]
        where $(X\topl)^+=((X\topl)^\top X\topl)^{-1}(X\topl)^\top$ is the Moore--Penrose left inverse of $X\topl$, so that
        \[X^{(l+1)}=X\topl B\topl. \]
        Note that $W_{VO}\topl$ is well defined because $A\topl$ is lower triangular with nonzero diagonal and $d\topl_i\neq 0$ for all $i$. 

        Let us define $B\topl =\sigma\topl(X\topl )^{+}  Q\topl X\topl $, where $Q\topl \in\R^{d\times d}$ is an orthogonal matrix. For $\sigma\topl Q\topl X\topl  = X\topl B\topl$ to hold, it is enough that:
        \[ Q\topl X\topl=P\topl Q\topl X\topl ,\]
        where $P\topl =(X\topl )(X\topl )^{+}$ is the projection matrix onto the column space of $X\topl $. This means that rotating a column of $X\topl $ by $Q\topl $ must lie in the subspace spanned by the columns of $X\topl $.

        We can construct such a $Q\topl $ by decomposing the feature space $\R^d$ into two orthogonal subspaces: $\R^d=\operatorname{col}(X\topl )\oplus \operatorname{col}(X\topl )^\perp$, with dimensions $T$ and $d-T$ respectively. We then define $Q\topl $ as a block-diagonal orthogonal matrix with respect to this decomposition:
        \[
        Q\topl =U\topl\begin{pmatrix}
          Q\topl_{1}&0\\
          0&Q\topl_{2}
        \end{pmatrix}(U\topl)^\top,
        \]
        where $Q\topl_{1}\in\R^{T\times T}$ and $Q\topl_{2}\in\R^{(d-T)\times (d-T)}$ are arbitrary orthogonal matrices, $U\topl=[U_1\topl,U_2\topl]$ is orthogonal, with the columns of $U_1\topl$ spanning $\operatorname{col}(X\topl )$ and the columns of $U_2\topl$ spanning $\operatorname{col}(X\topl )^\perp$.

        By construction, $Q\topl$ preserves $\operatorname{col}(X\topl)$. Thus, with the above choice of $B\topl$,
        \[
    X^{(l+1)}=X\topl B\topl=\sigma\topl Q\topl X\topl .
\] Since $\sigma\topl>0$ and $Q\topl$ is orthogonal, the pairwise inner products of the normalized tokens remain constant across layers.
    \end{proof}

\subsection{Proof of Theorem \ref{thm:incr-lambda-skip} and discussion of its assumptions}\label{app:pf:thm:incr-lambda-skip}

\thIncrLambdaSkip*
\begin{proof}First, let us compute the exact value of $\E[Y_i(\lambda)^\top Y_j(\lambda)]$.
Let $A_\lambda=(1-\lambda)I+\lambda A_{u}$ with coordinates $A_\lambda=(a^\lambda_{ij})$, where $\sum_{j=1}^Ta_{ij}^\lambda=1$. Then, the expression of token $i$ after the head update is:
\begin{equation}\label{eq:token_lambda}
    Y_i(\lambda)=\beta Z_i+(WZA_\lambda^\top)_i=\beta Z_i+\sum_{k=1}^Ta^\lambda_{ik}WZ_k.
\end{equation}
Let us compute $\E[Z^\top_kMZ_l]$ for a fixed matrix $M\in\R^{d\times d}$ and arbitrary indices $k,l\in[T]$. Since $\E[\epsilon_i]=0$,
    \begin{align*}
            \E[Z^\top_kMZ_l]=\bar z^\top M\bar z+\E[\epsilon^\top_kM\epsilon_l]=\bar z^\top M\bar z+\Tr(\E[\epsilon_l\epsilon^\top_k]M).
    \end{align*}
Then, since $\E[\epsilon_l\epsilon^\top_l]=\Sigma_V$ and $\E[\epsilon_l\epsilon^\top_k]=\Sigma_C$ for $l\neq k$,
\begin{align*}
    \E[Z^\top_kMZ_l]=\bar z^\top M\bar z+\Tr(\Sigma_CM)+\delta_{kl}\Tr(BM)
            =\Tr(CM)+\delta_{kl}\Tr(BM).
\end{align*}
From Equation~\eqref{eq:token_lambda}, we have
\begin{align*}
    Y_i(\lambda)^\top Y_j(\lambda)=\beta^2Z_i^\top Z_j+\beta\sum_{k=1}^Ta^\lambda_{jk}Z^\top_iW Z_k+\beta\sum_{k=1}^Ta^\lambda_{ik}Z^\top_kW^\top Z_j+\sum_{k=1}^T\sum_{l=1}^Ta^\lambda_{ik}a^\lambda_{jl}Z^\top_kW^\top W Z_l.
\end{align*}
Combining both equations, we have
\begin{align*}
    \E[Y_i(\lambda)^\top Y_j(\lambda)]&=\beta^2(\Tr(C)+\delta_{ij}\Tr(B))\\ &+\beta(\sum_{k=1}^Ta^\lambda_{jk}(\Tr(CW)+\delta_{ik}\Tr(BW))+\sum_{k=1}^Ta^\lambda_{ik}(\Tr(CW^\top)+\delta_{kj}\Tr(BW^\top)))\\&+\sum_{k=1}^T\sum_{l=1}^Ta^\lambda_{ik}a^\lambda_{jl}(\Tr(CW^\top W)+\delta_{kl}\Tr(BW^\top W)).
\end{align*}
Since $\sum_{k=1}^Ta_{ik}^\lambda=1$, we obtain
\begin{align*}
    \E[Y_i(\lambda)^\top Y_j(\lambda)]&=\beta^2(\Tr(C)+\delta_{ij}\Tr(B))+\beta(\Tr(CW)+a^\lambda_{ji}\Tr(BW)+\Tr(CW^\top )+a^\lambda_{ij}\Tr(BW^\top ))\\&+\Tr(CW^\top W)+\sum_{k=1}^Ta^\lambda_{ik}a^\lambda_{jk}\Tr(BW^\top W)\\
    &=\beta^2(\Tr(C)+\delta_{ij}\Tr(B))+\beta(2\Tr(CW)+(a^\lambda_{ji}+a^\lambda_{ij})\Tr(BW))\\&+\Tr(CW^\top W)+\sum_{k=1}^Ta^\lambda_{ik}a^\lambda_{jk}\Tr(BW^\top W),
\end{align*}
where the last equality holds because $C$ and $B$ are symmetric. Note that $$a^\lambda_{ij}=\begin{cases}
    1-\lambda+\frac{\lambda}{i}&\text{if }j=i\\
    \frac{\lambda}{i}&\text{if }j<i\\
    0&\text{if }j>i
    
\end{cases},$$
so
$$\sum_{k=1}^Ta^\lambda_{ik}a^\lambda_{jk}=\begin{cases}
    \frac{\lambda}{j}&\text{if }i<j\\
    1-\lambda(2-\lambda)(1-\frac{1}{i})&\text{if }i=j
    
\end{cases}.$$
We now distinguish the two cases and conclude the first part of the proof.

(i) \underline{Case $i<j$:}
\begin{align*}
    \E[Y_i(\lambda)^\top Y_j(\lambda)]&=\beta^2\Tr(C)+\beta(2\Tr(CW)+\frac{\lambda}{j}\Tr(BW))+\Tr(CW^\top W)+\frac{\lambda}{j}\Tr(BW^\top W)\\
    &=\beta^2\Tr(C)+2\beta\Tr(CW)+\Tr(CW^\top W)+\frac{\lambda}{j}\left(\beta\Tr(BW)+\Tr(BW^\top W)\right).
\end{align*}
(ii) \underline{Case $i=j$:}
\begin{align*}
    \E[\norm{Y_i(\lambda)}^2]&=\beta^2\Tr(B+C)+2\beta\Tr(CW)+2\beta\left(1-\lambda+\frac{\lambda}{i}\right)\Tr(BW)\\&+\Tr(CW^\top W)+\left(1-\lambda(2-\lambda)\left(1-\frac{1}{i}\right)\right)\Tr(BW^\top W)\\
    &=\beta^2\Tr(B+C)+2\beta\Tr((B+C)W)+\Tr((B+C)W^\top W)\\&-2\left(1-\frac{1}{i}\right)(\beta\Tr(BW)+\Tr(BW^\top W))\lambda+\left(1-\frac{1}{i}\right)\Tr(BW^\top W)\lambda^2.
\end{align*}
We now analyze the monotonicity of $\hat\rho(Y(\lambda))$ and provide sufficient conditions for the numerator and denominator to be strictly increasing or decreasing. Let $D_k(\lambda):=\E[\norm{Y_k(\lambda)}^2]$. Then,
\begin{align*}
    D'_k(\lambda)=-2\left(1-\frac{1}{k}\right)(\beta\Tr(BW)+(1-\lambda)\Tr(BW^\top W)).
\end{align*}
First, note that for $i\neq j$, $\E[(\epsilon_i-\epsilon_j)(\epsilon_i-\epsilon_j)^\top ]=2(\Sigma_V-\Sigma_C)$, so $B=\Sigma_V-\Sigma_C$ is positive semidefinite, which implies $\Tr(BW^\top W)\geq0$. Therefore, $D'_k$ increases with $\lambda$. 
Then, $D'_k(\lambda)<0\;\forall \lambda\in[0,1]$ is equivalent to $D'_k(1)< 0$ and $D'_k(\lambda)>0\;\forall \lambda\in[0,1]$ is equivalent to $D'_k(0)> 0$.
 Since $D'_k(1)=-2\beta\left(1-\frac{1}{k}\right)\Tr(BW)$ and $D_k'(0)=-2\left(1-\frac{1}{k}\right)(\beta\Tr(BW)+\Tr(BW^\top W))$, $D_k(\lambda)$ strictly decreases iff $\beta\Tr(BW)> 0$ and strictly increases iff $\beta\Tr(BW)+\Tr(BW^\top W)<0$.
We distinguish two cases: 

 (i) $\beta\Tr(BW)> 0$. 
 Then, $\E[Y_i(\lambda)^\top Y_j(\lambda)]$ strictly increases with $\lambda$ and $D_k(\lambda)$ strictly decreases. Since both numerator and denominator are assumed to be positive, we conclude that $\hat\rho(Y(\lambda))$ strictly increases with $\lambda$.

(ii) $\beta\Tr(BW)+\Tr(BW^\top W)<0$.
Then, $\E[Y_i(\lambda)^\top Y_j(\lambda)]$ strictly decreases with $\lambda$ and $D_k(\lambda)$ strictly increases. Since both numerator and denominator are assumed to be positive, we conclude that $\hat\rho(Y(\lambda))$ strictly decreases with $\lambda$.
\end{proof}
\begin{remark}[Interpretation of the conditions of Theorem \ref{thm:incr-lambda-skip}]\label{rem:trace-conditions-theorem} 
First, note that
$B=\frac{1}{2}\E[(Z_i-Z_j)(Z_i-Z_j)^\top]$,
  so $B$ captures the covariance structure of token-to-token differences. Therefore,
  \[\Tr(BW)=\frac{1}{2}\E[(Z_i-Z_j)^\top W(Z_i-Z_j)].\]
  Then, $\Tr(BW)> 0$ means that, on average over token pairs, the value-output map $W$ maps discrepancy directions between tokens to directions positively aligned with the original differences.
  Similarly, 
  \[\Tr(BW^\top W)=\frac{1}{2}\E[\norm{W(Z_i-Z_j)}^2],\]
  so this term measures how much token-to-token variability is retained by the attention branch after applying \(W\).

  Assuming $\beta>0$, the first condition, $\beta \Tr(BW)> 0$ 
  means that the attention branch is positively aligned with the residual token differences. Intuitively, in this regime, sparse attention preserves or reinforces token-specific components. Making attention more uniform averages out these components, which increases cross-token alignment and thus leads to an increase of \(\hat\rho(Y(\lambda))\).
 By contrast, $\beta\Tr(BW)< 0$ means that the attention branch is anti-aligned with the residual stream, which can potentially counteract the mixing effect of the attention.

The second condition, $\beta\Tr(BW)+\Tr(BW^\top W)<0$ corresponds to the case where this anti-alignment is strong enough to dominate the nonnegative term \(\Tr(BW^\top W)\). In this regime, sparse attention cancels token differences rather than preserving them. Making attention more uniform weakens this cancellation, so \(\hat\rho(Y(\lambda))\) decreases.

The intermediate case $\beta\Tr(BW)<0$, but $\beta\Tr(BW)+\Tr(BW^\top W)>0$
has a mixed interpretation. The attention branch is anti-aligned with the residual token differences, but this anti-alignment is not strong enough to dominate the nonnegative term \(\Tr(BW^\top W)\). In this case, $\E[Y_i(\lambda)^\top Y_j(\lambda)]$ still increases with \(\lambda\), since its slope is proportional to \(\beta\Tr(BW)+\Tr(BW^\top W)\). However, by analyzing $D'_k(\lambda)$, we see that the token norms are not monotone: they initially decrease and then increase, attaining their minimum at
\[
    \lambda_*=
    \frac{\beta\Tr(BW)+\Tr(BW^\top W)}
         {\Tr(BW^\top W)}
    \in (0,1).
\]
Thus, in this regime, \(\hat\rho(Y(\lambda))\) is increasing in $\lambda$ for $\lambda\in[0,\lambda_*)$, but for $\lambda\in[\lambda_*,1]$ there is a competition between two effects: increasing uniformity increases both cross-token alignment and the token norms. Consequently, Theorem \ref{thm:incr-lambda-skip} does not imply a general monotonicity statement for \(\hat\rho(Y(\lambda))\) in this intermediate regime.
\end{remark}

\section{Proofs for Section \ref{sec:sink_switch}}\label{appx:proofs6}

Both Proposition \ref{thm:sink_switch_cppaste_generic} and Corollary \ref{thm:copypaste_sink_concrete} directly follow from the more general result below.

\begin{theorem}\label{thm:sink_switch_generic}
Consider any head of a globally optimal solution of the regularized objective. Assume that $W_{VO}x_0=0.$ If this head is implementing the attention sink mechanism, then it must implement the hard attention switch mechanism. Moreover, let $a_i$ be a collection of all input embeddings across all training data points with the positions ranging from 1 to $T-1.$ Let $b_i$ be the collection of the outputs of the attention layer of all these embeddings. If we further assume $\text{rank}(\{a_i\})=\text{rank}(\{b_i\})$ and that, for all input embedding sequences to the head, either the sequence is full column rank or the sequence lies in a strict half-space, then if the head is implementing the hard attention switch mechanism, it must implement it solely via the attention sink, i.e., all context embeddings $a_i$ for which $b_i=0$ attend fully to the BOS token. 
\end{theorem}
\begin{proof}
We start with the first implication. If the head is implementing the attention sink mechanism, there must exist at least one context in one sample that attends solely to the BOS token. This means that the query and key matrices are non-zero. From this, it follows that some context in some sequence has to have non-zero output of the attention layer. Otherwise, we could set all the matrices in the attention layer to 0 and achieve the same mapping with better regularization loss. This implies the head is implementing the hard attention switch (considering this non-zero mapping as the witness of the condition in the definition of the hard attention map).

For the second implication, we just need to realize that $b_i$ can be written as a convex combination of some value embeddings of some $a_j$s, i.e., $b_i = \sum_j \alpha_j W_{VO} a_j.$ This means that $b_i \in \text{Span}(\{W_{VO} a_k\}).$ As this holds for every 
$b_i,$ we get $\text{Span}(\{b_i\})\subset \text{Span}(\{W_{VO} a_i\}).$ As $\text{rank}(\{a_i\})=\text{rank}(\{b_i\})$, $W_{VO}$ has full rank on $\text{Span}(\{a_i\}),$ otherwise the rank inequality could not be satisfied. This means that the kernel of $W_{VO}$ on $\text{Span}(\{a_i\})$ is empty except for $0$. Now assume by contradiction that some $a_i$ for which $b_i=0$ does not fully attend to sink but rather uses a non-trivial combination of the previous value embeddings, i.e., $0=b_i=\sum_j \alpha_j W_{VO} a_j=W_{VO}\sum_j \alpha_j a_j.$ If the embeddings $a_j$ lie in a strict half-space and neither of the $W_{VO}a_j$ is 0 (as guaranteed by the kernel computation), then so must $W_{VO}a_j$ embeddings. As we are taking a convex combination, it is impossible to mix the embeddings up to 0, which is a contradiction. Furthermore, if the embeddings $a_i$ are full rank, then $0 \neq \sum_j \alpha_j a_j \in \text{Span}(\{a_i\})$ and, as such, cannot belong to the kernel of $W_{VO}$, which is again a contradiction. Therefore, such $a_i$ can only attend to sink with an attention weight of roughly 1.
\end{proof}




\begin{remark}\label{rem:dataset_size}
Let us discuss the expected minimal size of the training dataset to allow for the crux condition of Proposition \ref{thm:sink_switch_cppaste_generic} (for the simplified setting in which the copy-paste tokens copy-paste immediate past) and Corollary~\ref{thm:copypaste_sink_concrete} (each independent dimension being followed by a copy-paste token at least once or each token being copy-pasted at least once, respectively) to be satisfied with high probability. If the dimensions are roughly equally likely to be sampled, then this is an example of the coupon collector problem: denoting the dimension of the span of the input space as $R$, the expected size of the dataset is roughly $R\log R$ if there is no bias concerning the position at which any of the independent dimensions occur. Even if certain dimensions are only sampled at certain positions within the sequence, the bound would worsen only mildly to $R(\log R+\log T),$ where $T$ is the context length, using the union bound over context position in the standard coupon collector bound. 
\end{remark}



\begin{restatable}{corollary}{attnsinkattnswitchcppasteconcrete}
\label{thm:copypaste_sink_concrete}
Consider the input-output distribution of Proposition~\ref{thm:sink_switch_cppaste_generic} assuming there is a fixed set of dormant tokens $\gD$ and a fixed set of copy-paste tokens $\gC.$ 
Let tokens be sampled from a first-order Markov chain in which BOS is only followed by dormant tokens and any token can follow any other token with non-zero probability. Let inputs be computed as a sum of token embedding and positional embedding and every input sequence be full rank. Then, if every token gets copy-pasted at least once (except the BOS), all dormant tokens have to attend fully to the BOS token. 
\end{restatable}

\begin{remark}\label{rmk:exp}
Heads implementing the behavior of Corollary \ref{thm:copypaste_sink_concrete} as well as the one from Proposition~\ref{thm:sink_switch_cppaste_generic} are ubiquitous in pretrained transformers. Empirically, $16\%, 13\%$ and $21\%$ of the heads of gpt2, Llama 3.1-8B and Gemma 7B, pooled across 50 sequences per each of 4  language datasets (\emph{\texttt{TinyStories}, \texttt{tinyshakespeare}, \texttt{WikiText}, and \texttt{CodeSearchNet-Python}}), have at least $60\%$ of the attention mass on the sink and the lower-1 diagonal, plus at least $10\%$ on each of the two. It is also well known that such heads are the essential first step for implementing the induction head, a useful tool for  in-context learning \cite{sanford2024transformers}. The heads that show the generalized copy-paste pattern in which the copy-paste tokens copy also from more distant past are even more prevalent. 
\end{remark}

\section{Proofs for Section \ref{sec:sink_diag}}\label{appx:proof7}


\subsection{Backcopy with a low-dimensional positional frame}
\label{app:backcopy-frame}

We give the complete statement and proof of
Theorem~\ref{thm:backcopy-frame}. We then give both fixed-dimensional and logarithmic-dimensional families
for which sinks are asymptotically cheaper.

\paragraph{Data and positional assumptions.}
Let $\mathcal C$ and $\mathcal D$ be disjoint, nonempty sets of copy-paste
and dormant token indices, respectively, and put
$C=|\mathcal C|$, $D=|\mathcal D|$, and $m=\min\{C,D\}$.
Let $\{h_0\}\cup\{h_i:i\in\mathcal C\cup\mathcal D\}$ be an
orthonormal family in $\mathbb R^d$. Let $p_1,\ldots,p_T$ be unit vectors
orthogonal to every token direction, and define
\[
 E=\operatorname{span}\{p_1,\ldots,p_T\},\qquad
 r=\dim E,\qquad
 P=\begin{bmatrix}p_1^\top\\ \vdots\\ p_T^\top\end{bmatrix}.
\]
In particular, $1\le r\le T$ and $r+C+D+1\le d$.
Assume, for some $B\ge1$, that
\begin{equation}
 \|P\|_{\mathrm{op}}^2\le \frac{BT}{r}.
\label{eq:backcopy-frame-bound}
\end{equation}
Assume also that there are a linear map $R:E\to E$, a gap
$\gamma\in(0,1]$, and $L\ge1$ such that
\begin{equation}
\begin{aligned}
 p_{t-1}^\top Rp_t&=1
 &&(2\le t\le T),\\
 p_s^\top Rp_t&\le1-\gamma
 &&(2\le t\le T,\ 1\le s\le t,\ s\ne t-1),\\
 p_1^\top Rp_1&\le1,\qquad \|R\|_*\le Lr.
\end{aligned}
\label{eq:backcopy-shift-full}
\end{equation}
We extend $R$ by zero on $E^\perp$.

Set $s_0=\sqrt{d/2}$ and define
\begin{equation}
 x_{00}=\sqrt d\,h_0,\qquad x_{ti}=s_0(h_i+p_t)
 \quad(1\le t\le T,\ i\in\mathcal C\cup\mathcal D).
\label{eq:backcopy-inputs}
\end{equation}
An input sequence has BOS at position $0$, a dormant token at position $1$,
and an arbitrary ordinary token at each position $2,\ldots,T$.
Every such sequence is in the support of the distribution. For an input
sequence $x_0=x_{00}$, $x_t=x_{t,i_t}$, its target is
\begin{equation}
 f^*(x)_t=
 \begin{cases}
 x_0, &t=0,\\
 2x_t, &i_t\in\mathcal D,\\
 x_t+x_{t-1}, &i_t\in\mathcal C.
 \end{cases}
\label{eq:backcopy-targets}
\end{equation}
The last case occurs only for $t\ge2$.

\paragraph{Block, routing, and cost.}
We use ideal RMSNorm, $\operatorname{RMS}(y)=\sqrt d\,y/\|y\|_2$ for
$y\ne0$, with no learned gain. All inputs in
\eqref{eq:backcopy-inputs} have norm $\sqrt d$, so the first RMSNorm
leaves them unchanged. Consistently with the paper's convention, put
$W_{QK}=W_Q^\top W_K$. For convenience in the proof, write
\begin{equation}
 M=W_{QK}^\top=W_K^\top W_Q,\qquad
 S_{ts}=\frac{x_t^\top W_{QK}x_s}{\sqrt d}
       =\frac{x_s^\top Mx_t}{\sqrt d},\quad 0\le s\le t.
\label{eq:backcopy-scores}
\end{equation}
A prescribed key $\pi(t)$ must satisfy
$S_{t,\pi(t)}\ge S_{ts}+\kappa$ for every competing causal key
$s\ne\pi(t)$, and is assigned attention weight one in the hard-attention
model. Both patterns have $\pi(0)=0$ and $\pi(t)=t-1$ for copy-paste
tokens. For dormant tokens, the sink pattern has $\pi(t)=0$, whereas the
diagonal pattern has $\pi(t)=t$.

The block is
\begin{equation}
 y_t=x_t+W_OW_Vx_{\pi(t)},\qquad
 f_\Theta(x)_t
 =y_t+W_2\sigma\bigl(W_1\operatorname{RMS}(y_t)+b_1\bigr)+b_2,
\label{eq:backcopy-block}
\end{equation}
where $\sigma$ is coordinatewise ReLU. Query, key, value, and output
matrices are $d\times d$, and the MLP has hidden width $w$.
Its biases are unregularized, and its weight cost is
\begin{equation}
 \mathcal R(\Theta)
 =\|W_Q\|_F^2+\|W_K\|_F^2+\|W_V\|_F^2+\|W_O\|_F^2
  +\|W_1\|_F^2+\|W_2\|_F^2.
\label{eq:backcopy-cost}
\end{equation}
An admissible representation satisfies both the prescribed routing margins
and $f_\Theta(x)=f^*(x)$ on every sequence in the support.

\begin{theorem}[Full version of Theorem~\ref{thm:backcopy-frame}]
\label{thm:backcopy-frame-full}
Assume $T\ge3$, the data and positional conditions above, and
$w\ge r+D+1$. Define
\begin{align}
 U_{\mathrm{sink}}
 &=\frac{2\kappa}{\sqrt d}
   \left[\frac{2Lr}{\gamma}
     +\sqrt{2D}\left(1+\frac1\gamma\right)\right]
   +(5+25C)(r+D)+27C,
\label{eq:backcopy-upper-full}\\
 L_{\mathrm{diag}}
 &=2\sqrt{\frac23}\,\kappa
   \sqrt{\frac{mr(T-1)(T-2)}{dB}}.
\label{eq:backcopy-lower-full}
\end{align}
There exists an admissible sink representation of cost at most
$U_{\mathrm{sink}}$. Every admissible diagonal representation has cost
at least $L_{\mathrm{diag}}$. Consequently,
$U_{\mathrm{sink}}<L_{\mathrm{diag}}$ implies that a sink representation
is strictly cheaper than every admissible diagonal representation.
\end{theorem}

\begin{proof}
We first recall the factorization identity
\begin{equation}
 \inf_{A^\top B=N}\bigl(\|A\|_F^2+\|B\|_F^2\bigr)=2\|N\|_*.
\label{eq:backcopy-factorization}
\end{equation}
Indeed, $\|A^\top B\|_*\le\|A\|_F\|B\|_F
\le(\|A\|_F^2+\|B\|_F^2)/2$.
For a singular value decomposition $N=U\Sigma V^\top$, equality is
attained by $A=\Sigma^{1/2}U^\top$ and
$B=\Sigma^{1/2}V^\top$, padding with zero rows if necessary.

\paragraph{Upper bound on the sink representation cost.}
Let $h_{\mathcal D}=\sum_{b\in\mathcal D}h_b$ and set
\begin{equation}
 \beta=\frac{2\kappa}{\sqrt d\,\gamma},\qquad
 \alpha=\frac1{\sqrt2}\left(\beta+\frac{2\kappa}{\sqrt d}\right),
 \qquad
 M_{\mathrm{sink}}=\beta R+\alpha h_0h_{\mathcal D}^\top.
\label{eq:backcopy-sink-product}
\end{equation}
For an ordinary query at position $t$ and an ordinary key at position $s$,
orthogonality of the token and positional subspaces gives
\begin{equation}
 S_{ts}=\frac{\beta\sqrt d}{2}\,p_s^\top Rp_t.
\label{eq:backcopy-ordinary-score}
\end{equation}
For a copy-paste query, the BOS score is zero. Its predecessor has score
$\beta\sqrt d/2=\kappa/\gamma\ge\kappa$, and its score exceeds
every competing ordinary-key score by at least
$\beta\sqrt d\,\gamma/2=\kappa$. Thus every copy-paste query selects
its immediate predecessor.

For a dormant query, the BOS score is $\alpha\sqrt{d/2}$.
Every ordinary-key score is at most $\beta\sqrt d/2$, including at
$t=1$ by the final positional condition. Therefore
\[
 S_{t0}-S_{ts}\ge
 \alpha\sqrt{d/2}-\frac{\beta\sqrt d}{2}=\kappa.
\]
Every dormant query selects BOS. At position $0$ the only causal key is
BOS itself. By \eqref{eq:backcopy-factorization}, the query and key
matrices can be chosen with total squared cost
\begin{equation}
 2\|M_{\mathrm{sink}}\|_*
 \le2\beta\|R\|_*+2\alpha\|h_0h_{\mathcal D}^\top\|_*
 \le2\beta Lr+2\alpha\sqrt D.
\label{eq:backcopy-qk-upper}
\end{equation}

Let $\Pi$ be the orthogonal projector onto
$E\oplus\operatorname{span}\{h_i:i\in\mathcal C\cup\mathcal D\}$,
and choose $W_V=W_O=\Pi$. This projector annihilates BOS and preserves
every ordinary input. Its cost is
\begin{equation}
 \|W_V\|_F^2+\|W_O\|_F^2=2(r+C+D).
\label{eq:backcopy-vo-upper}
\end{equation}
The post-attention residual is consequently $y_0=x_0$,
$y_t=x_t$ on dormant tokens, and $y_t=x_t+x_{t-1}$ on copy-paste tokens.
The MLP must therefore add $x_t$ only on dormant tokens.

To construct this MLP, let
$F=E\oplus\operatorname{span}\{h_b:b\in\mathcal D\}$ and
$n=\dim F=r+D$. Choose an orthonormal basis $u_1,\ldots,u_n$
consisting of a basis of $E$ and the dormant token directions. Define
$h_{\mathcal C}=\sum_{a\in\mathcal C}h_a$ and, for an MLP input $z$,
\begin{equation}
\begin{aligned}
 A(z)&=h_{\mathcal C}^\top z,\\
 H_j(z)&=\operatorname{ReLU}(u_j^\top z-5A(z)+s_0),
 &&1\le j\le n,\\
 H_0(z)&=\operatorname{ReLU}(-5A(z)+s_0),\\
 g(z)&=\sum_{j=1}^n u_j\bigl(H_j(z)-H_0(z)\bigr).
\end{aligned}
\label{eq:backcopy-mlp-gate}
\end{equation}
This is a two-layer ReLU MLP of width $n+1$, with every hidden bias
equal to $s_0$ and output bias zero.

For a dormant token, $z=\operatorname{RMS}(y_t)=x_t\in F$,
$A(z)=0$, and $|u_j^\top z|\le s_0$. It follows that
$H_j(z)=u_j^\top z+s_0$, $H_0(z)=s_0$, and
$g(z)=\sum_j u_ju_j^\top z=z=x_t$.
For BOS, all inner products with $h_{\mathcal C}$ and $u_j$ vanish,
so $H_j(z)=H_0(z)=s_0$ and $g(z)=0$.

For a copy-paste token, write $y=x_{t,a}+x_{t-1,i}$ with
$a\in\mathcal C$. We have $\|y\|_2\le2\sqrt d$ and
$h_{\mathcal C}^\top y\ge s_0$. Thus, for
$z=\operatorname{RMS}(y)$,
$A(z)\ge s_0/2$ and $|u_j^\top z|\le\sqrt d=\sqrt2\,s_0$.
Consequently,
\[
 u_j^\top z-5A(z)+s_0
 \le\left(\sqrt2-\frac32\right)s_0<0,
 \qquad
 -5A(z)+s_0\le-\frac32s_0<0.
\]
All hidden units vanish, so $g(z)=0$. The block therefore gives exactly
the targets in \eqref{eq:backcopy-targets}.

The first-layer rows in \eqref{eq:backcopy-mlp-gate} are
$u_j^\top-5h_{\mathcal C}^\top$ and $-5h_{\mathcal C}^\top$;
the second-layer columns are $u_j$ and $-\sum_j u_j$.
Since $h_{\mathcal C}\perp F$ and $\|h_{\mathcal C}\|_2^2=C$,
\begin{equation}
 \|W_1\|_F^2=n(1+25C)+25C,\qquad
 \|W_2\|_F^2=2n.
\label{eq:backcopy-mlp-cost}
\end{equation}
If $w>n+1$, the remaining units are set to zero. Adding
\eqref{eq:backcopy-qk-upper}, \eqref{eq:backcopy-vo-upper}, and
\eqref{eq:backcopy-mlp-cost} gives
\[
 \mathcal R(\Theta)
 \le2\beta Lr+2\alpha\sqrt D
      +2(r+C+D)+(3+25C)(r+D)+25C
 =U_{\mathrm{sink}}.
\]

\paragraph{Lower bound on the diagonal representation cost.}
Consider any admissible diagonal representation, and write
$M=W_K^\top W_Q$. Choose $m$ disjoint pairs
$(a_j,b_j)\in\mathcal C\times\mathcal D$ and put
\[
 v_j=\frac{h_{b_j}-h_{a_j}}{\sqrt2},\qquad
 V=[v_1\ \cdots\ v_m].
\]
The columns of $V$ are orthonormal.
For every $t=3,\ldots,T$, the support contains a sequence with
$a_j$ at both positions $t-1,t$ and a sequence with $b_j$ at both
positions. The former query must select its predecessor, whereas the latter
must select itself. Comparing these two keys gives
\begin{align}
 \frac1{\sqrt d}(x_{t-1,a_j}-x_{t,a_j})^\top Mx_{t,a_j}&\ge\kappa,
\label{eq:backcopy-active-margin}\\
 \frac1{\sqrt d}(x_{t,b_j}-x_{t-1,b_j})^\top Mx_{t,b_j}&\ge\kappa.
\label{eq:backcopy-dormant-margin}
\end{align}
Adding the inequalities cancels the query positional component. With
$\Delta p_t=p_t-p_{t-1}$, we obtain
\[
 2\kappa\le
 \frac{s_0^2}{\sqrt d}\,\Delta p_t^\top M(h_{b_j}-h_{a_j})
 =\sqrt{\frac d2}\,\Delta p_t^\top Mv_j.
\]
Therefore
\begin{equation}
 (p_t-p_{t-1})^\top Mv_j\ge a,
 \qquad a=\frac{2\sqrt2\,\kappa}{\sqrt d}.
\label{eq:backcopy-increments}
\end{equation}

We use the following elementary sequence estimate. If
$z_{i+1}-z_i\ge a>0$ for $i=1,\ldots,N-1$, then, with
$\bar z=N^{-1}\sum_i z_i$,
\begin{equation}
\begin{aligned}
 \sum_{i=1}^N z_i^2
 &\ge\sum_{i=1}^N(z_i-\bar z)^2
 =\frac1N\sum_{i<j}(z_j-z_i)^2\\
 &\ge\frac{a^2}{N}\sum_{i<j}(j-i)^2
 =\frac{a^2}{12}N(N^2-1).
\end{aligned}
\label{eq:backcopy-sequence-energy}
\end{equation}
Apply this with $N=T-1$ to the scores
$p_2^\top Mv_j,\ldots,p_T^\top Mv_j$.
Writing $P'$ for the submatrix of $P$ with rows $2,\ldots,T$ and
summing over $j$ yields
\begin{equation}
 \|P'MV\|_F^2
 \ge\frac{ma^2}{12}(T-1)\bigl((T-1)^2-1\bigr).
\label{eq:backcopy-energy-lower}
\end{equation}
On the other hand, \eqref{eq:backcopy-frame-bound} and
$\|V\|_{\mathrm{op}}=1$ imply
\begin{equation}
 \|P'MV\|_F
 \le\|P\|_{\mathrm{op}}\|M\|_F
 \le\sqrt{\frac{BT}{r}}\,\|M\|_*.
\label{eq:backcopy-energy-upper}
\end{equation}
Combining these inequalities with
$\mathcal R(\Theta)\ge\|W_Q\|_F^2+\|W_K\|_F^2\ge2\|M\|_*$
gives
\[
 \mathcal R(\Theta)
 \ge\frac{2a\sqrt m}{\sqrt{12}}
 \sqrt{\frac{r(T-1)((T-1)^2-1)}{BT}}
 =2\sqrt{\frac23}\,\kappa
   \sqrt{\frac{mr(T-1)(T-2)}{dB}}.
\]
Here we used $((T-1)^2-1)=T(T-2)$. This proves the lower bound and,
together with the sink construction, the claimed sufficient comparison.
\end{proof}

\paragraph{Support and cost conventions.}
The diagonal proof only requires the repeated copy-paste and repeated dormant
pairs used above; full support is a simple sufficient assumption.
Starting the increment argument at $t=3$ is necessary because position $1$
is dormant. The lower bound concerns the prescribed diagonal routing class,
even if other attention patterns can compute the same output mapping.
If hidden biases are included in the penalty, the displayed sink construction
incurs the additional cost $(r+D+1)d/2$. All factorized matrix costs above
use the literal sum of squared Frobenius norms, and therefore include the
factor $2$ in \eqref{eq:backcopy-factorization}.

\paragraph{Dependence on context length.}
Because $r=\operatorname{rank}P\le T$, the frame factor
$\max\{1,T/r\}$ equals $T/r$. Thus
\eqref{eq:backcopy-lower-full} is exactly the form obtained from
\[
 \frac{2\sqrt{2/3}\,\kappa}{\sqrt{dB}}
 \sqrt{\frac{m(T-1)((T-1)^2-1)}{\max\{1,T/r\}}}.
\]
For orthonormal positions, $r=T$ and $B=1$, recovering the
$T^{3/2}/\sqrt d$ dependence. For $T\gg r$, the lower bound is of
order $\kappa T\sqrt{mr/(dB)}$. At fixed token-set sizes, the upper
bound is $O(r+\kappa Lr/(\sqrt d\,\gamma))$.

The dependence of $L$ and $\gamma$ on $T$ cannot generally be omitted.
Indeed, for $1\le i<j\le T-1$, apply
\eqref{eq:backcopy-shift-full} at query position $j+1$ to obtain
\[
 (p_j-p_i)^\top Rp_{j+1}\ge\gamma,
 \qquad
 \|p_j-p_i\|_2\ge\frac{\gamma}{\|R\|_{\mathrm{op}}}
                 \ge\frac{\gamma}{Lr}.
\]
Disjoint balls of radius $\gamma/(2Lr)$ around these unit vectors lie
inside a ball of radius $1+\gamma/(2Lr)$ in $E$. Comparing volumes gives
\begin{equation}
 T-1\le\left(1+\frac{2Lr}{\gamma}\right)^r.
\label{eq:backcopy-packing}
\end{equation}
Hence bounded $L$ and a uniformly positive $\gamma$ cannot persist as
$T\to\infty$ at fixed $r$. This does not prevent an asymptotic separation:
at fixed $r,d,C,D,\kappa$ and bounded $B$, the sufficient condition
$L(T)/\gamma(T)=o(T)$ already makes the sink upper bound smaller than
the diagonal lower bound for all sufficiently large admissible lengths.

\paragraph{Relation to finite-temperature softmax.}
The theorem uses hard routing with a margin, as in the main text.
For ordinary softmax, a dominant logit with gap at least $\kappa$ has
weight at least $(1+T e^{-\kappa})^{-1}$. Thus a uniform attention-mass
tolerance $\varepsilon$ can be enforced by taking
$\kappa\ge\log(T/\varepsilon)$; the displayed bounds retain their
explicit dependence on this choice. Exact output equality is asserted for
the hard-attention model.

\subsubsection{An explicit family with fixed embedding dimension}
\label{app:backcopy-fixed-d-example}

The next construction shows that the context-length comparison is
nonvacuous at fixed dimension. It includes a monotone clock so that both
attention patterns are realizable.

\begin{theorem}[A fixed-dimensional asymptotic separation]
\label{thm:backcopy-fixed-d-family}
Fix an integer $k\ge3$, nonempty token sets $\mathcal C,\mathcal D$,
a margin $\kappa>0$, and an ambient dimension
$d\ge2k+C+D+3$. Consider the block above with hidden width
$w\ge2k+D+3$. For every integer $q\ge3$, there is a backcopy
distribution with $T=q^k$ and positional rank $r=2k+2$ satisfying the
preceding theorem. If $\mathcal R^*_{\mathrm{sink}}$ and
$\mathcal R^*_{\mathrm{diag}}$ denote the infimal costs in the two
admissible routing classes, then, as $q\to\infty$,
\[
 \mathcal R^*_{\mathrm{sink}}=O(T^{2/k}),\qquad
 \mathcal R^*_{\mathrm{diag}}=\Theta(T).
\]
In particular, a sink is strictly cheaper for all sufficiently large $q$.
\end{theorem}

\begin{proof}
Set $f_\ell=q^\ell$ for $\ell=0,\ldots,k-1$ and define
\begin{equation}
 u_t=\frac1{\sqrt k}
 \bigl(\cos(2\pi f_0t/T),\sin(2\pi f_0t/T),\ldots,
       \cos(2\pi f_{k-1}t/T),\sin(2\pi f_{k-1}t/T)\bigr).
\label{eq:backcopy-harmonics}
\end{equation}
These are unit vectors. The frequencies are distinct and belong to
$(0,T/2)$, so discrete Fourier orthogonality gives
$U^\top U=(T/(2k))I_{2k}$ for the matrix with rows $u_t^\top$.
Let $R_h$ rotate each harmonic pair backward by its one-step angle.
Then $R_hu_t=u_{t-1}$ and $\|R_h\|_*=2k$.

For any nonzero cyclic difference $a\pmod T$, write $a=q^j b$
with $q\nmid b$. The frequency $f_{k-j-1}$ produces phase
$2\pi b/q$. Therefore
\begin{equation}
 \frac1k\sum_{\ell=0}^{k-1}\cos(2\pi f_\ell a/T)
 \le1-\gamma_T,\qquad
 \gamma_T=\frac{1-\cos(2\pi/q)}{k}
          =\Theta(T^{-2/k}).
\label{eq:backcopy-harmonic-gap}
\end{equation}
The value $a=q^{k-1}$ attains this bound. The predecessor has cyclic
difference zero; all its competing causal positions have a nonzero cyclic
difference. Thus \eqref{eq:backcopy-harmonic-gap} supplies the required
predecessor gap.

Choose $\rho^2=1/(2k+1)$ and append a two-dimensional clock:
\begin{equation}
 p_t=\sqrt{1-\rho^2}\,u_t
      \oplus\rho\bigl(\sin(t/T),\cos(t/T)\bigr),\qquad
 R=(1-\rho^2)^{-1}R_h\oplus0.
\label{eq:backcopy-clock}
\end{equation}
All $p_t$ have unit norm. Their rank is $2k+2$: their real sine and cosine
columns are an invertible change of coordinates of complex exponential
columns with distinct bases
$e^{\pm2\pi\mathrm{i}f_\ell/T}$ and $e^{\pm\mathrm{i}/T}$.
The first $2k+2$ rows form a nonsingular Vandermonde matrix, up to
nonzero column factors. Also,
\[
 \|P\|_{\mathrm{op}}^2
 \le(1-\rho^2)\frac{T}{2k}+\rho^2T
 =\frac{2T}{2k+1},\qquad
 \|R\|_*=\frac{2k}{1-\rho^2}=2k+1<r.
\]
Thus the frame and shift conditions hold with
$B=2r/(2k+1)$, $L=1$, and $\gamma=\gamma_T$.
The condition at position $1$ follows from the fact that every harmonic
inner product is at most one. The sink upper bound is consequently
$O(\gamma_T^{-1})=O(T^{2/k})$, and the diagonal lower bound is
$\Omega(T)$.

For completeness, we construct a diagonal representation with cost
$O(T)$. Let $e$ be the first clock-coordinate vector and set
\[
 \xi=p_t^\top Rp_t
     =\frac1k\sum_{j=1}^k\cos(2\pi/q^j),\qquad
 \delta=1-\xi.
\]
These quantities are independent of $t$, and
$\delta/\gamma_T\le k$. Define
\begin{equation}
 \beta=\frac{2\kappa}{\sqrt d\,\gamma_T},\qquad
 \eta=\frac{T}{\rho\cos(1)}
       \left(\beta\delta+\frac{2\kappa}{\sqrt d}\right),
 \qquad
 M_{\mathrm{diag}}=\beta R+\eta e h_{\mathcal D}^\top.
\label{eq:backcopy-diagonal-witness}
\end{equation}
For a copy-paste query, the clock term vanishes, so its predecessor routing
is unchanged. For a dormant query at $t$, the difference between the
current-key score and any earlier ordinary-key score satisfies
\[
 S_{tt}-S_{ts}
 \ge\frac{\sqrt d}{2}
 \left[-\beta\delta
       +\eta\rho\bigl(\sin(t/T)-\sin(s/T)\bigr)\right]
 \ge\kappa,
\]
because $\sin(t/T)-\sin(s/T)\ge\cos(1)(t-s)/T$ for
$1\le s<t\le T$. The BOS score is zero, while
\[
 S_{tt}\ge\frac{\sqrt d}{2}
    \left[\beta\xi+\frac{\eta\rho\cos(1)}{T}\right]
 =\frac{\sqrt d}{2}\left(\beta+\frac{2\kappa}{\sqrt d}\right)
 \ge\kappa.
\]
Thus all dormant queries select themselves. Taking $W_V=W_O=\Pi$ as
above and setting the MLP identically to zero gives exactly the target
mapping for diagonal routing, including at BOS.
Finally, $\beta=O(T^{2/k})$ and $\beta\delta=O(1)$, so
$\eta=O(T)$ and
\[
 \mathcal R(\Theta)
 \le2\beta\|R\|_*+2\eta\sqrt D+2(r+C+D)=O(T).
\]
Together with the lower bound, this proves
$\mathcal R^*_{\mathrm{diag}}=\Theta(T)$. Since $2/k<1$, the sink
upper bound is eventually strictly smaller.
\end{proof}

The asymptotic statement is along the sequence lengths $T=q^k$, with $k$
fixed and $q\to\infty$. The ambient dimension and token directions are fixed in this family;
the positional embeddings and their shared shift depend on $T$.
Thus the result concerns a family of finite-length backcopy distributions,
rather than prefixes of one fixed positional-embedding sequence.
Taking $k=3$ gives the rank-eight example stated in the main text.

\paragraph{Uniform dependence on the harmonic count.}
The implicit constant in the preceding fixed-$k$ bound is not uniform
when $k$ grows. We record the dependence explicitly. The elementary
inequality $1-\cos(2\pi/q)\ge 8/q^2$ for $q\ge3$ gives
$\gamma_T^{-1}\le kq^2/8$. Choosing the minimal ambient dimension
$d=2k+C+D+3$ and recalling $r=2k+2$, $L=1$ and fixed $C,D$,
\eqref{eq:backcopy-upper-full} yields, for a constant $K_{C,D}$
independent of $k,q,T,\kappa$,
\begin{equation}
 \mathcal R^*_{\mathrm{sink}}
 \le K_{C,D}\bigl(k+\kappa k^{3/2}q^2\bigr)
 = K_{C,D}\bigl(k+\kappa k^{3/2}T^{2/k}\bigr),
 \qquad T=q^k,\quad q\ge3,\quad k\ge3.
 \label{eq:backcopy-uniform-k}
\end{equation}
Taking $q=3$ and $k\to\infty$ gives $d,r=\Theta(\log T)$ and,
for fixed $\kappa$, $\mathcal R^*_{\mathrm{sink}}=O((\log T)^{3/2})$.
The diagonal lower bound remains $\Omega(T)$, while the diagonal
construction in the proof above has cost $O(T)$ uniformly in $k$:
indeed, $\beta=O(\kappa k^{1/2}q^2)$ in its worst-case form, but
in the diagonal witness $\beta\delta=O(\kappa\sqrt{k})$ would be
an overly crude estimate. More precisely, using the expression for
$\xi$ and $\delta$ in that construction, one has
\[
 \delta=\frac1k\sum_{j=1}^{k}
 \bigl(1-\cos(2\pi/q^j)\bigr)
 \le\frac{2\pi^2}{k(q^2-1)},
\]
so $\beta\delta=O(\kappa/\sqrt{k})$ for $d=\Theta(k)$.
Consequently $\eta=O(\kappa T)$ and the diagonal cost is
$O(k+\kappa T)$ for this choice of $d$. This already gives a
subpolynomial sink upper bound. The following construction improves it
to logarithmic cost and works for every sufficiently large $T$.

\subsubsection{A logarithmic-cost family with logarithmic dimension}
\label{app:backcopy-logarithmic-example}

\begin{theorem}[Logarithmic sink cost versus linear diagonal cost]
\label{thm:backcopy-logarithmic-family}
Fix nonempty token sets $\mathcal C,\mathcal D$ and a hard-attention
margin $\kappa>0$. For every sufficiently large integer $T$, there
exist unit positional embeddings of rank $r=\Theta(\log T)$ in an ambient
dimension $d=\Theta(\log T)$, satisfying
\eqref{eq:backcopy-frame-bound}--\eqref{eq:backcopy-shift-full}
with $B\le3$, $L=1$, and $\gamma=1/2$. With hidden width
$w\ge r+D+1$, the optimum costs among admissible sink and diagonal
representations satisfy
\[
 \mathcal R^*_{\mathrm{sink}}=O(\log T),\qquad
 \mathcal R^*_{\mathrm{diag}}=\Theta(T),
\]
where the implicit constants may depend on $C,D,\kappa$ but not on $T$.
In particular, sinks are strictly cheaper for all sufficiently large $T$.
\end{theorem}

\begin{proof}
For an integer $T$, put $N=\lfloor(T-1)/2\rfloor$ and
$k=\lceil32\log T\rceil$ (natural logarithm). For large enough $T$,
$k\le N$. We first show that there exist distinct frequencies
$f_1,\ldots,f_k\in\{1,\ldots,N\}$ such that
\begin{equation}
 \frac1k\sum_{j=1}^k\cos\!\left(\frac{2\pi f_j a}{T}\right)
 \le\frac12\qquad\text{for all }a\in\{1,\ldots,T-1\}.
 \label{eq:backcopy-random-frequency-gap}
\end{equation}
Sample the frequencies independently and uniformly from
$\{1,\ldots,N\}$. For every nonzero residue $a\pmod T$,
roots-of-unity summation gives
\[
 \mathbb E\cos\!\left(\frac{2\pi f_j a}{T}\right)
 =\begin{cases}
   -1/(T-1),& T\text{ odd},\\
   -(1+(-1)^a)/(T-2),& T\text{ even},
  \end{cases}
 \le0.
\]
Since every summand lies in $[-1,1]$, Hoeffding's inequality implies
\[
 \mathbb P\!\left\{\frac1k\sum_{j=1}^k
 \cos\!\left(\frac{2\pi f_j a}{T}\right)>\frac12\right\}
 \le e^{-k/8}.
\]
A union bound over the $T-1$ nonzero offsets gives failure
probability at most $(T-1)e^{-k/8}\le T^{-3}$.
The probability of any repeated frequency is at most
$k(k-1)/(2N)=O((\log T)^2/T)$. Thus, for sufficiently large $T$,
there is a distinct-frequency realization satisfying
\eqref{eq:backcopy-random-frequency-gap}. Fix such a realization.

Define the $2k$-dimensional harmonic vectors
\begin{equation}
 u_t=\frac1{\sqrt{k}}
 \bigl(\cos(2\pi f_1t/T),\sin(2\pi f_1t/T),\ldots,
 \cos(2\pi f_kt/T),\sin(2\pi f_kt/T)\bigr),
 \quad 1\le t\le T.
 \label{eq:backcopy-log-harmonics}
\end{equation}
They are unit vectors and, because the frequencies are distinct and
none is a Nyquist frequency, their row matrix $U$ satisfies
$U^\top U=(T/(2k))I_{2k}$. Let $R_h$ be the orthogonal,
block-diagonal rotation for which $R_hu_t=u_{t-1}$.
Then $u_{t-1}^\top R_hu_t=1$ and
$u_s^\top R_hu_t\le1/2$ whenever $s\ne t-1$ modulo $T$.

Set $\rho^2=1/(2k+1)$ and add the same two-dimensional monotone clock:
\begin{equation}
 p_t=\sqrt{1-\rho^2}\,u_t\oplus
 \rho\bigl(\sin(t/T),\cos(t/T)\bigr),
 \qquad
 R=(1-\rho^2)^{-1}R_h\oplus0.
 \label{eq:backcopy-log-clock}
\end{equation}
All $p_t$ have unit norm. Their span has rank $r=2k+2$:
complexifying the harmonic and clock coordinates produces the
$2k+2$ distinct exponential bases
$e^{\pm2\pi\mathrm{i}f_j/T}$ and $e^{\pm\mathrm{i}/T}$.
The first $2k+2$ samples form a nonsingular Vandermonde system,
so the real coordinate functions have full column rank.
The predecessor score is exactly one, and all competing ordinary
positions have score at most $1/2$. Also,
\[
 \|P\|_{\mathrm{op}}^2
 \le (1-\rho^2)\frac{T}{2k}+\rho^2T
 =\frac{2T}{2k+1},\qquad
 \|R\|_* =\frac{2k}{1-\rho^2}=2k+1<r.
\]
Thus $B=2r/(2k+1)<3$, $L=1$ and $\gamma=1/2$
work in \eqref{eq:backcopy-frame-bound}--\eqref{eq:backcopy-shift-full};
$p_1^\top Rp_1\le1$ follows from the harmonic cosine bound.
Choose $d=r+C+D+1$ and embed the token directions orthogonally
to the positional subspace. The finite-length theorem gives
\[
 \mathcal R^*_{\mathrm{sink}}
 \le O\!\left(r+\frac{\kappa r}{\sqrt d}\right)
 =O(\log T+\kappa\sqrt{\log T}),
\]
which is $O(\log T)$ at fixed $\kappa$. Since $r/d=\Theta(1)$
and $B$ is bounded, the same theorem gives
$\mathcal R^*_{\mathrm{diag}}=\Omega(\kappa T)$.

We verify that the diagonal class is realizable with cost $O(T)$.
Let $e$ be the first unit coordinate vector of the clock subspace,
let $h_{\mathcal D}=\sum_{b\in\mathcal D}h_b$, and write
\[
 \xi=p_t^\top Rp_t=\frac1k\sum_{j=1}^k
  \cos\!\left(\frac{2\pi f_j}{T}\right),
 \qquad \delta=1-\xi\in[0,2].
\]
Both quantities are independent of $t$. Set
\begin{equation}
 \beta=\frac{2\kappa}{\sqrt d\,\gamma},\qquad
 \eta=\frac{T}{\rho\cos(1)}
 \left(\beta\delta+\frac{2\kappa}{\sqrt d}\right),\qquad
 M_{\mathrm{diag}}=\beta R+\eta e h_{\mathcal D}^\top.
 \label{eq:backcopy-log-diagonal-witness}
\end{equation}
For copy-paste queries, the second term vanishes, so the predecessor
remains preferred by margin at least $\kappa$. For a dormant query,
the clock contributes $\eta\rho\sin(s/T)$ to the key at position $s$.
Since, for $1\le s<t\le T$,
$\sin(t/T)-\sin(s/T)\ge\cos(1)/T$ and
$p_s^\top Rp_t\le1$, the current-key advantage over earlier
ordinary keys is at least
\[
 \frac{\sqrt d}{2}\left[-\beta\delta+
       \eta\rho\frac{\cos(1)}{T}\right]=\kappa.
\]
Moreover, $\sin(t/T)\ge\cos(1)/T$ for every $t\ge1$, so the
current-key logit exceeds the BOS logit zero by at least
\[
 \frac{\sqrt d}{2}
  \left[\beta\xi+\eta\rho\frac{\cos(1)}{T}\right]
 =\frac{\sqrt d}{2}\left[\beta+\frac{2\kappa}{\sqrt d}\right]
 \ge\kappa.
\]
Thus this product enforces the diagonal routing. With
$W_V=W_O=\Pi$ as in the sink proof and zero MLP, every dormant
token receives its own input and every copy-paste token receives
its predecessor. BOS remains unchanged; all targets are exact in
the hard-attention model.

Finally, $\|R\|_*=O(k)$, $\beta=O(\kappa/\sqrt k)$,
$\rho^{-1}=O(\sqrt k)$ and $\delta\le2$. Therefore
$\eta=O(\kappa T)$ and
\[
 \mathcal R^*_{\mathrm{diag}}
 \le2\|M_{\mathrm{diag}}\|_*+2(r+C+D)
 \le2\beta\|R\|_*+2\eta\sqrt D+2(r+C+D)
 =O(\kappa T+\log T)=O(T)
\]
for fixed $\kappa$. Together with the lower bound this proves
$\mathcal R^*_{\mathrm{diag}}=\Theta(T)$.
\end{proof}

\paragraph{Why the dimension must grow for subpolynomial cost.}
The fact that $r,d=\Theta(\log T)$ above is not merely an artifact
of the harmonic construction. With fixed $r,d$ and $\kappa$, every
admissible representation (sink or diagonal) in this data model
satisfies the packing lower bound
\begin{equation}
 \mathcal R(\Theta)
 \ge\frac{\sqrt2\,\kappa}{\sqrt d}
 \bigl((T-1)^{1/r}-1\bigr).
 \label{eq:backcopy-all-pattern-packing}
\end{equation}
To see this, fix any $1\le i<j\le T-1$. The support contains
an input sequence with the same dormant token at positions $i,j$
and a copy-paste token at $j+1$. Because the latter must select its
immediate predecessor, the attention margin gives
\[
 \kappa
 \le\frac1{\sqrt d}(x_{j,b}-x_{i,b})^\top Mx_{j+1,a}
 \le\sqrt{\frac d2}\,\|M\|_{\mathrm{op}}\,\|p_j-p_i\|_2.
\]
Using $\mathcal R(\Theta)\ge2\|M\|_*$ and
$\|M\|_{\mathrm{op}}\le\|M\|_*$, we obtain
$\|p_j-p_i\|_2\ge2\sqrt2\kappa/(\sqrt d\,\mathcal R(\Theta))$.
Packing $T-1$ such unit vectors into an $r$-dimensional unit ball,
with pairwise separation $\Delta$, gives
$T-1\le(1+2/\Delta)^r$, which rearranges to
\eqref{eq:backcopy-all-pattern-packing}. In particular, a single
fixed-dimensional family cannot have $T^{o(1)}$ cost at fixed
$\kappa$ under these routing and support assumptions. This does
not contradict the fixed-$k$ result: for every prescribed
$\varepsilon>0$ one can first choose a sufficiently large fixed
$k\ge\max\{3,2/\varepsilon\}$ and corresponding fixed $d$, and
then let $T=q^k\to\infty$, obtaining a sink cost
$O(T^{\varepsilon})$ along those lengths.

\subsection{Sink-diagonal separation for the generic clustered dormant group distribution}\label{app:proof-sinkdiaggeneric}

\sinkdiaggeneric*
Before we start the proof, we recall some definitions and introduce additional quantities.
Define the normalized 
vectors
\[
u_0:=\frac{b_{\mathrm{BOS}}}{\sqrt d},
\qquad
u_c:=\frac{\bar d_c}{\sqrt d},
\qquad
u_{C+c}:=\frac{c_c}{\sqrt d}.
\]
Let $m:=1+2C$, $U\in\mathbb R^{d\times m}$ have columns $u_0,u_1,\dots,u_{2C}$ and $G:=U^\top U$. Since $G_{ii}=1$ and $|G_{ij}|\le \phi$ for $i\neq j$, Gershgorin's circle theorem gives
\[
\lambda_{\min}(G)\ge 1-2C\phi.
\]
Assume $2C\phi<1$, and define
\[
\gamma:=\frac{1}{1-2C\phi},
\qquad
\beta:=\sqrt{\frac{1+2C\phi}{1-2C\phi}},
\qquad
\Pi:=(U^\top U)^{-1}U^\top.
\]
Then, $\Pi u_j=e_j$ for all $j$, $\|\Pi\|_2^2\le\gamma$, and $\|U\|_2\le\sqrt{1+2C\phi}$. Recall 
$r_\eta=\dim\mathrm{span}\{\eta_{ci}: \text{all }c,i\}$.
For each group $c$, recall that $\mathcal P_c$ is the set of unordered pairs $\{i,j\}$ such that both relative orders of $d_{ci}$ and $d_{cj}$ occur in the data. Recall also
\[
\Sigma_D:=\sum_{c=1}^C\sum_{\{i,j\}\in\mathcal P_c}(d_{ci}-d_{cj})(d_{ci}-d_{cj})^\top,
\qquad
r_{\mathrm{eff}}(\Sigma_D):=\frac{\operatorname{tr}(\Sigma_D)}{\|\Sigma_D\|_2}.
\]
Define
\[
s_\lambda:=\max\{\lambda(C+1),\sqrt{C+\lambda^2}\},
\qquad
\tilde s_\lambda:=\sqrt{1+\lambda^2(C+1)^2}.
\]
and
\[
\Delta_{\mathrm{diag}}
:=
\min\{\|d_{ci}-d_{cj}\|_2^2:\ c\in[C],\ i\neq j,\ \text{$d_{cj}$ may appear before $d_{ci}$ in the data}\},
\]
with the convention that $\Delta_{\mathrm{diag}}=+\infty$ if no such pair exists.

We assume
\[
2C\phi\le \frac14,
\qquad
\phi\le \frac1{20},
\qquad
\delta\le \frac12,
\qquad
\frac1{20}\le \epsilon\le \frac14.
\]
Since $\lambda=\sqrt{1-\delta^2}$, this implies $\lambda\ge \sqrt3/2$. Moreover,
\[
\gamma\le \frac43,
\qquad
\beta\le \frac43.
\]

Now, Theorem~\ref{thm:sink_diag_generic} is implied by the following result. 

\begin{theorem}\label{thm:umbrella_bounds_relaxed_simple}
Under the setup above, the following four bounds hold.

\smallskip

\noindent
\textbf{1. Sink upper bound.}
There exists a sink solution with total cost at most
\[
\boxed{
U_{\mathrm{sink}}
\le
12\,\frac{\kappa C}{\sqrt d}
+
5(r_\eta+C).
}
\]

\smallskip

\noindent
\textbf{2. Diagonal lower bound.}
Every diagonal solution has total cost at least
\[
\boxed{
L_{\mathrm{diag}}
\ge
\frac{\kappa}{\delta^2\sqrt d}\,r_{\mathrm{eff}}(\Sigma_D)
+
\frac{\kappa C}{2\sqrt d}.
}
\]

\smallskip

\noindent
\textbf{3. Diagonal upper bound.}
There exists a diagonal solution with total cost at most
\[
\boxed{
U_{\mathrm{diag}}
\le
13\,\frac{\kappa C}{\sqrt d}
+
\frac{4\kappa\sqrt d}{\Delta_{\mathrm{diag}}}\,r_\eta
+
3(r_\eta+C).
}
\]

\smallskip

\noindent
\textbf{4. Sink lower bound.}
Every sink solution has total cost at least
\[
\boxed{
L_{\mathrm{sink}}
\ge
\frac{\kappa C}{2\sqrt d}.
}
\]
\end{theorem}

\begin{proof}
We prove the four bounds one by one.

\medskip

\noindent
\textbf{1. Sink upper bound.}
We construct an explicit sink solution.

\smallskip

\emph{Step 1: query-key construction.}
Because every nuisance vector $\eta_{ci}$ is orthogonal to $\mathcal S$, we have $U^\top\eta_{ci}=0$ and therefore $\Pi\eta_{ci}=0$. Hence
\[
\Pi d_{ci}
=
\Pi(\lambda\bar d_c+\eta_{ci})
=
\lambda\sqrt d\,e_c.
\]
Similarly,
\[
\Pi b_{\mathrm{BOS}}=\sqrt d\,e_0,
\qquad
\Pi c_c=\sqrt d\,e_{C+c}.
\]

Let $r:=(0,1,\dots,1)^\top\in\mathbb R^m$. Define
\[
M:=a\,r e_0^\top+b\sum_{c=1}^C e_{C+c}e_c^\top,
\qquad
W_{QK}:=\Pi^\top M\Pi.
\]
We choose
\[
a:=\frac{\kappa}{\lambda\sqrt d},
\qquad
b:=\frac{\kappa(1+\lambda)}{\lambda^2\sqrt d}.
\]

Let us check the attention scores. For a dormant query $d_{ci}$, the BOS score is
\[
s(d_{ci},b_{\mathrm{BOS}})
=
\frac{d_{ci}^\top W_{QK}b_{\mathrm{BOS}}}{\sqrt d}
=
\frac{(\lambda\sqrt d\,e_c)^\top M(\sqrt d\,e_0)}{\sqrt d}
=
\lambda a\sqrt d
=
\kappa.
\]
Every non-BOS score for a dormant query is $0$, so dormant tokens attend to BOS with margin at least $\kappa$.

For a copy-paste query $c_c$, the BOS score is
\[
s(c_c,b_{\mathrm{BOS}})
=
\frac{c_c^\top W_{QK}b_{\mathrm{BOS}}}{\sqrt d}
=
a\sqrt d
=
\frac{\kappa}{\lambda}
\ge \kappa.
\]
The score of a dormant key from the correct group is
\[
s(c_c,d_{ci})
=
\frac{c_c^\top W_{QK}d_{ci}}{\sqrt d}
=
\lambda b\sqrt d
=
\kappa\left(1+\frac1\lambda\right).
\]
Thus, if the correct group is present, its score is larger than that of BOS by exactly $\kappa$, and it is larger than that of wrong groups by at least $\kappa$. If the correct group is absent, the score of BOS is larger than that of all wrong groups by at least $\kappa$. Since $\Pi d_{ci}$ depends only on $c$, all correct-group dormant keys receive the same score, and the copy-paste attention is uniform over them.

Let us now bound the cost. Since the first term in $M$ is rank one and the second term has singular value $b$ with multiplicity $C$,
\[
\|M\|_*
\le
a\sqrt{2C}+bC
=
\frac{\kappa}{\sqrt d}
\left(
\frac{\sqrt{2C}}{\lambda}
+
\frac{C(1+\lambda)}{\lambda^2}
\right).
\]
Therefore
\[
\|W_{QK}\|_*
\le
\|\Pi\|_2^2\|M\|_*
\le
\gamma\,\frac{\kappa}{\sqrt d}
\left(
\frac{\sqrt{2C}}{\lambda}
+
\frac{C(1+\lambda)}{\lambda^2}
\right).
\]
The variational form of nuclear norm then gives
\[
\|W_Q\|_F^2+\|W_K\|_F^2
\le
2\|W_{QK}\|_*.
\]
Using $\gamma\le4/3$ and $\lambda\ge\sqrt3/2$, we get the simplified estimate
\[
\|W_Q\|_F^2+\|W_K\|_F^2
\le
12\,\frac{\kappa C}{\sqrt d}.
\]

\smallskip

\emph{Step 2: the base dormant-content map.}
Let $P_\eta$ be the orthogonal projector onto the nuisance span $\mathrm{span}\{\eta_{ci}\}$. Define
\[
D_{\mathrm{cent}}:=\operatorname{diag}(0,\underbrace{1,\dots,1}_{C},\underbrace{0,\dots,0}_{C}),
\qquad
Q_{\mathrm{cent}}:=U D_{\mathrm{cent}}\Pi.
\]
Thus $Q_{\mathrm{cent}}$ acts as the identity on each centroid $\bar d_c$, and kills $b_{\mathrm{BOS}}$, every copy-paste vector $c_c$, and every nuisance vector $\eta_{ci}$. Define
\[
T_{\mathrm{base}}:=P_\eta+Q_{\mathrm{cent}}.
\]
Then
\[
T_{\mathrm{base}}d_{ci}=d_{ci},
\qquad
T_{\mathrm{base}}b_{\mathrm{BOS}}=0,
\qquad
T_{\mathrm{base}}c_c=0.
\]
Moreover,
\[
\|T_{\mathrm{base}}\|_*
\le
\|P_\eta\|_*+\|Q_{\mathrm{cent}}\|_*.
\]
Here $\|P_\eta\|_*=r_\eta$, while
\[
\|Q_{\mathrm{cent}}\|_*
\le
\|U\|_2\,\|D_{\mathrm{cent}}\|_*\,\|\Pi\|_2
\le
C\beta.
\]
Therefore
\[
\|T_{\mathrm{base}}\|_*
\le
r_\eta+C\beta.
\]

\smallskip

\emph{Step 3: scaled value-output and MLP maps.}
Let
\[
W_OW_V:=W_{VO}:=\frac12 T_{\mathrm{base}}.
\]
The variational form of nuclear norm then gives
\[
\|W_V\|_F^2+\|W_O\|_F^2
\le
2\|W_{VO}\|_*
=
\|T_{\mathrm{base}}\|_*.
\]

Let the MLP implement the linear map
\[
T_{\mathrm{MLP}}:=(1+\epsilon)T_{\mathrm{base}}.
\]
Since biases are unregularized, a linear map can be realized on the bounded set of relevant normalized post-attention states by placing all ReLU preactivations in one positive region and using an output bias to subtract the constant offset. Hence
\[
\|W_1\|_F^2+\|W_2\|_F^2
\le
2(1+\epsilon)\|T_{\mathrm{base}}\|_*.
\]
The combined value-output and MLP cost is therefore at most
\[
(3+2\epsilon)\|T_{\mathrm{base}}\|_*.
\]
Since $\epsilon\le1/4$ and $\beta\le4/3$,
\[
(3+2\epsilon)(r_\eta+C\beta)
\le
5(r_\eta+C).
\]

\smallskip

\emph{Step 4: verifying the relaxed labels.}
Dormant tokens attend to BOS, and the BOS value is $0$. Hence the post-attention state of a dormant token is $d_{ci}$. The MLP contributes $(1+\epsilon)d_{ci}$. Therefore the final dormant output is
\[
d_{ci}+(1+\epsilon)d_{ci}=(2+\epsilon)d_{ci},
\]
which is allowed because it corresponds to $\phi_{ci}=1+\epsilon$.

If a copy-paste token has no correct-group predecessor, it attends to BOS and remains $c_c$. Since $T_{\mathrm{base}}c_c=0$, the MLP contributes $0$, so the final output is exactly $c_c$.

Suppose a copy-paste token has correct-group mean $\mu$. Its attention contribution is $\frac12\mu$, so the post-attention state is
\[
c_c+\frac12\mu.
\]
The MLP sees
\[
\mathrm{RMS}\!\left(c_c+\frac12\mu\right)
=
\frac{c_c+\frac12\mu}{R_\mu},
\qquad
R_\mu:=\frac{\|c_c+\frac12\mu\|_2}{\sqrt d}.
\]
Since $T_{\mathrm{base}}c_c=0$ and $T_{\mathrm{base}}\mu=\mu$, the MLP contributes
\[
(1+\epsilon)\frac{1}{2R_\mu}\mu.
\]
Therefore the final coefficient of $\mu$ is
\[
\theta_\mu
=
\frac12\left(1+\frac{1+\epsilon}{R_\mu}\right).
\]

We now bound $R_\mu$. Since $\mu$ is an average of dormant tokens of norm $\sqrt d$, we have $\|\mu\|_2\le\sqrt d$. Furthermore, we have
\[
\langle c_c,\mu\rangle
=
\lambda\langle c_c,\bar d_c\rangle,
\]
because the nuisance component of $\mu$ is orthogonal to $\mathcal S$. Hence $|\langle c_c,\mu\rangle|\le \lambda\phi d\le\phi d$. Therefore
\[
R_\mu^2
=
1+\frac{\|\mu\|_2^2}{4d}+\frac{\langle c_c,\mu\rangle}{d}
\le
1+\frac14+\phi
\le
\frac{13}{10}.
\]
Thus $R_\mu\le\sqrt{13/10}<7/6$.

For the lower bound,
\[
\|\mu\|_2^2
\ge
\lambda^2 d,
\]
and so
\[
R_\mu^2
\ge
1+\frac{\lambda^2}{4}-\lambda\phi
\ge 1,
\]
using $\lambda\ge\sqrt3/2$ and $\phi\le1/20$. Hence $1\le R_\mu\le 7/6$.

Therefore
\[
\theta_\mu
\le
\frac12(1+1+\epsilon)
\le
1+\epsilon,
\]
and
\[
\theta_\mu
\ge
\frac12\left(1+\frac{6}{7}(1+\epsilon)\right)
=
\frac{13}{14}+\frac{3\epsilon}{7}
\ge
1-\epsilon,
\]
where the last inequality uses $\epsilon\ge1/20$. Thus the relaxed copy-paste label is satisfied.

Combining the query-key cost with the value-output and MLP cost proves
\[
U_{\mathrm{sink}}
\le
12\,\frac{\kappa C}{\sqrt d}
+
5(r_\eta+C).
\]

\medskip

\noindent
\textbf{2. Diagonal lower bound.} This lower bound only uses attention-map constraints, so the relaxed value labels do not affect it.

\smallskip

\emph{Step 1: dormant self-attention cost.}
Fix a group $c$ and a pair $\{i,j\}\in\mathcal P_c$. Since the diagonal pattern must work for both relative orders, we have
\[
\frac{d_{ci}^\top W_{QK}d_{ci}-d_{ci}^\top W_{QK}d_{cj}}{\sqrt d}\ge\kappa
\]
and
\[
\frac{d_{cj}^\top W_{QK}d_{cj}-d_{cj}^\top W_{QK}d_{ci}}{\sqrt d}\ge\kappa.
\]
Adding the two inequalities gives
\[
(d_{ci}-d_{cj})^\top W_{QK}(d_{ci}-d_{cj})
\ge
2\kappa\sqrt d.
\]
Summing over all admissible pairs yields
\[
\langle W_{QK},\Sigma_D\rangle
\ge
2\kappa\sqrt d\,N_D,
\]
where $N_D:=\sum_c|\mathcal P_c|$. Hence, 
\[
\|W_{QK}\|_*
\ge
\frac{2\kappa\sqrt d\,N_D}{\|\Sigma_D\|_2}.
\]
Since $d_{ci}-d_{cj}=\eta_{ci}-\eta_{cj}$ and $\|\eta_{ci}\|_2=\delta\sqrt d$, we have
\[
\operatorname{tr}(\Sigma_D)
\le
4\delta^2 d\,N_D.
\]
Thus
\[
\|W_{QK}\|_*
\ge
\frac{\kappa}{2\delta^2\sqrt d}\,r_{\mathrm{eff}}(\Sigma_D).
\]

\smallskip

\emph{Step 2: copy-paste cost.}
In the diagonal pattern, copy-paste tokens attend to themselves when the correct group is absent. Hence for every $c$,
\[
\frac{c_c^\top W_{QK}(\lambda\bar d_c-c_c)}{\sqrt d}\ge\kappa,
\]
and for every $c'\neq c$,
\[
\frac{c_c^\top W_{QK}(\lambda\bar d_c-\lambda\bar d_{c'})}{\sqrt d}\ge\kappa.
\]
Summing all these inequalities gives
\[
\langle W_{QK},B_{\mathrm{copy}}\rangle\ge C^2\kappa\sqrt d,
\]
with
\[
B_{\mathrm{copy}}=d\,U N_\lambda U^\top,\quad
N_\lambda
=
\sum_{c=1}^C e_{C+c}
\bigl(\lambda((C+1)e_c-s)-e_{C+c}\bigr)^\top,
\quad
s:=\sum_{k=1}^C e_k.
\]
Note that
\[
\|N_\lambda\|_2=\tilde s_\lambda,
\]
implying that
\[
\|B_{\mathrm{copy}}\|_2
\le
d(1+2C\phi)\tilde s_\lambda.
\]
Therefore
\[
\|W_{QK}\|_*
\ge
\frac{C^2\kappa}{\sqrt d\,(1+2C\phi)\tilde s_\lambda}.
\]
Note that $B_{\mathrm{copy}}$ and $\Sigma_D$ act on orthogonal subspaces, so the costs in the two steps add:
\[
\|W_{QK}\|_*
\ge
\frac{\kappa}{2\delta^2\sqrt d}\,r_{\mathrm{eff}}(\Sigma_D)
+
\frac{C^2\kappa}{\sqrt d\,(1+2C\phi)\tilde s_\lambda}.
\]
The variational form of nuclear norm then gives 
\[
\|W_Q\|_F^2+\|W_K\|_F^2
\ge
\frac{\kappa}{\delta^2\sqrt d}\,r_{\mathrm{eff}}(\Sigma_D)
+
\frac{2C^2\kappa}{\sqrt d\,(1+2C\phi)\tilde s_\lambda}.
\]
Under the assumptions, $1+2C\phi\le5/4$ and $\tilde s_\lambda\le C+2$, implying that
\[
\frac{2C^2}{(1+2C\phi)\tilde s_\lambda}
\ge
\frac{2C^2}{(5/4)(C+2)}
\ge
\frac C2.
\]
Thus, we have the desired lower bound:
\[
L_{\mathrm{diag}}
\ge
\frac{\kappa}{\delta^2\sqrt d}\,r_{\mathrm{eff}}(\Sigma_D)
+
\frac{\kappa C}{2\sqrt d}.
\]

\medskip

\noindent
\textbf{3. Diagonal upper bound.}
We construct an explicit diagonal solution.

Let $P_\eta$ be the orthogonal projector onto the nuisance span. Define
\[
W_{QK}^{\mathrm{diag}}
=
\Pi^\top M_{\mathrm{diag}}\Pi+\alpha_{\mathrm{diag}}P_\eta,
\]
where
\[
M_{\mathrm{diag}}
=
s\sum_{c=1}^C e_c e_c^\top
+
a\sum_{c=1}^C e_{C+c}e_{C+c}^\top
+
b\sum_{c=1}^C e_{C+c}e_c^\top,
\]
with
\[
s:=\frac{\kappa}{\lambda^2\sqrt d},
\qquad
a:=\frac{\kappa}{\sqrt d},
\qquad
b:=\frac{2\kappa}{\lambda\sqrt d},
\qquad
\alpha_{\mathrm{diag}}:=\frac{2\kappa\sqrt d}{\Delta_{\mathrm{diag}}}.
\]
We now verify the attention scores induced by this construction. First consider a dormant query
$d_{ci}=\lambda \bar d_c+\eta_{ci}$. Since $\Pi d_{ci}=\lambda\sqrt d\,e_c$ and
$\Pi\eta_{ci}=0$, the special-span part $\Pi^\top M_{\mathrm{diag}}\Pi$ gives
\[
\frac{d_{ci}^\top \Pi^\top M_{\mathrm{diag}}\Pi d_{c'j}}{\sqrt d}
=
\begin{cases}
\lambda^2 s\sqrt d=\kappa, & c'=c,\\
0, & c'\neq c.
\end{cases}
\]
It also gives score $0$ to BOS and to all copy-paste keys. Thus, before adding the nuisance projector, a dormant token gives score $\kappa$ to every dormant key from its own group and score $0$ to all dormant keys from other groups, BOS, and copy-paste keys. This already separates its group from all other groups, but it does not yet distinguish the token from the other dormant tokens in the same group.

The nuisance term $\alpha_{\mathrm{diag}}P_\eta$ provides this within-group diagonal separation. Indeed, for another dormant key $d_{c'j}$, the additional score is
\[
\frac{\alpha_{\mathrm{diag}}\,\eta_{ci}^\top\eta_{c'j}}{\sqrt d}.
\]
Therefore the nuisance contribution to the score gap between the self key $d_{ci}$ and another dormant key $d_{c'j}$ is
\[
\frac{\alpha_{\mathrm{diag}}}{\sqrt d}
\bigl(\|\eta_{ci}\|_2^2-\eta_{ci}^\top\eta_{c'j}\bigr)
=
\frac{\alpha_{\mathrm{diag}}}{2\sqrt d}
\|\eta_{ci}-\eta_{c'j}\|_2^2,
\]
where we used that all nuisance vectors have the same norm. In particular, for $c'=c$ and $j\neq i$, the special-span scores are tied at $\kappa$, and the nuisance term breaks the tie in favor of the self key by at least
\[
\frac{\alpha_{\mathrm{diag}}\Delta_{\mathrm{diag}}}{2\sqrt d}
=
\kappa.
\]
For $c'\neq c$, the special-span part already gives the self key an additional margin $\kappa$ over the wrong group, and the nuisance term only increases the self-versus-other gap. Thus every dormant query attends to itself with margin at least $\kappa$.

Next consider a copy-paste query $c_c$. Since $\Pi c_c=\sqrt d\,e_{C+c}$, the special-span part gives
\[
s(c_c,c_c)=a\sqrt d=\kappa,
\]
and for a dormant key $d_{c'i}$,
\[
s(c_c,d_{c'i})
=
\begin{cases}
\lambda b\sqrt d=2\kappa, & c'=c,\\
0, & c'\neq c.
\end{cases}
\]
It gives score $0$ to BOS and to copy-paste keys other than $c_c$. The nuisance projector does not affect copy-paste queries, because copy-paste vectors lie in the special span and are orthogonal to the nuisance span. Therefore, if a correct-group dormant key is present, it receives score $2\kappa$, while the self key receives score $\kappa$ and all wrong-group dormant keys receive score $0$. Hence the correct group wins by margin at least $\kappa$. If no correct-group dormant key is present, the self key receives score $\kappa$ and all remaining admissible keys receive score $0$, so the copy-paste token attends to itself by margin $\kappa$.

The nuclear norm is bounded by
\[
\|W_{QK}^{\mathrm{diag}}\|_*
\le
\gamma\|M_{\mathrm{diag}}\|_*+\alpha_{\mathrm{diag}}\|P_\eta\|_*.
\]
Now
\[
\|M_{\mathrm{diag}}\|_*
\le
Cs+Ca+Cb
=
\frac{\kappa C}{\sqrt d}
\left(1+\frac1\lambda\right)^2,
\]
and $\|P_\eta\|_*=r_\eta$. Hence
\[
\|W_{QK}^{\mathrm{diag}}\|_*
\le
\gamma\,\frac{\kappa C}{\sqrt d}
\left(1+\frac1\lambda\right)^2
+
\frac{2\kappa\sqrt d}{\Delta_{\mathrm{diag}}}r_\eta.
\]
The variational form of nuclear norm then gives
\[
\|W_Q\|_F^2+\|W_K\|_F^2
\le
2\gamma\,\frac{\kappa C}{\sqrt d}
\left(1+\frac1\lambda\right)^2
+
\frac{4\kappa\sqrt d}{\Delta_{\mathrm{diag}}}r_\eta.
\]
Using $\gamma\le4/3$ and $\lambda\ge\sqrt3/2$ gives
\[
2\gamma\left(1+\frac1\lambda\right)^2\le 13.
\]
Thus
\[
\|W_Q\|_F^2+\|W_K\|_F^2
\le
13\,\frac{\kappa C}{\sqrt d}
+
\frac{4\kappa\sqrt d}{\Delta_{\mathrm{diag}}}r_\eta.
\]

For the value-output map, use
\[
W_{VO}^{\mathrm{diag}}:=(1-\epsilon)T_{\mathrm{base}}.
\]
This makes dormant self-attention contribute $(1-\epsilon)d_{ci}$, so the final dormant coefficient is $2-\epsilon$, which is allowed. It also makes copy-paste tokens receive $(1-\epsilon)\mu_c$ when the correct group is present, which is allowed, and gives zero value to copy-paste self-attention when no correct group is present.

The value-output cost is at most
\[
2(1-\epsilon)\|T_{\mathrm{base}}\|_*
\le
2(r_\eta+C\beta)
\le
3(r_\eta+C).
\]
Set the MLP to zero. Therefore
\[
U_{\mathrm{diag}}
\le
13\,\frac{\kappa C}{\sqrt d}
+
\frac{4\kappa\sqrt d}{\Delta_{\mathrm{diag}}}r_\eta
+
3(r_\eta+C).
\]

\medskip

\noindent
\textbf{4. Sink lower bound.}
This lower bound only uses copy-routing constraints in the sink pattern. The copy-paste query $c_c$ must make the correct group have score larger than BOS by $\kappa$:
\[
\frac{c_c^\top W_{QK}(\lambda\bar d_c-b_{\mathrm{BOS}})}{\sqrt d}\ge\kappa,
\]
and it must make the correct group have score larger than any wrong dormant group by $\kappa$:
\[
\frac{c_c^\top W_{QK}(\lambda\bar d_c-\lambda\bar d_{c'})}{\sqrt d}\ge\kappa
\qquad(c'\neq c).
\]
For a fixed $c$, summing the first inequality and all $C-1$ wrong-group inequalities gives
\[
\frac{
c_c^\top W_{QK}
\left(
\lambda(C+1)\bar d_c-\lambda\sum_{k=1}^C\bar d_k-b_{\mathrm{BOS}}
\right)
}{\sqrt d}
\ge
C\kappa.
\]
Summing this inequality over all $c=1,\dots,C$ gives
\[
\frac{\langle W_{QK},B_{\mathrm{sink}}\rangle}{\sqrt d}
\ge
C^2\kappa,
\]
where
\[
B_{\mathrm{sink}}
:=
\sum_{c=1}^C
c_c
\left(
\lambda(C+1)\bar d_c-\lambda\sum_{k=1}^C\bar d_k-b_{\mathrm{BOS}}
\right)^\top .
\]

We now upper-bound the operator norm of $B_{\mathrm{sink}}$. Let $s:=\sum_{k=1}^C e_k$. Using the matrix $U$ of normalized special vectors, we can write
\[
B_{\mathrm{sink}}
=
d\,U N_{\mathrm{sink}}U^\top,
\]
where
\[
N_{\mathrm{sink}}
:=
\sum_{c=1}^C
e_{C+c}
\left(
\lambda((C+1)e_c-s)-e_0
\right)^\top .
\]
Therefore,
\[
\|B_{\mathrm{sink}}\|_2
\le
d\,\|U\|_2^2\,\|N_{\mathrm{sink}}\|_2
\le
d(1+2C\phi)\|N_{\mathrm{sink}}\|_2.
\]

It remains to compute $\|N_{\mathrm{sink}}\|_2$. The nonzero rows of $N_{\mathrm{sink}}$ are
\[
r_c:=\lambda((C+1)e_c-s)-e_0,
\qquad c=1,\dots,C.
\]
A direct calculation gives
\[
\langle r_c,r_c\rangle
=
1+\lambda^2(C^2+C-1),
\]
and, for $c\neq c'$,
\[
\langle r_c,r_{c'}\rangle
=
1-\lambda^2(C+2).
\]
Thus $N_{\mathrm{sink}}N_{\mathrm{sink}}^\top$ is a $C\times C$ matrix with constant diagonal and constant off-diagonal entries. Its eigenvalue on the span of $\mathbf 1_C$ is
\[
C+\lambda^2,
\]
and its eigenvalue on the orthogonal complement of $\mathbf 1_C$ is
\[
\lambda^2(C+1)^2.
\]
Hence
\[
\|N_{\mathrm{sink}}\|_2
=
s_\lambda
:=
\max\{\lambda(C+1),\sqrt{C+\lambda^2}\}.
\]
Therefore,
\[
\|B_{\mathrm{sink}}\|_2
\le
d(1+2C\phi)s_\lambda.
\]

By nuclear/operator norm duality,
\[
C^2\kappa\sqrt d
\le
\langle W_{QK},B_{\mathrm{sink}}\rangle
\le
\|W_{QK}\|_*\|B_{\mathrm{sink}}\|_2.
\]
Using the operator-norm bound above, we obtain
\[
\|W_{QK}\|_*
\ge
\frac{C^2\kappa}{\sqrt d\,(1+2C\phi)s_\lambda}.
\]
The variational form of the nuclear norm, equivalently the balanced factorization bound for
$W_{QK}=W_Q^\top W_K$, gives
\[
\|W_Q\|_F^2+\|W_K\|_F^2
\ge
2\|W_{QK}\|_*.
\]
Therefore
\[
L_{\mathrm{sink}}
\ge
\frac{2C^2\kappa}{\sqrt d\,(1+2C\phi)s_\lambda}.
\]

Finally, under the small-geometry assumptions, $s_\lambda\le C+1$ and $1+2C\phi\le5/4$. Hence
\[
\frac{2C^2}{(1+2C\phi)s_\lambda}
\ge
\frac{2C^2}{(5/4)(C+1)}
\ge
\frac C2.
\]
Therefore
\[
L_{\mathrm{sink}}
\ge
\frac{\kappa C}{2\sqrt d},
\]
which completes the proof.
\end{proof}

\section{Proof of recovery of Theorem 1 from \cite{ranmilo2026necessary} from Theorem~\ref{thm:sink_switch_generic}}\label{app:ranmilo}

Here we give a full proof of how Theorem 1 from \cite{ranmilo2026necessary} can be recovered from Theorem~\ref{thm:sink_switch_generic}. 

\paragraph{Setting.} Fix integers $n\geq 5$ and $L\geq 4$.  Let $P$ be a probability
measure on $\mathbb{R}$ with a bounded density, and let
$\mathcal{C}:=\operatorname{span}\{e_4,\ldots,e_n\}$ and
$\mathcal{A}:=\operatorname{span}\{e_2,e_3\}\oplus\mathcal{C}$,
where $e_k$ denotes the $k$-th standard basis vector of $\mathbb{R}^n$.
Consider the following distribution of sequences
$x=(x^{(1)},\ldots,x^{(L)})$: choose a trigger position
$j\in\{2,\ldots,L\}$ uniformly, set $x^{(1)}=e_1$, and independently
sample $u^{(i)}\in\mathcal{C}$, $i\geq 2$, with each of its content
coordinates drawn from $P$.  Set
\begin{equation}\label{eq:trigger_task_inputs}
 x^{(i)}=\begin{cases}
 e_2+u^{(i)}, & i=j,\\
 e_3+u^{(i)}, & i\in\{2,\ldots,L\}\setminus\{j\}.
 \end{cases}
\end{equation}
The target is
\begin{equation}\label{eq:trigger_task_targets}
 y^{(i)}=\begin{cases}
 \displaystyle\frac{1}{j-1}\sum_{k=2}^{j}x^{(k)},&i=j,\\
 0,&i\neq j.
 \end{cases}
\end{equation}

Let $f$ be a single-layer causal attention model with combined
value--output map $V=W_O W_V\in\mathbb{R}^{n\times n}$, namely,
\begin{equation}\label{eq:trigger_task_model}
 f(x)^{(i)}=V\sum_{k=1}^{i}\alpha_{i,k}(x)x^{(k)},
 \qquad
 \alpha_{i,k}(x)\geq 0,
 \qquad \sum_{k=1}^{i}\alpha_{i,k}(x)=1.
\end{equation}
In particular, the usual causal softmax attention satisfies these
conditions.  Define the worst-case Euclidean error
\begin{equation}\label{eq:trigger_task_loss}
 \mathcal{L}(f):=
 \sup_{(x,y)\in\operatorname{supp}(\mathcal{D})}
 \max_{1\leq i\leq L}\|f(x)^{(i)}-y^{(i)}\|_2,
\end{equation}
where $\mathcal{D}$ is the joint input--target distribution above.

\begin{theorem}[Single-layer sink necessity for the trigger-conditional task]
\label{thm:trigger_conditional_sink_necessity}

There are constants $\eta_0>0$ and $C_P>0$, depending only on
$n$ and $P$, such that every model satisfying
$\mathcal{L}(f)\leq\eta\leq\eta_0$ obeys
\begin{equation}\label{eq:trigger_sink_quantitative}
 \alpha_{i,1}(x)\geq 1-C_P\eta
\end{equation}
for \emph{every} $(x,y)\in\operatorname{supp}(\mathcal{D})$ and
\emph{every} non-trigger position $i\neq j$.
Consequently, for any $\epsilon,\delta>0$ there exists
$\eta_{\epsilon}>0$ such that $\mathcal{L}(f)\leq\eta_{\epsilon}$ implies
\begin{equation}\label{eq:trigger_sink_probabilistic}
 \mathbb{P}_{(x,y)\sim\mathcal{D}}\!\left(
 \alpha_{i,1}(x)\geq 1-\epsilon
 \ \text{for all }i\neq j
 \right)\geq 1-\delta.
\end{equation}
In the exact case $\mathcal{L}(f)=0$, necessarily
$\alpha_{i,1}(x)=1$ at every non-trigger position.
\footnote{No restriction of the form $L<n$ and no global-optimality assumption
are required.}
\end{theorem}

\begin{proof}
We construct a finite set of trigger examples witnessing the same
input--output rank condition as in the hard-switch-to-sink direction
of Theorem~C.1, and then derive its quantitative counterpart.

\emph{Step 1: A full-rank set of target outputs.}
Since $P$ admits a probability density, its support contains
at least two distinct real numbers $a\neq b$.
Define
\begin{equation}\label{eq:trigger_witness_content}
 u^0:=a\sum_{m=4}^{n}e_m,
 \qquad
 u^m:=u^0+(b-a)e_m\quad (m=4,\ldots,n).
\end{equation}
Independence of the content coordinates implies that all these
vectors belong to the support of the content law.
For each $u\in\{u^0,u^4,\ldots,u^n\}$, choose an admissible sequence
whose trigger is at $j=2$ and whose second token is $e_2+u$.
Its target at position $2$ is $e_2+u$.
Choose one more admissible sequence with trigger at $j=3$,
second token $e_3+u^0$, and third token $e_2+u^0$.
Its target at position $3$ is
$\frac12(e_2+e_3)+u^0$.
(All the displayed prefixes can be extended to length $L$ using
ordinary tokens with content $u^0$.)

Put $r:=n-1=\dim\mathcal{A}$ and form the matrix of these $r$
target columns,
\begin{equation}\label{eq:trigger_witness_Y}
 Y:=\left[
 e_2+u^0,\ e_2+u^4,\ldots,e_2+u^n,\
 \tfrac12(e_2+e_3)+u^0
 \right]\in\mathbb{R}^{n\times r}.
\end{equation}
For $m=4,\ldots,n$, subtracting the first column from the
column indexed by $m$ gives $(b-a)e_m$; hence the columns of $Y$
span $\mathcal{C}$.  Subtracting $u^0\in\mathcal{C}$ from the first
column yields $e_2$, and the last column then provides $e_3$.
Thus $\operatorname{col}(Y)=\mathcal{A}$ and
\begin{equation}\label{eq:trigger_witness_constants}
 s:=\sigma_r(Y)>0,
 \qquad
 R:=\max\!\left\{
 \|e_2+u^0\|_2,\ \|e_3+u^0\|_2,\
 \max_{4\leq m\leq n}\|e_2+u^m\|_2
 \right\}<\infty.
\end{equation}
In particular, all non-BOS inputs span $\mathcal{A}$, and the
trigger targets span $\mathcal{A}$ as well.  For an exact realizer,
this verifies the input--output rank equality in Theorem~C.1.

\emph{Step 2: Quantitative injectivity of $V$ on $\mathcal{A}$.}
Assume $\mathcal{L}(f)\leq\eta$.
At the BOS position, causality gives $\alpha_{1,1}=1$ and the
target is zero; hence
\begin{equation}\label{eq:trigger_bos_value_bound}
 \|Ve_1\|_2=\|f(x)^{(1)}\|_2\leq\eta.
\end{equation}
Index the witness sequences from Step 1 by $\ell=1,\ldots,r$;
let $j_\ell\in\{2,3\}$ be their respective trigger positions and
let $y_\ell$ be the corresponding columns of $Y$.
For each witness, set
\begin{equation}\label{eq:trigger_witness_Z}
 z_\ell:=\sum_{k=2}^{j_\ell}
 \alpha_{j_\ell,k}(x_\ell)x_\ell^{(k)}\in\mathcal{A},
 \qquad Z:=[z_1,\ldots,z_r].
\end{equation}
Because the weights in this sum are nonnegative and add up to at
most one, the definition of $R$ gives
\begin{equation}\label{eq:trigger_Z_norm}
 \|Z\|_{\mathrm{op}}\leq\|Z\|_F\leq R\sqrt r.
\end{equation}
The model output at the trigger position of witness $\ell$ is
$\alpha_{j_\ell,1}(x_\ell)Ve_1+Vz_\ell$.
By the loss bound and \eqref{eq:trigger_bos_value_bound},
\begin{equation}\label{eq:trigger_witness_approx}
 \|Vz_\ell-y_\ell\|_2
 \leq\eta+\alpha_{j_\ell,1}(x_\ell)\|Ve_1\|_2
 \leq 2\eta,
 \qquad
 \|VZ-Y\|_{\mathrm{op}}\leq 2\sqrt r\,\eta.
\end{equation}
If $\eta\leq\eta_0:=s/(4\sqrt r)$, the singular-value
perturbation inequality yields
\begin{equation}\label{eq:trigger_weyl}
 \sigma_r(VZ)\geq\sigma_r(Y)-\|VZ-Y\|_{\mathrm{op}}
 \geq s/2>0.
\end{equation}
It follows that $Z$ has rank $r$.  Since its columns lie in the
$r$-dimensional subspace $\mathcal{A}$, $Z$ is a linear bijection
from $\mathbb{R}^r$ onto $\mathcal{A}$.
For any $u\in\mathcal{A}$, write $u=Zw$ and use
\eqref{eq:trigger_Z_norm}--\eqref{eq:trigger_weyl} to obtain
\begin{equation}\label{eq:trigger_restricted_singular}
 \|Vu\|_2=\|VZw\|_2
 \geq\sigma_r(VZ)\|w\|_2
 \geq\frac{s}{2R\sqrt r}\|u\|_2.
\end{equation}
Thus $V$ is quantitatively injective on $\mathcal{A}$ with
lower bound $\mu:=s/(2R\sqrt r)>0$.

\emph{Step 3: The strict-half-space margin forces the sink.}
Set $h:=e_2+e_3$.  By \eqref{eq:trigger_task_inputs}, every
non-BOS token, whether ordinary or a trigger, satisfies
\begin{equation}\label{eq:trigger_halfspace_margin}
 \langle h,x^{(k)}\rangle=1\qquad(k\geq 2),
 \qquad \|h\|_2=\sqrt 2.
\end{equation}
Fix any admissible sequence and any non-trigger position $i\geq2$.
Define the non-BOS attention mixture
\begin{equation}\label{eq:trigger_non_bos_mixture}
 z_i:=\sum_{k=2}^{i}\alpha_{i,k}(x)x^{(k)}\in\mathcal{A}.
\end{equation}
Using \eqref{eq:trigger_halfspace_margin} and normalization of
attention weights gives
\begin{equation}\label{eq:trigger_mixture_lower}
 \langle h,z_i\rangle=\sum_{k=2}^{i}\alpha_{i,k}(x)
 =1-\alpha_{i,1}(x),
 \qquad
 \|z_i\|_2\geq\frac{1-\alpha_{i,1}(x)}{\sqrt2}.
\end{equation}
Since the target at $i$ is zero, the worst-case loss bound and
\eqref{eq:trigger_bos_value_bound} imply
\begin{equation}\label{eq:trigger_mixture_upper}
 \|Vz_i\|_2
 \leq\|f(x)^{(i)}\|_2
     +\alpha_{i,1}(x)\|Ve_1\|_2
 \leq 2\eta.
\end{equation}
Combining \eqref{eq:trigger_restricted_singular},
\eqref{eq:trigger_mixture_lower}, and
\eqref{eq:trigger_mixture_upper}, we obtain
\begin{equation}\label{eq:trigger_sink_final_bound}
 \frac{s}{2R\sqrt{2r}}\bigl(1-\alpha_{i,1}(x)\bigr)
 \leq 2\eta,
 \qquad\text{so}\qquad
 1-\alpha_{i,1}(x)
 \leq \underbrace{\frac{4R\sqrt{2r}}{s}}_{=:C_P}\,\eta.
\end{equation}
For $i=1$ the claim is automatic, since $\alpha_{1,1}=1$.
All the estimates hold for every sequence in the support,
independently of its trigger position, establishing
\eqref{eq:trigger_sink_quantitative}.

Finally, for any $\epsilon,\delta>0$, take
\begin{equation}\label{eq:trigger_explicit_eta}
 \eta_{\epsilon}
 :=\min\!\left\{\frac{s}{4\sqrt r},
               \frac{\epsilon s}{4R\sqrt{2r}}\right\}>0.
\end{equation}
Then $\mathcal{L}(f)\leq\eta_{\epsilon}$ implies
$\alpha_{i,1}(x)\geq1-\epsilon$ simultaneously for all non-trigger
positions and all admissible sequences, so the event in
\eqref{eq:trigger_sink_probabilistic} has probability one (hence
at least $1-\delta$).  Setting $\eta=0$ in
\eqref{eq:trigger_sink_final_bound} gives the exact implication.
\end{proof}

\section{Extension of Section~\ref{sec:sink_diag} to Relative Positional Embeddings}\label{app:rope}

Although our theory is stated for absolute positional embeddings, a similar geometric interpretation can be developed for RoPE-style relative positional encodings. Standard RoPE applies position-dependent orthogonal rotations to queries and keys, namely $q_i=\Phi_i W_Q x_i$ and $k_j=\Phi_j W_K x_j$, so that
\[
S_{ij}=\frac{1}{\sqrt{d}}x_i^\top W_Q^\top \Phi_i^\top\Phi_j W_Kx_j.
\]
Thus, even for a fixed context length, the rotations $\Phi_i$ cannot generally be replaced by a single common rotation.

The appropriate interpretation is instead to impose the geometric assumptions on position-indexed embeddings. More precisely, for a context embedding $x_a$ occurring at position $i$, define
\[
z_{a,i}=\Phi_i x_a,
\qquad
z_{a,i}^\top z_{b,j}
=x_a^\top\Phi_i^\top\Phi_jx_b.
\]
Hence, the similarity between two occurrences depends on both their context embeddings and their positions. For a fixed context length $T$, the collection of rotations $\{\Phi_0,\ldots,\Phi_T\}$ is fixed and finite, allowing us to consider a fixed position-indexed geometry.

The absorption can be performed separately at each position. In particular, defining $\widetilde W_{Q,i}=\Phi_iW_Q\Phi_i^\top$ and $\widetilde W_{K,i}=\Phi_iW_K\Phi_i^\top$, we obtain $q_i=\widetilde W_{Q,i}z_{a,i}$ and $k_i=\widetilde W_{K,i}z_{a,i}$. Since $\Phi_i$ is orthogonal, $\|\widetilde W_{Q,i}\|_F=\|W_Q\|_F$ and $\|\widetilde W_{K,i}\|_F=\|W_K\|_F$. Importantly, however, these transformed matrices depend on position and cannot be selected independently, as they must all originate from the same shared matrices $W_Q$ and $W_K$.

Under this interpretation, dormant groups should be defined in the position-indexed query-key geometry. In particular, the same residual-stream embedding may belong to different dormant groups when occurring at sufficiently different positions, while two different residual-stream embeddings may belong to the same dormant group when their position-indexed representatives are sufficiently close. This is analogous to the absolute-positional setting, where the input to the attention layer already depends on both token identity and position.

Importantly, the dormant-group geometry used by the query-key mechanism need not coincide with the geometry of the residual stream. The two parts of the proof of Theorem 7.2 use these geometric assumptions for different purposes.

\paragraph{Sink upper bound.}
The query-key part determines the attention pattern. For the sink upper bound, the problem becomes a finite collection of margin constraints over token-position pairs: dormant queries must select BOS, while copy-paste queries must select the preceding dormant embeddings belonging to the correct position-indexed group. Analogues of the separation and incoherence assumptions in Theorem 7.2 can then be formulated for the position-indexed geometry.

However, obtaining the corresponding low-cost sink construction additionally requires that all these constraints can be realized by a single shared pair $W_Q,W_K$. Indeed, the transformed matrices must satisfy
\[
\Phi_i^\top\widetilde W_{Q,i}\Phi_i=W_Q,
\qquad
\Phi_i^\top\widetilde W_{K,i}\Phi_i=W_K
\]
for every position $i$. Consequently, a collection of individually low-cost query-key maps in the position-indexed geometry does not necessarily correspond to a valid RoPE attention head. The cost of satisfying this shared-parameter restriction may depend on the conditioning of the position-indexed geometry and the context length, and is not automatically controlled by the original geometric assumptions.

The value-output and MLP parts, on the other hand, do not need to distinguish the individual dormant groups. They only require a low-cost linear map $T_{\rm base}$ satisfying $T_{\rm base}d=d$ for every dormant residual embedding $d$, together with $T_{\rm base}b_{\rm BOS}=0$ and $T_{\rm base}c=0$ for every copy-paste residual embedding $c$.

More generally, the residual-side complexity can be expressed as
\[
\begin{aligned}
\tau_{\rm res}=\inf_T\big\{&\|T\|_*:
Td=d\text{ for every dormant }d,\\
&Tb_{\rm BOS}=0,\;
Tc=0\text{ for every copy-paste }c
\big\}.
\end{aligned}
\]
Using the same value-output and MLP construction as in the original proof gives a residual-side cost of at most $(3+2\epsilon)\tau_{\rm res}$. Under the original residual-stream assumptions of Theorem 7.2, we have $\tau_{\rm res}\leq r_\eta+C\beta$, recovering the existing cost bound $5(r_\eta+C)$. Thus, changing the dormant groups in the position-indexed query-key geometry does not require the residual embeddings to form the same groups.

\paragraph{Diagonal lower bound.}
The diagonal lower bound depends even less on the residual geometry, since it uses only attention-score constraints. For example, suppose two dormant embeddings $x_a,x_b$ occur at positions $p<q$ in both relative orders. Diagonal attention requires
\[
\begin{aligned}
x_b^\top W_Q^\top W_Kx_b
-x_b^\top W_Q^\top\Phi_q^\top\Phi_pW_Kx_a
&\geq \kappa\sqrt{d},\\
x_a^\top W_Q^\top W_Kx_a
-x_a^\top W_Q^\top\Phi_q^\top\Phi_pW_Kx_b
&\geq \kappa\sqrt{d}.
\end{aligned}
\]
These are the RoPE analogues of the inequalities used in the original proof. They can be aggregated over admissible position-indexed pairs to obtain a generalized witness replacing $\Sigma_D$. This witness measures the complexity of the query-key separation required to make every dormant occurrence attend to itself.

Unlike the original construction, the generalized witness does not necessarily reduce to a covariance matrix of embedding differences. In particular, reversing the order of two embeddings also changes their position-indexed representatives. Nevertheless, when the position dependence within a dormant group is sufficiently small, the resulting witness approaches the original construction up to perturbation terms, suggesting a corresponding effective-rank lower-bound mechanism.

\paragraph{Discussion and limitations.}
Overall, the comparison retains the same conceptual structure: the sink upper bound consists of a query-key routing cost and a residual-side cost, while the diagonal lower bound depends entirely on query-key score constraints. Schematically, we may write
\[
\begin{aligned}
U_{\rm sink}
&\leq U_{\rm QK,sink}
+(3+2\epsilon)\tau_{\rm res},\\
L_{\rm diag}
&\geq L_{\rm QK,diag},
\end{aligned}
\]
where the query-key terms must be defined using the RoPE-compatible position-indexed score constraints.

However, this formulation does not establish an exact equivalence with Theorem 7.2 for arbitrary RoPE. A fully rigorous extension requires controlling the cost and feasibility of the sink construction under the shared $W_Q,W_K$ restriction, as well as deriving the precise position-indexed analogue of the effective-rank lower bound. In particular, fixed context length makes the problem finite but does not guarantee that the additional cost is bounded by a universal constant. Establishing these geometry-dependent bounds and precisely characterizing the resulting sink-versus-diagonal cost comparison is an interesting direction for future work.

\section{Further experimental evidence}
\subsection{Oversmoothing}\label{appx:exp-oversmoothing}
\paragraph{Experimental details.}
We consider the same experimental setting as in Section~\ref{sec:oversmoothing}: several trained transformer models (LLaMA3-8B, Gemma-7B, GPT2-XL, and Mistral-7B) and different context lengths. For each model, context length and dataset, we average results across all heads and over a batch of 50 sequences. To compute the theoretical quantities, we estimate the parameters of the token distribution for each context length from a batch of 50 sequences. 
For each attention head, $W$ corresponds to the composition of the learnable normalization scaling, the value matrix, and the output projection matrix corresponding to the head.

Consider a layer $l$ with $n_l$ heads. We denote by $\operatorname{ATTN}_l^h(Z\topl)$ the attention update corresponding to head $h$, and by $\operatorname{ATTN}_l(Z\topl)$ the full attention update of layer $l$, obtained by concatenating the contributions of all heads.

In our experiments, we use two slightly different quantities depending on the goal. 
When comparing the empirical cosine similarity with our theoretical approximation for a single head $h$, we use the token matrix $X+\operatorname{ATTN}_l^h(Z\topl)$, since this is precisely the quantity captured by the theory.

However, when empirically assessing the effect of attention on oversmoothing after the skip connection, a single-head update must be rescaled to be comparable to the full layer update; otherwise, its contribution would be too small relative to the residual stream. For this reason, we use

\begin{equation}\label{eq:normATTN}X+s_l\operatorname{ATTN}_l^h(Z\topl)\text{ , where } s_l=\frac{ \frac{1}{T}\sum_{i=1}^T\norm{\operatorname{ATTN}_l(Z\topl_i) }}{ \frac{1}{n_l}\sum_{h=1}^{n_l}\frac{1}{T}\sum_{i=1}^T\norm{\operatorname{ATTN}_l^h(Z\topl_i)}}.\end{equation}
Thus, $s_l$ rescales each head update by the ratio between the average norm of the full layer update and the average norm of a single-head update in that layer. 

In several experiments, we also need a measure of how dense an attention matrix is. To this end, we define the \emph{uniformity coefficient} of a matrix $A\in\R^{T\times T}$ as 
\begin{equation}\label{eq:unifcoeff} u(A)=\frac{1}{T}(1+\sum_{i=2}^{T}(-\sum_{j=1}^{i}A_{ij}\log A_{ij})/\log i)\in[0,1],
\end{equation} 
that is, the average normalized row-wise entropy of $A$. Values of $u(A)$ close to 0 correspond to sparse attention, while values close to 1 indicate attention closer to uniform.

\paragraph{Token representations tend to exhibit positive cosine similarity across layers.}
Our theoretical assumptions on the token representation distribution are motivated by the empirical observation that token representations tend to have positive cosine similarity across layers. Figure~\ref{fig:avgcossim-across-layers} provides evidence for this behavior. For each layer, it reports the average cosine similarity of the token matrix both before entering the layer and after adding the corresponding attention update to the residual stream. In all cases, the average cosine similarity remains positive across layers.
\begin{figure}
    \centering
    \begin{subfigure}[b]{0.49\textwidth}
        \centering
        {\small LLaMA3-8B, C4\par}
        \vspace{2pt}
        \includegraphics[width=\textwidth]{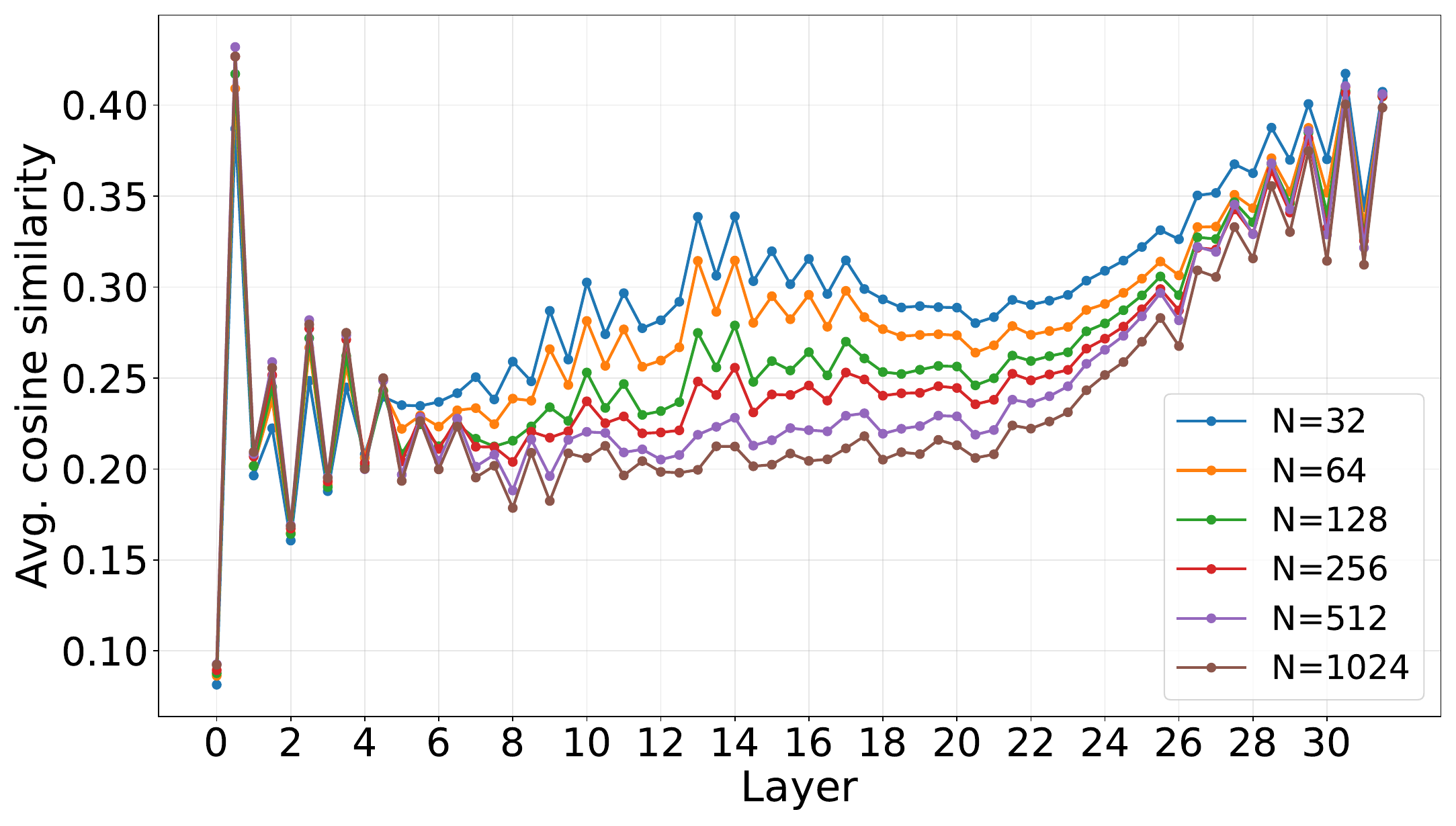}
    \end{subfigure}
    \hfill 
    \begin{subfigure}[b]{0.49\textwidth}
        \centering
        {\small GPT2-XL, CodeParrot\par}
        \vspace{2pt}
        \includegraphics[width=\textwidth]{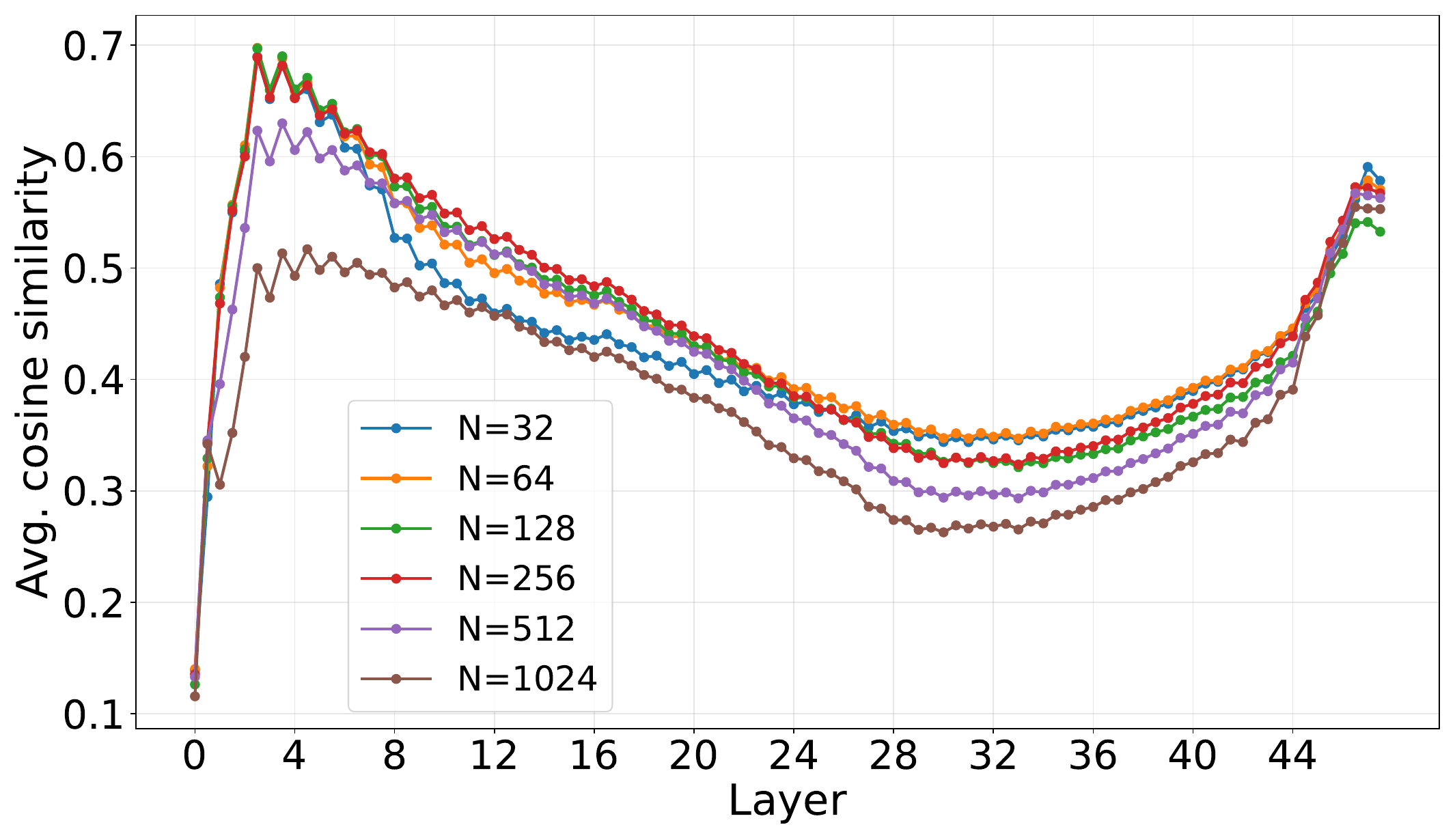}
    \end{subfigure}
     \hfill 
    \begin{subfigure}[b]{0.49\textwidth}
        \centering
        {\small Gemma-7B, WikiText\par}
        \vspace{2pt}
        \includegraphics[width=\textwidth]{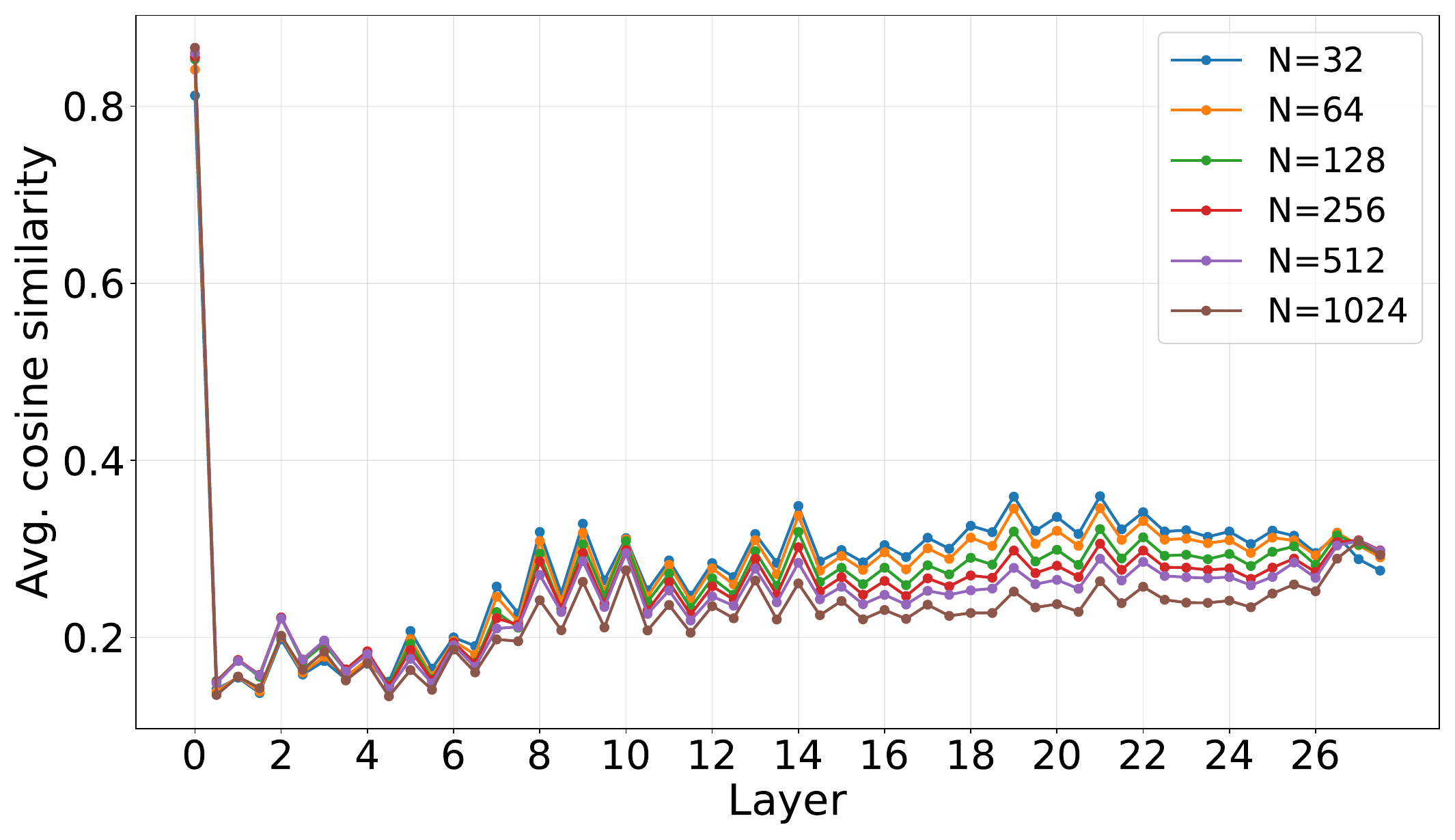}
    \end{subfigure}
    \hfill 
    \begin{subfigure}[b]{0.49\textwidth}
        \centering
        {\small Mistral-7B, C4\par}
        \vspace{2pt}
        \includegraphics[width=\textwidth]{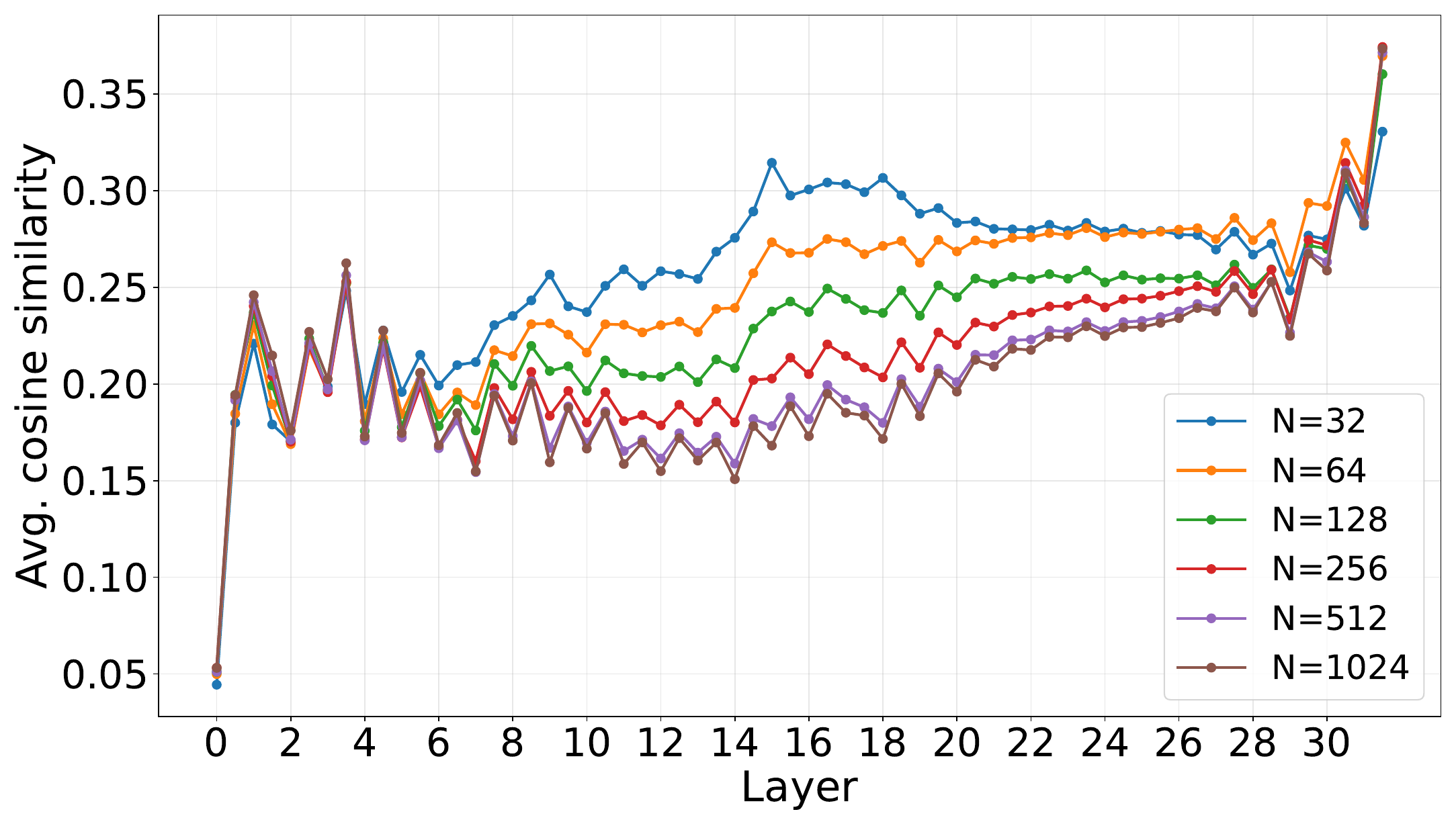}
    \end{subfigure}

    \caption{Evolution of average cosine similarity across layers for different models, datasets, and context lengths. For each layer, the first value corresponds to the token matrix before the layer, and the second to the token matrix after adding the attention update to the residual stream. In all cases, average cosine similarity remains positive across layers.}
    \label{fig:avgcossim-across-layers}
\end{figure}
 
\paragraph{Robustness across models and datasets.} 
\begin{figure}[t]
    \centering
    \begin{subfigure}[b]{0.49\textwidth}
        \centering
        \includegraphics[width=\textwidth]{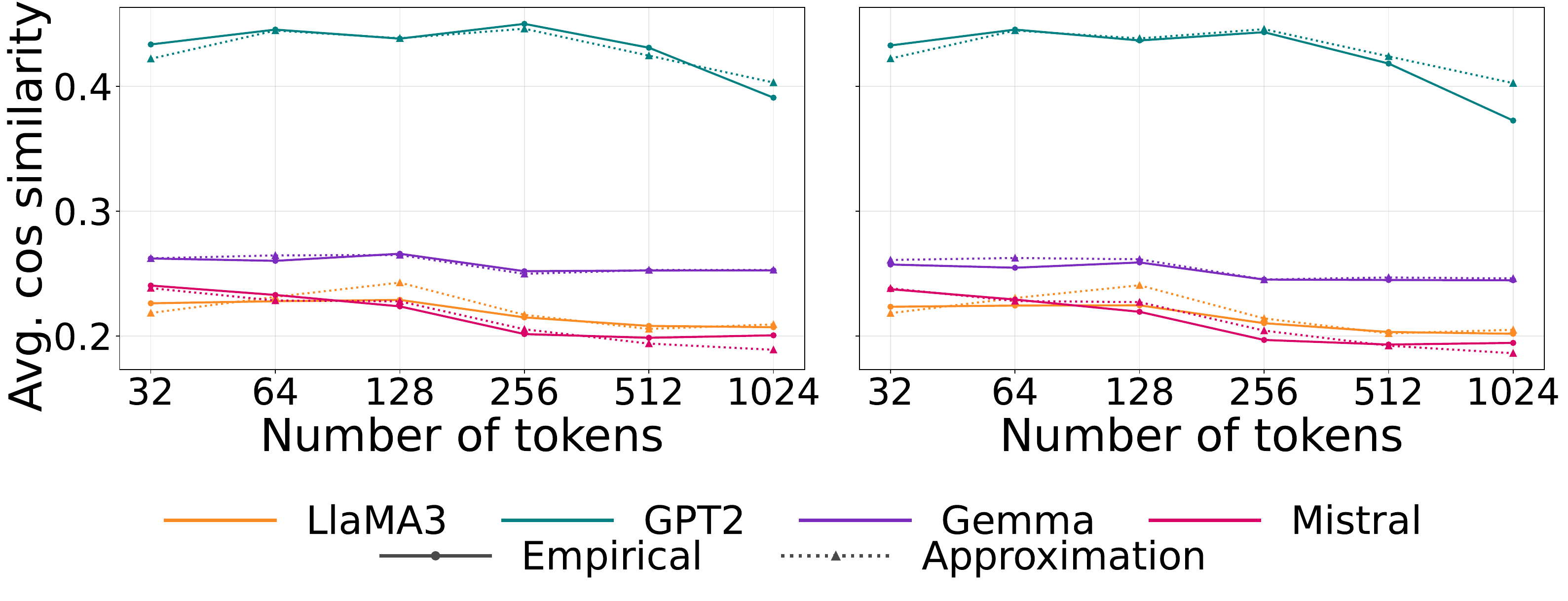}
    \end{subfigure}
    \hfill 
    \begin{subfigure}[b]{0.49\textwidth}
        \centering
        \includegraphics[width=\textwidth]{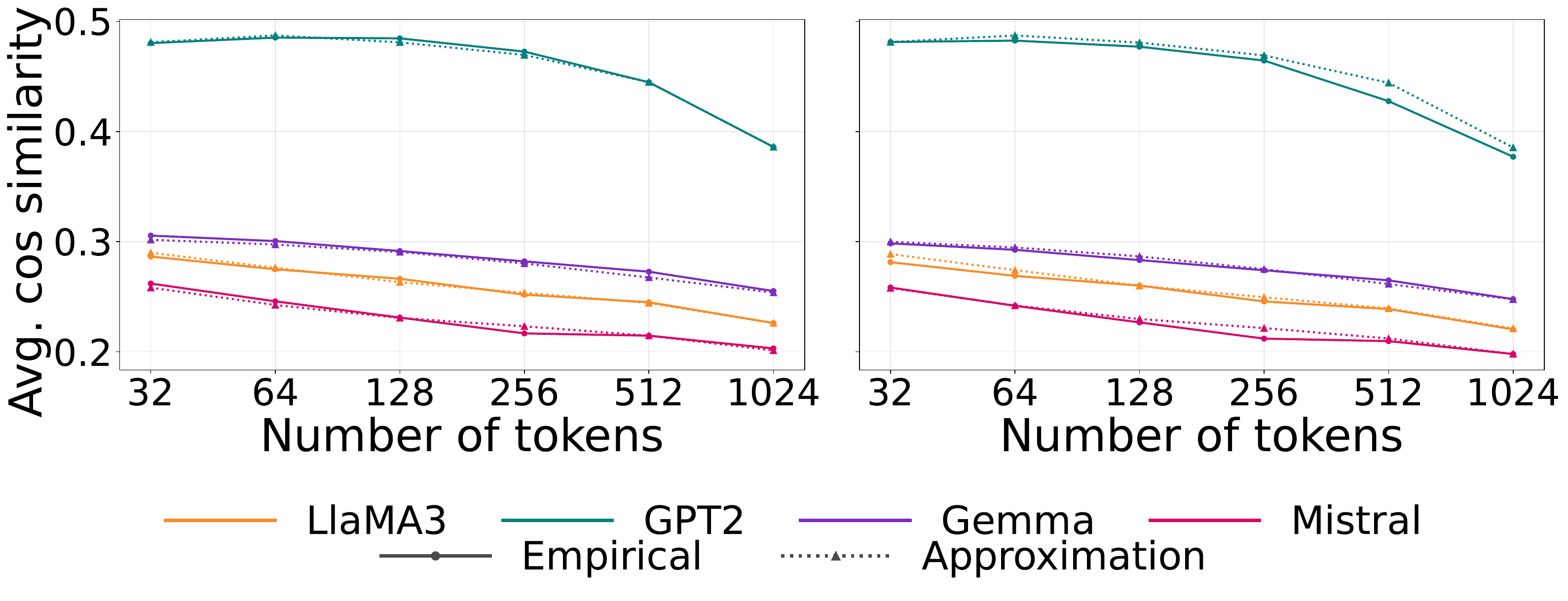}
    \end{subfigure}

    \caption{Empirical average cosine similarity (solid) and theoretical approximation \eqref{eq:approx} (dashed) for $X$ (first plot) and $X+WZ$ (second plot), across models and context lengths, on the CodeParrot dataset (left) and WikiText (right).}
    \label{fig:avgcossim-approx-datasets}
\end{figure}
\begin{figure}[p]
    \centering 
    \begin{subfigure}[b]{0.49\textwidth}
        \centering
        {\small LLaMA3-8B, CodeParrot\par}
        \vspace{2pt}
        \includegraphics[width=\textwidth]{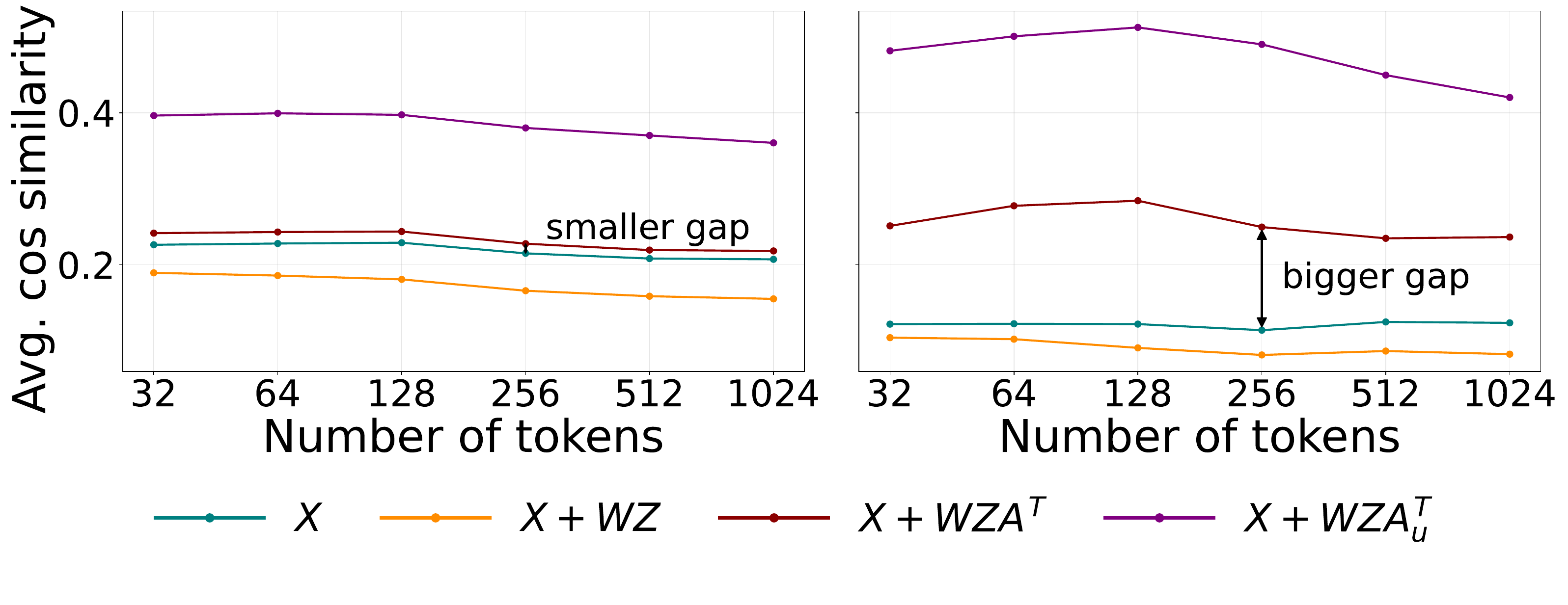}
    \end{subfigure}
    \hfill 
    \begin{subfigure}[b]{0.49\textwidth}
        \centering
        {\small LLaMA3-8B, WikiText\par}
        \vspace{2pt}
        \includegraphics[width=\textwidth]{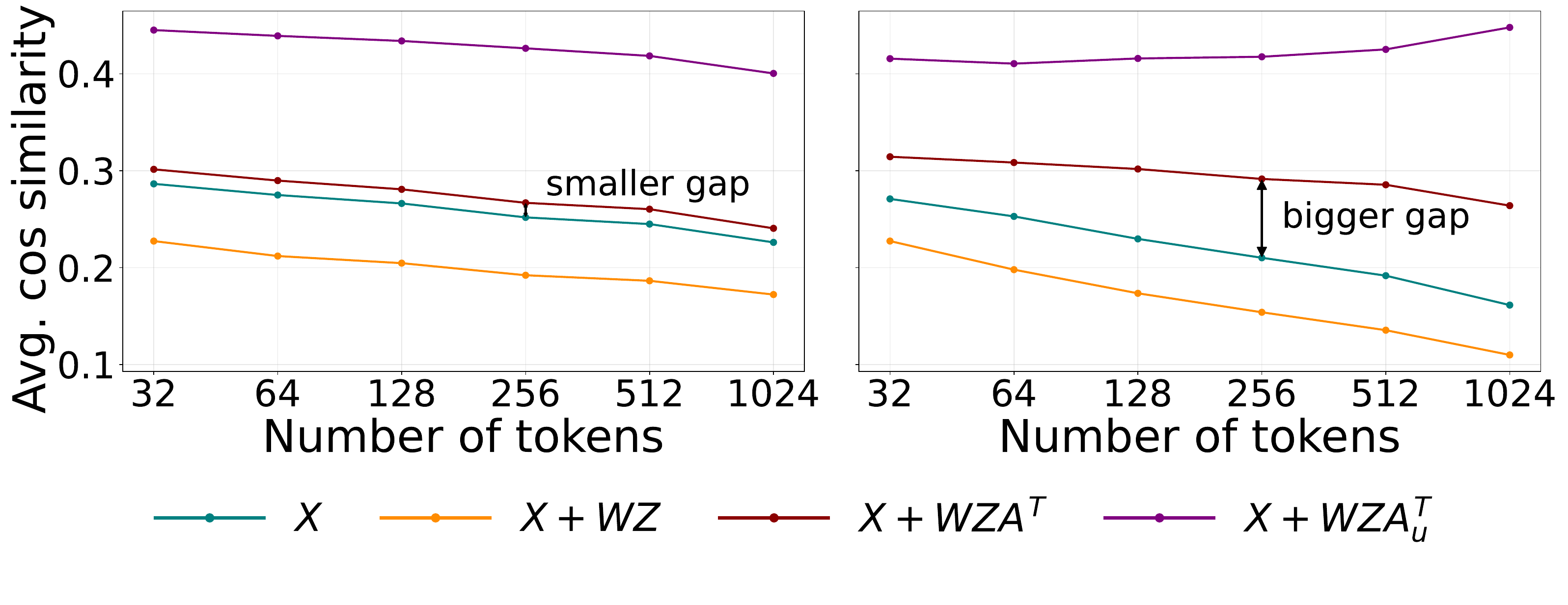}
    \end{subfigure}

    \begin{subfigure}[b]{0.49\textwidth}
        \centering
        {\small GPT2-XL, C4\par}
        \vspace{2pt}
        \includegraphics[width=\textwidth]{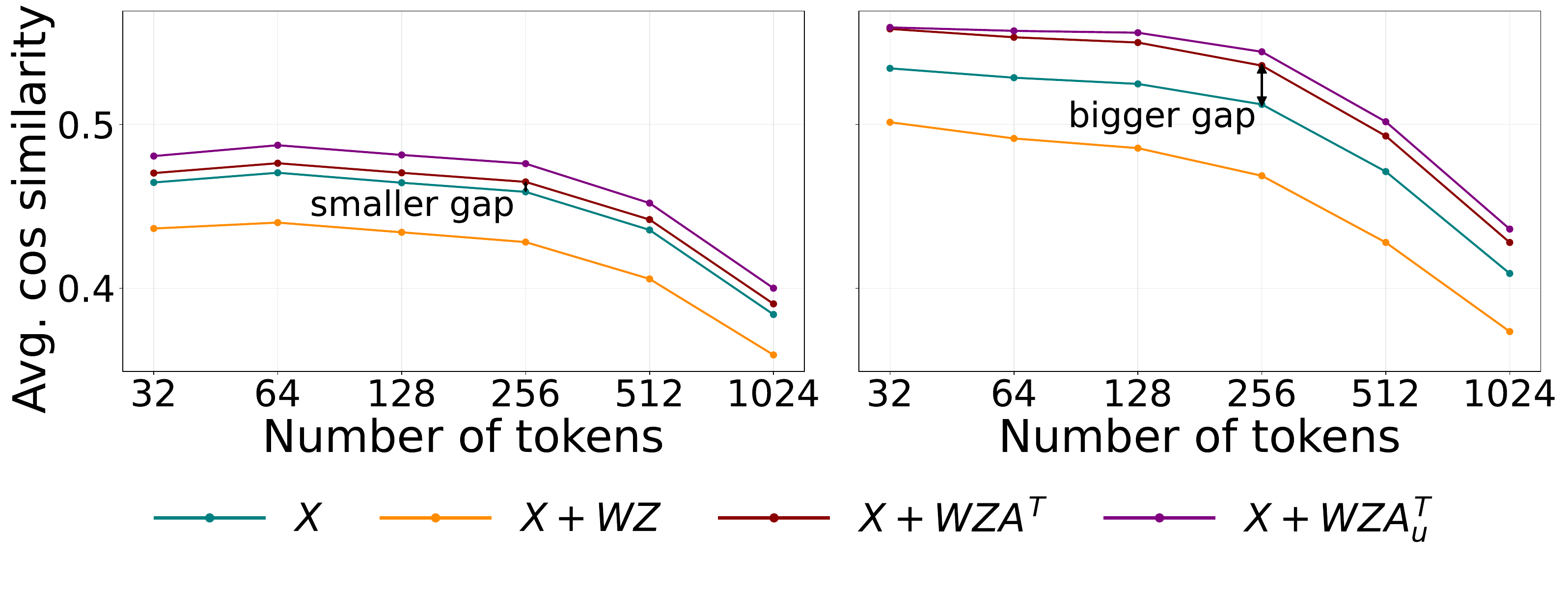}
    \end{subfigure}
    \hfill 
    \begin{subfigure}[b]{0.49\textwidth}
        \centering
        {\small GPT2-XL, CodeParrot\par}
        \vspace{2pt}
        \includegraphics[width=\textwidth]{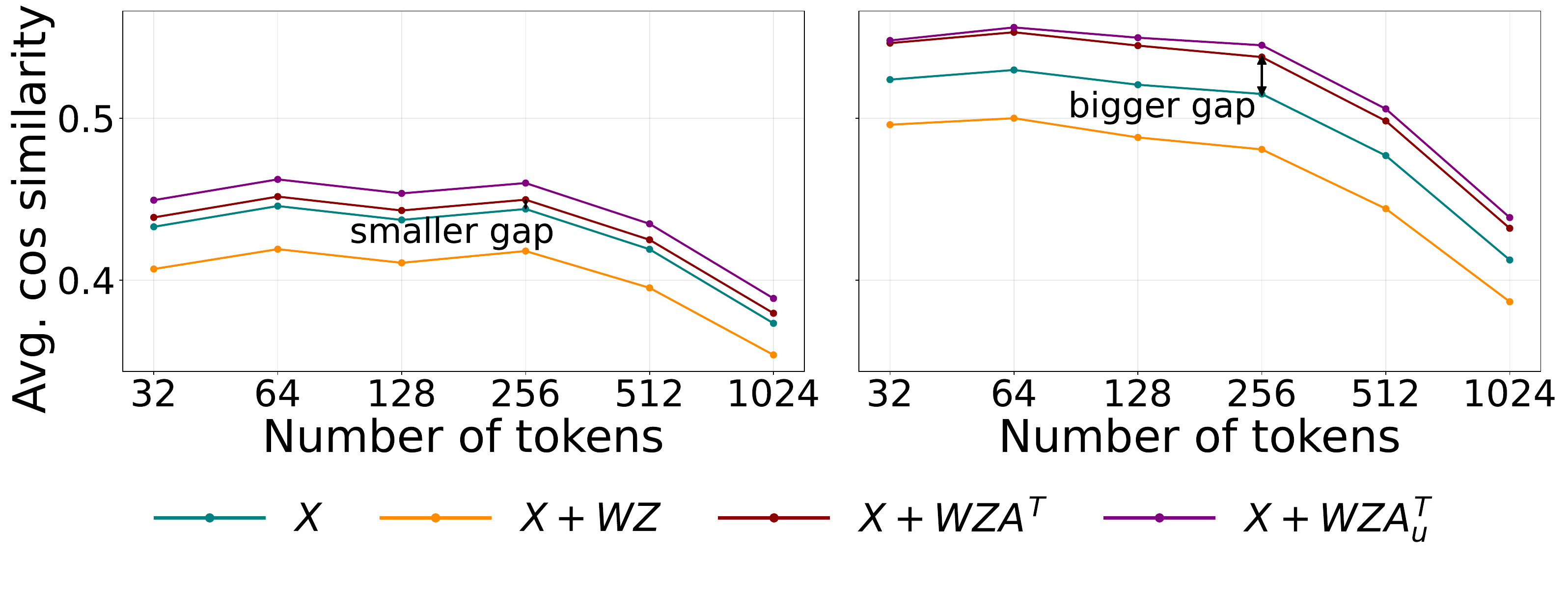}
    \end{subfigure}
    \begin{subfigure}[b]{0.49\textwidth}
        \centering
        {\small GPT2-XL, WikiText\par}
        \vspace{2pt}
        \includegraphics[width=\textwidth]{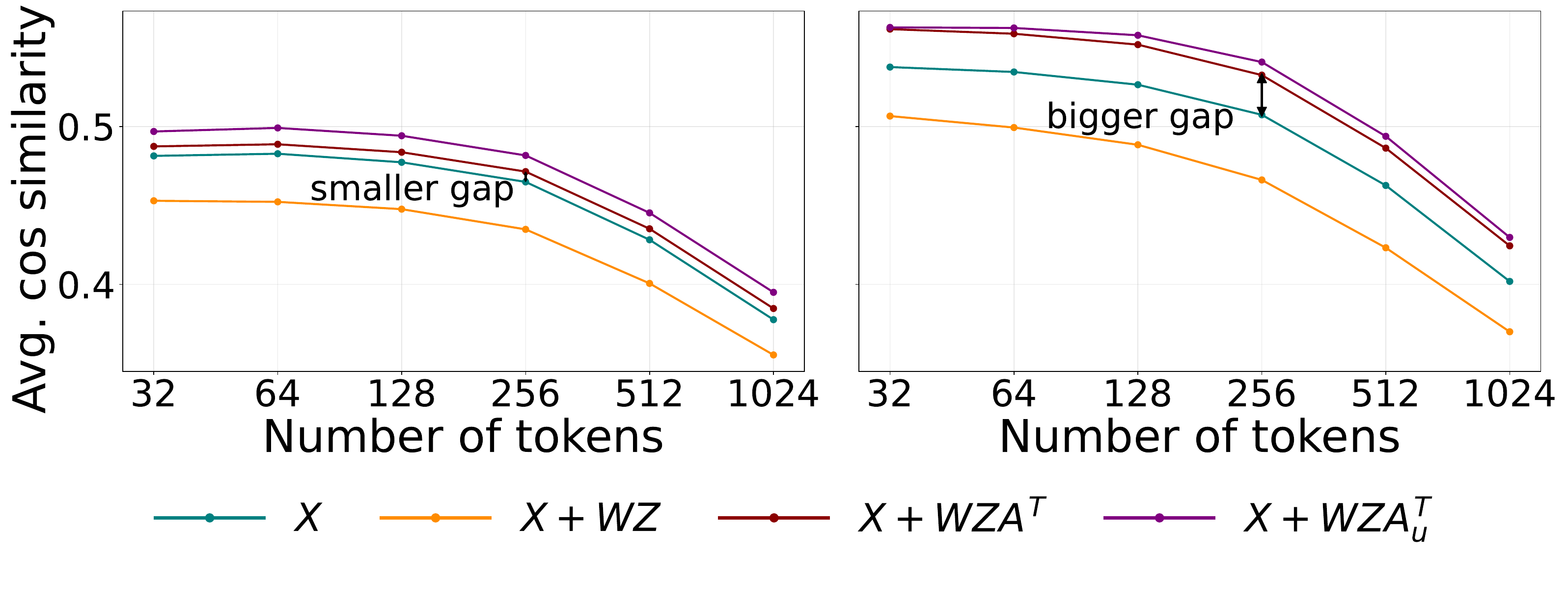}
    \end{subfigure}
    \hfill 
    \begin{subfigure}[b]{0.49\textwidth}
        \centering
        {\small Gemma-7B, C4\par}
        \vspace{2pt}
        \includegraphics[width=\textwidth]{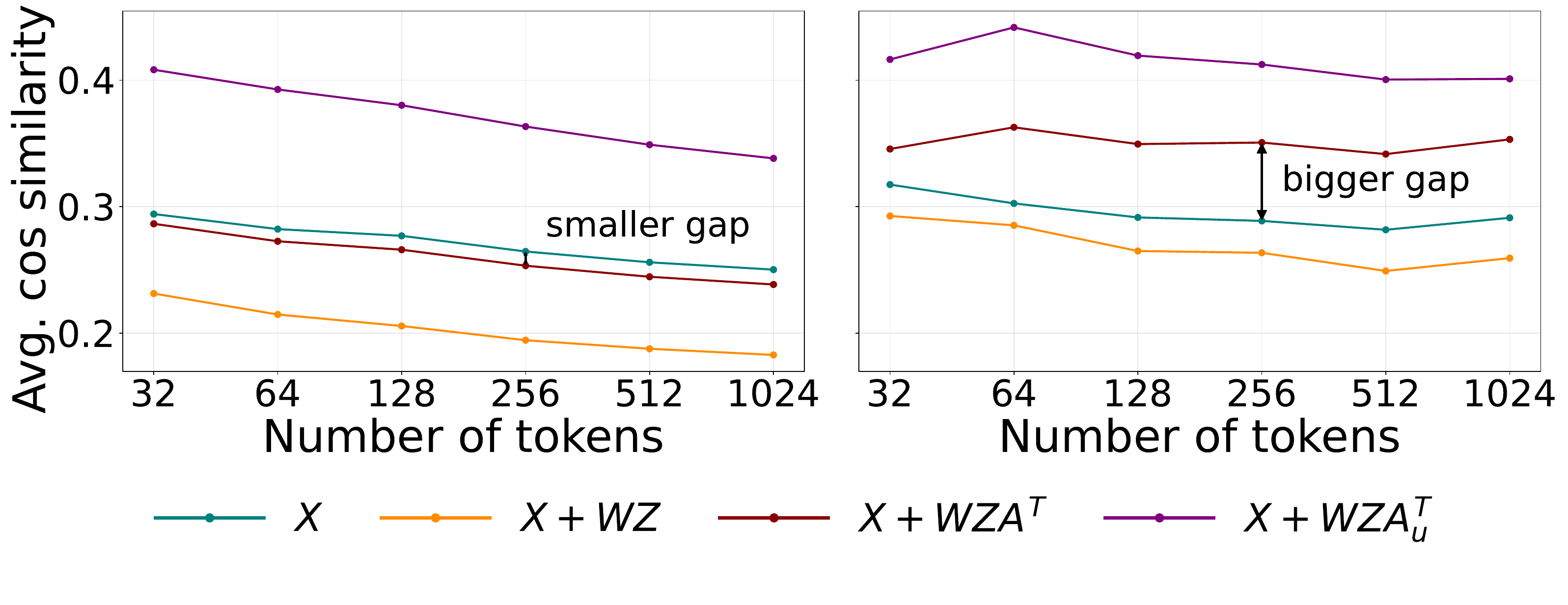}
    \end{subfigure}
    \begin{subfigure}[b]{0.49\textwidth}
        \centering
        {\small Gemma-7B, CodeParrot\par}
        \vspace{2pt}
        \includegraphics[width=\textwidth]{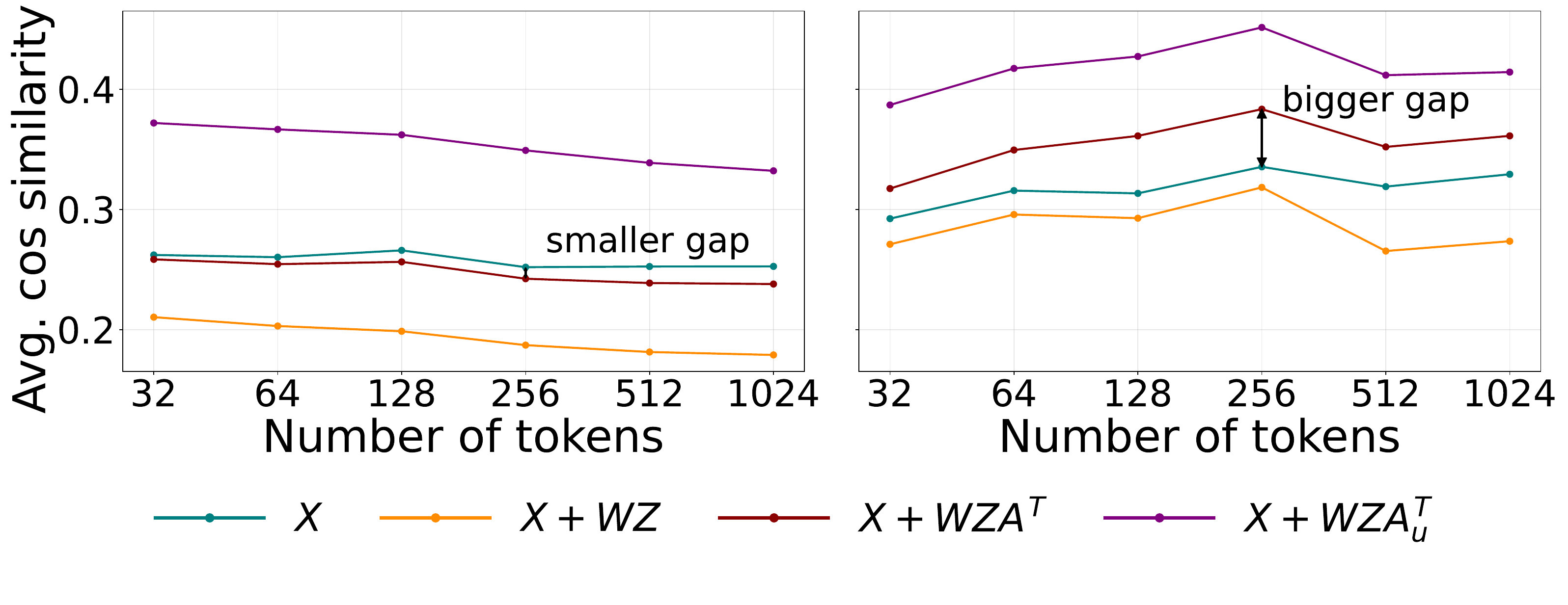}
    \end{subfigure}
    \hfill 
    \begin{subfigure}[b]{0.49\textwidth}
        \centering
        {\small Gemma-7B, WikiText\par}
        \vspace{2pt}
        \includegraphics[width=\textwidth]{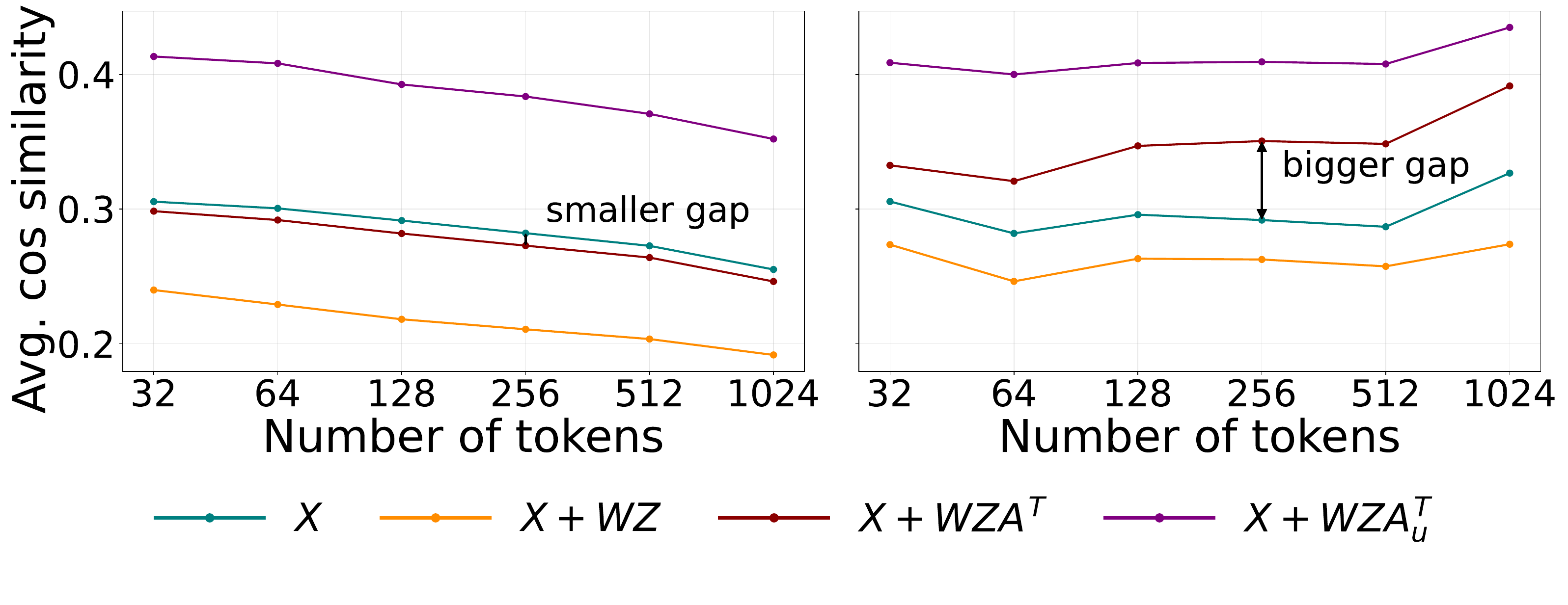}
    \end{subfigure}
    \begin{subfigure}[b]{0.49\textwidth}
        \centering
        {\small Mistral-7B, C4\par}
        \vspace{2pt}
        \includegraphics[width=\textwidth]{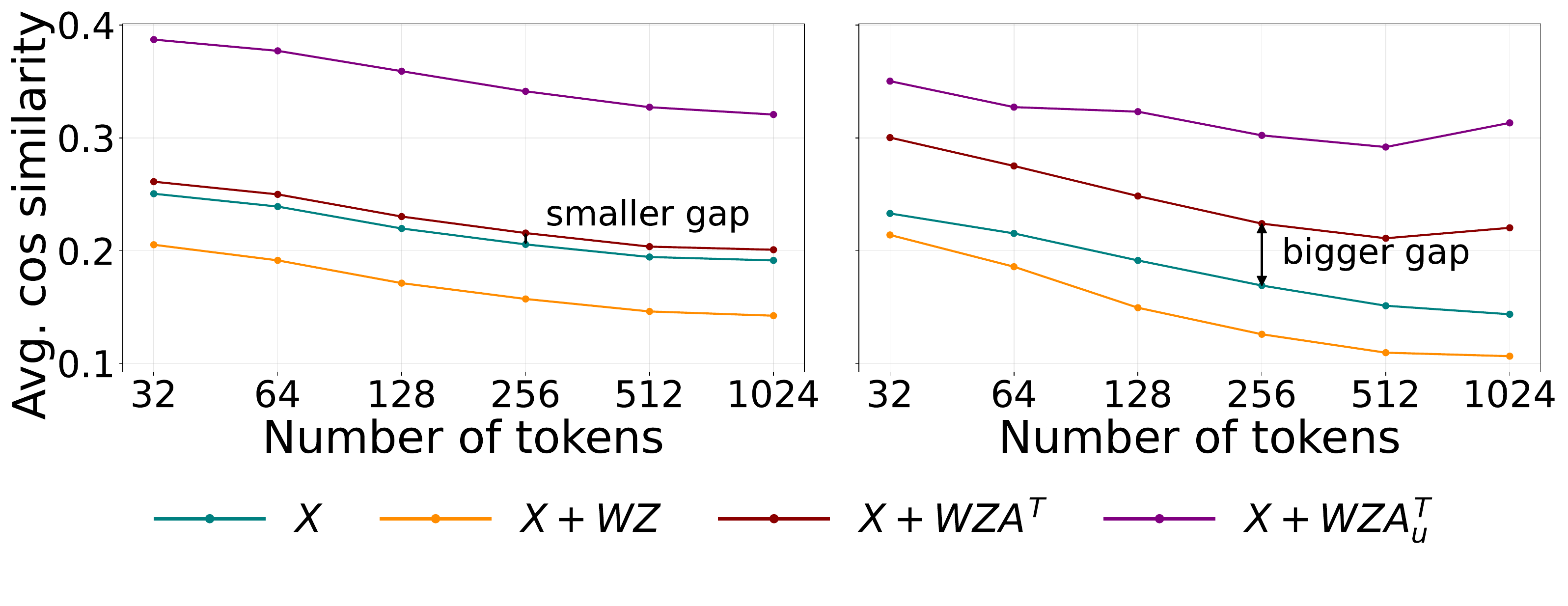}
    \end{subfigure}
    \hfill 
    \begin{subfigure}[b]{0.49\textwidth}
        \centering
        {\small Mistral-7B, CodeParrot\par}
        \vspace{2pt}
        \includegraphics[width=\textwidth]{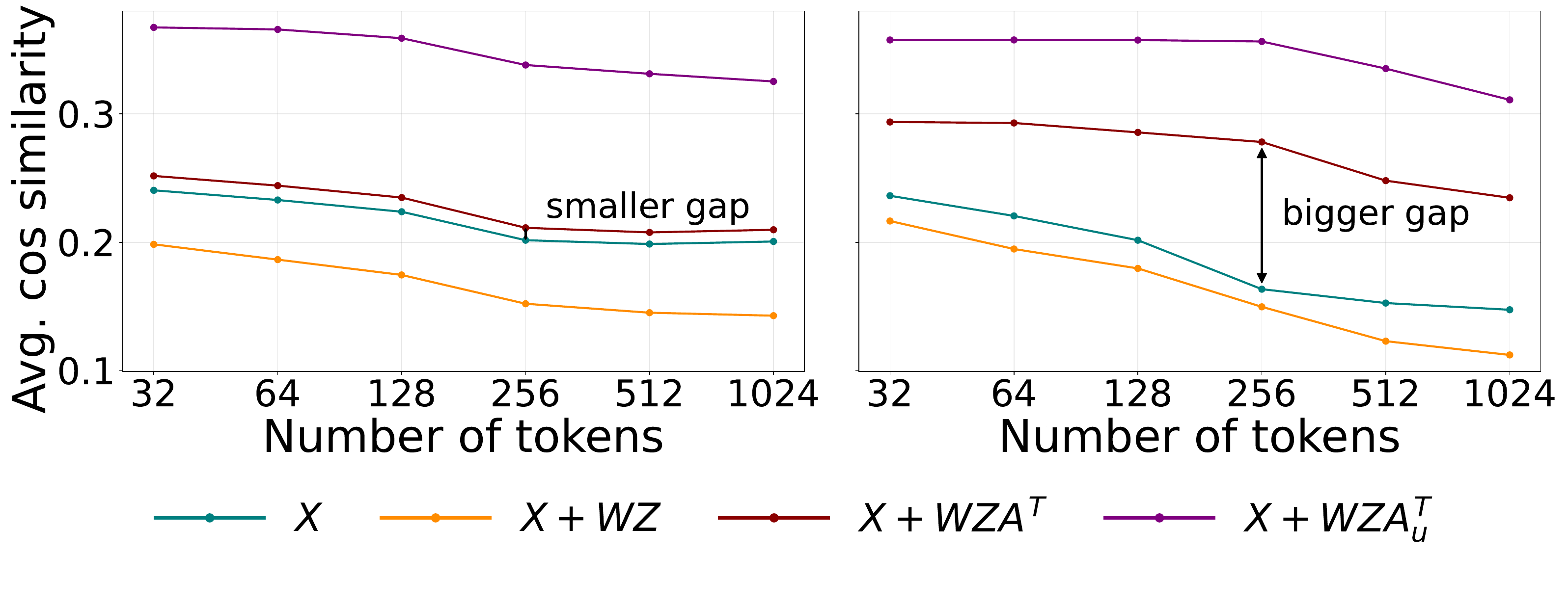}
    \end{subfigure}
    \begin{subfigure}[b]{0.49\textwidth}
        \centering
        {\small Mistral-7B, WikiText\par}
        \vspace{2pt}
        \includegraphics[width=\textwidth]{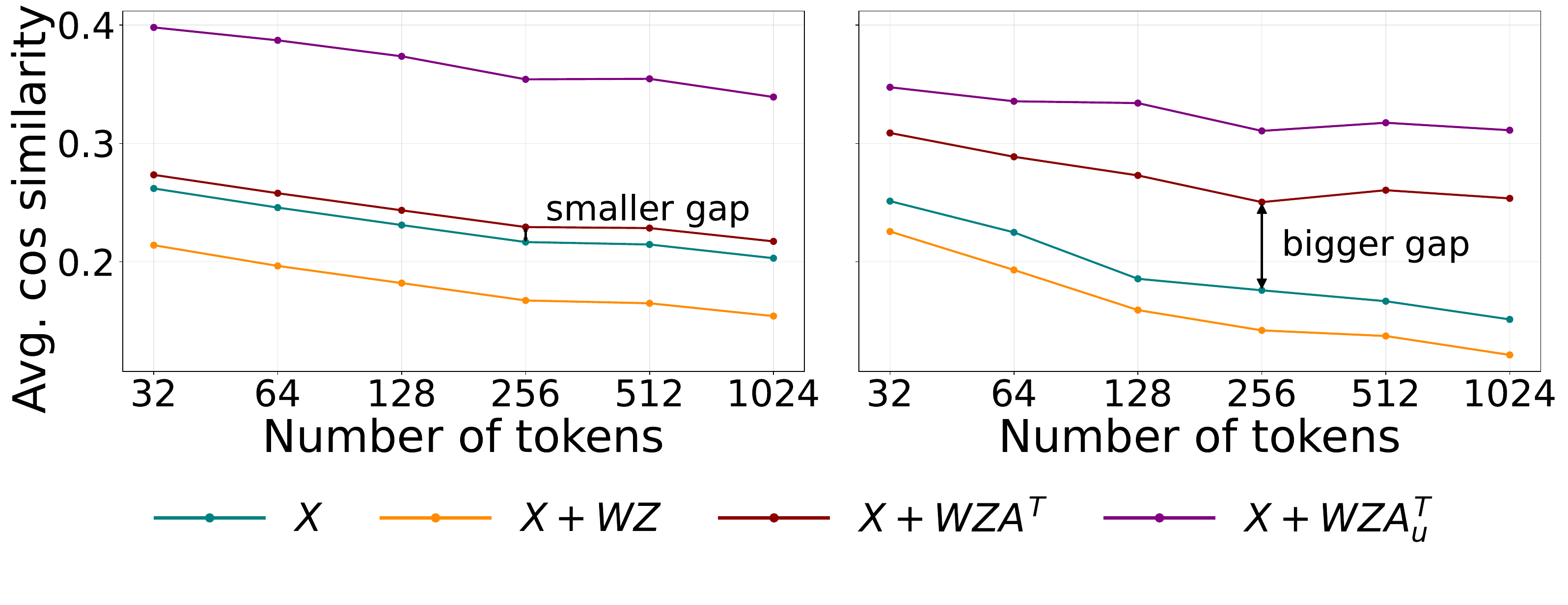}
    \end{subfigure}

    \caption{Empirical average cosine similarity for different models, datasets and context lengths  for different components of the attention step. Results averaged over all heads (first plot) and over heads with uniformity coefficient larger than 0.6 (second plot). Cosine similarity generally increases as the attention matrix becomes denser.}
    \label{fig:avgcossim-empirical-models-datasets}
\end{figure}
Here, we verify that the conclusions of Section~\ref{sec:oversmoothing} hold consistently across both models and datasets. To this end, we repeat the same experiments on LLaMA3-8B, Gemma-7B, GPT2-XL, and Mistral-7B using data from different domains, namely Wikipedia text (WikiText-103~\cite{wikitext}) and code (CodeParrot\footnote{\url{https://huggingface.co/datasets/codeparrot/codeparrot-clean-train}}). First, Figure~\ref{fig:avgcossim-approx-datasets} shows that the approximation provided in Theorem~\ref{thm:incr-lambda-skip} remains close to the empirical average cosine similarity across all models and datasets. Second, Figure~\ref{fig:avgcossim-empirical-models-datasets} shows that we observe the same qualitative behavior described in Section~\ref{sec:oversmoothing} across different models and datasets: \emph{(i)} attention slightly increases cosine similarity on average, and this increase becomes substantially larger when the average is restricted to heads with dense attention patterns; \emph{(ii)} if each head's attention matrix is replaced by causal uniform attention, then the increase in cosine similarity is significantly bigger. GPT2-XL is the only model where the gap with uniform attention is not that big. Note that this is the only model where normalization changes the cosine similarity. We leave for further work to analyze how changing the normalization method affects our analysis.

\begin{figure}[p]
    \centering
    \begin{subfigure}[b]{0.49\textwidth}
        \centering
        {\small LLaMA3-8B, C4\par}
        \vspace{2pt}
        \includegraphics[width=\textwidth]{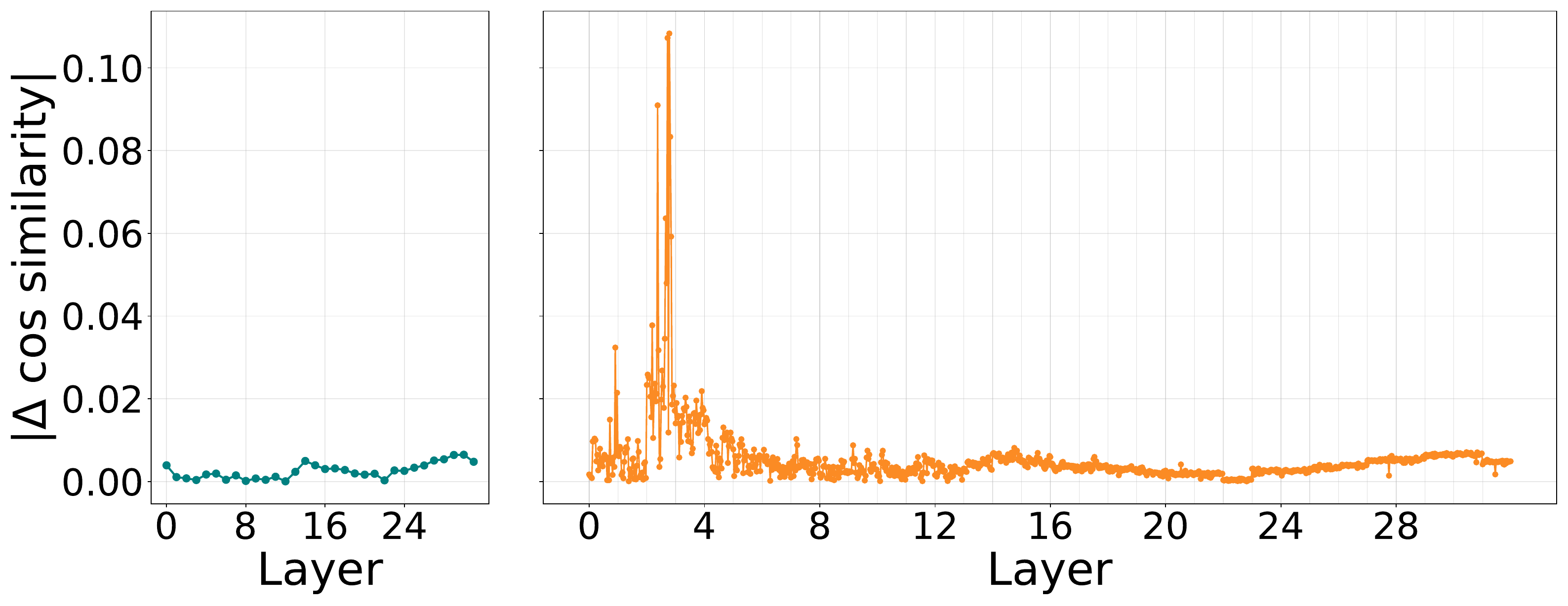}
    \end{subfigure}
    \hfill 
    \begin{subfigure}[b]{0.49\textwidth}
        \centering
        {\small GPT2-XL, CodeParrot\par}
        \vspace{2pt}
        \includegraphics[width=\textwidth]{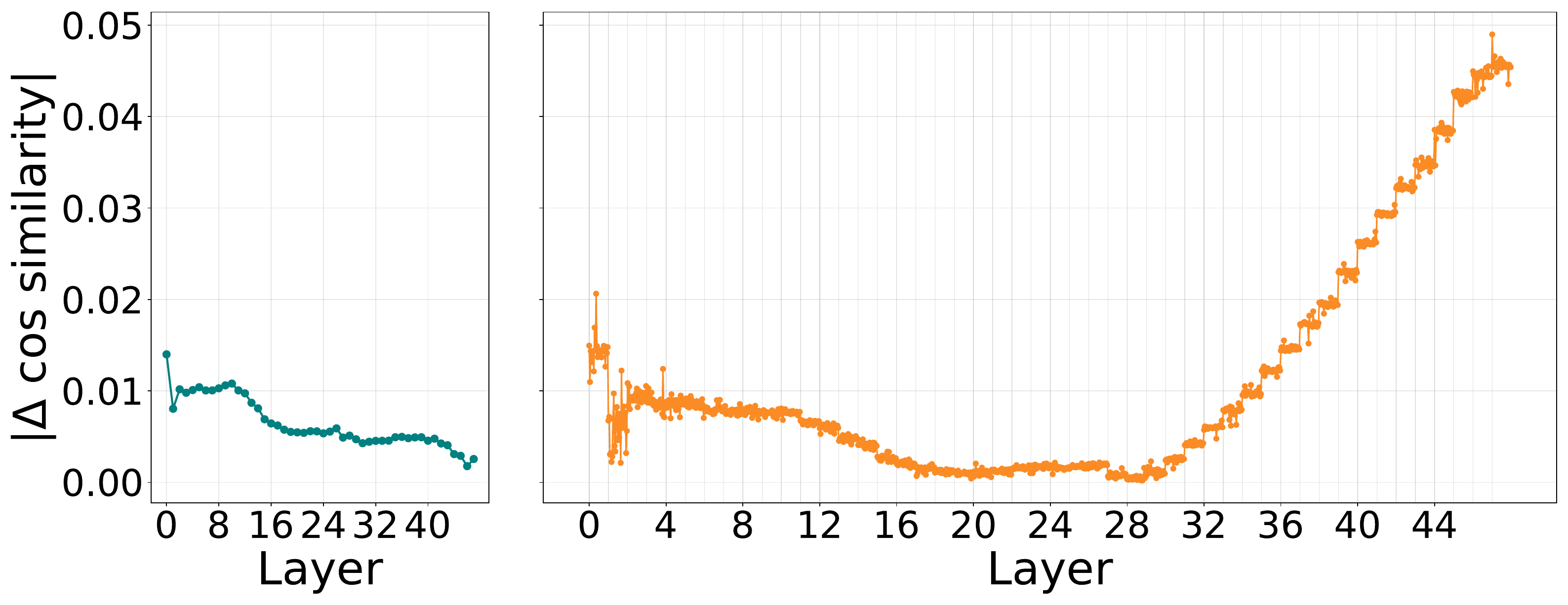}
    \end{subfigure}
     \hfill 
    \begin{subfigure}[b]{0.49\textwidth}
        \centering
        {\small Gemma-7B, WikiText\par}
        \vspace{2pt}
        \includegraphics[width=\textwidth]{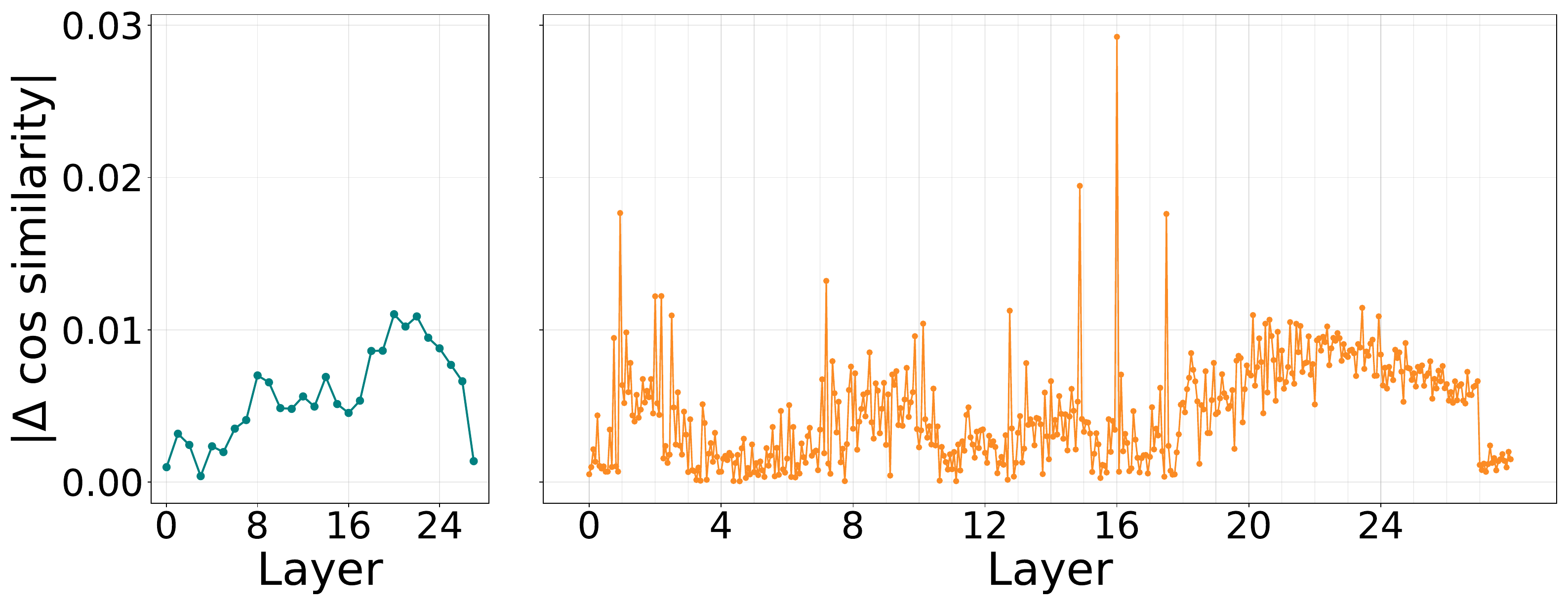}
    \end{subfigure}
    \hfill 
    \begin{subfigure}[b]{0.49\textwidth}
        \centering
        {\small Mistral-7B, C4\par}
        \vspace{2pt}
        \includegraphics[width=\textwidth]{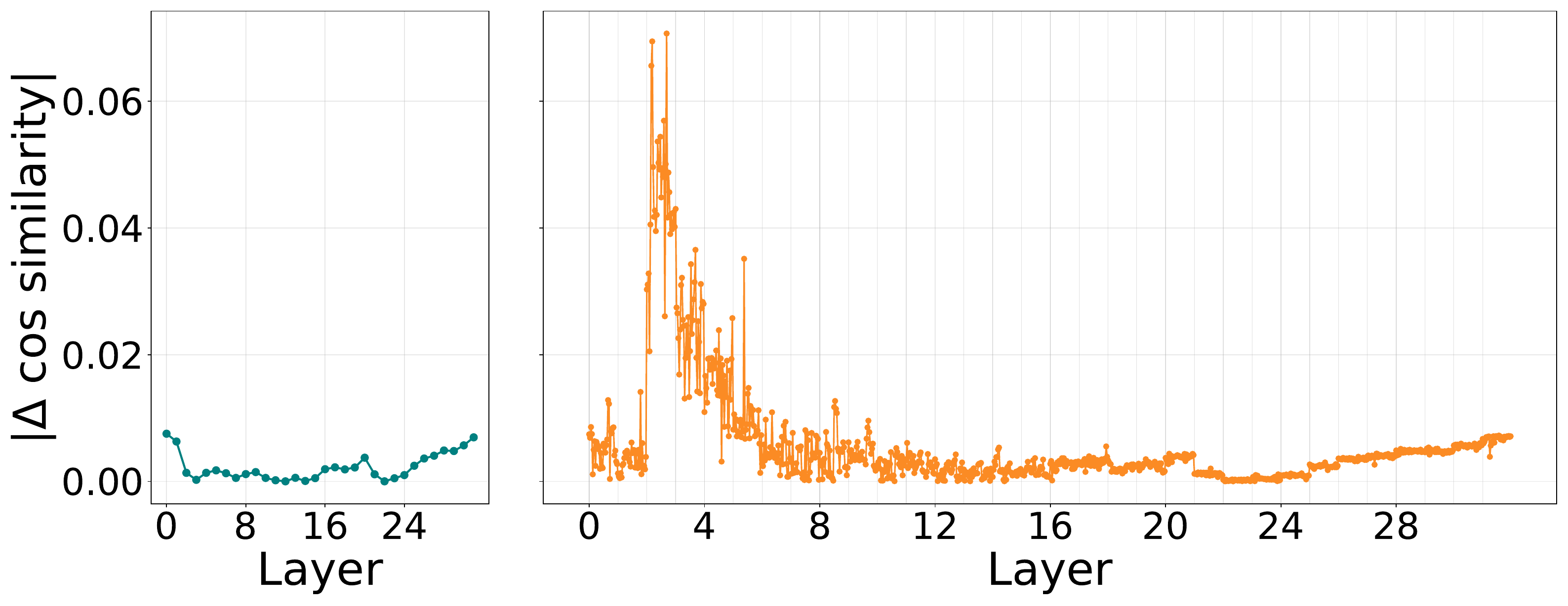}
    \end{subfigure}

    \caption{Absolute difference between the empirical average cosine similarity and the theoretical approximation \eqref{eq:approx} for individual heads, for \(X\) (blue plot) and \(X+WZ\) (orange plot), across different models and datasets at context length 512. The x-axis indexes the heads in order, from the first to the last layer.}
    \label{fig:cossim-approx-ind-heads}
\end{figure}

\begin{figure}[p]
    \centering
    \begin{subfigure}[b]{0.49\textwidth}
        \centering
        {\small LLaMA3-8B, C4\par}
        \vspace{2pt}
        \includegraphics[width=\textwidth]{figures/c4_llama_difference_theorem_cond_N_512.pdf}
    \end{subfigure}
    \hfill 
    \begin{subfigure}[b]{0.49\textwidth}
        \centering
        {\small GPT2-XL, CodeParrot\par}
        \vspace{2pt}
        \includegraphics[width=\textwidth]{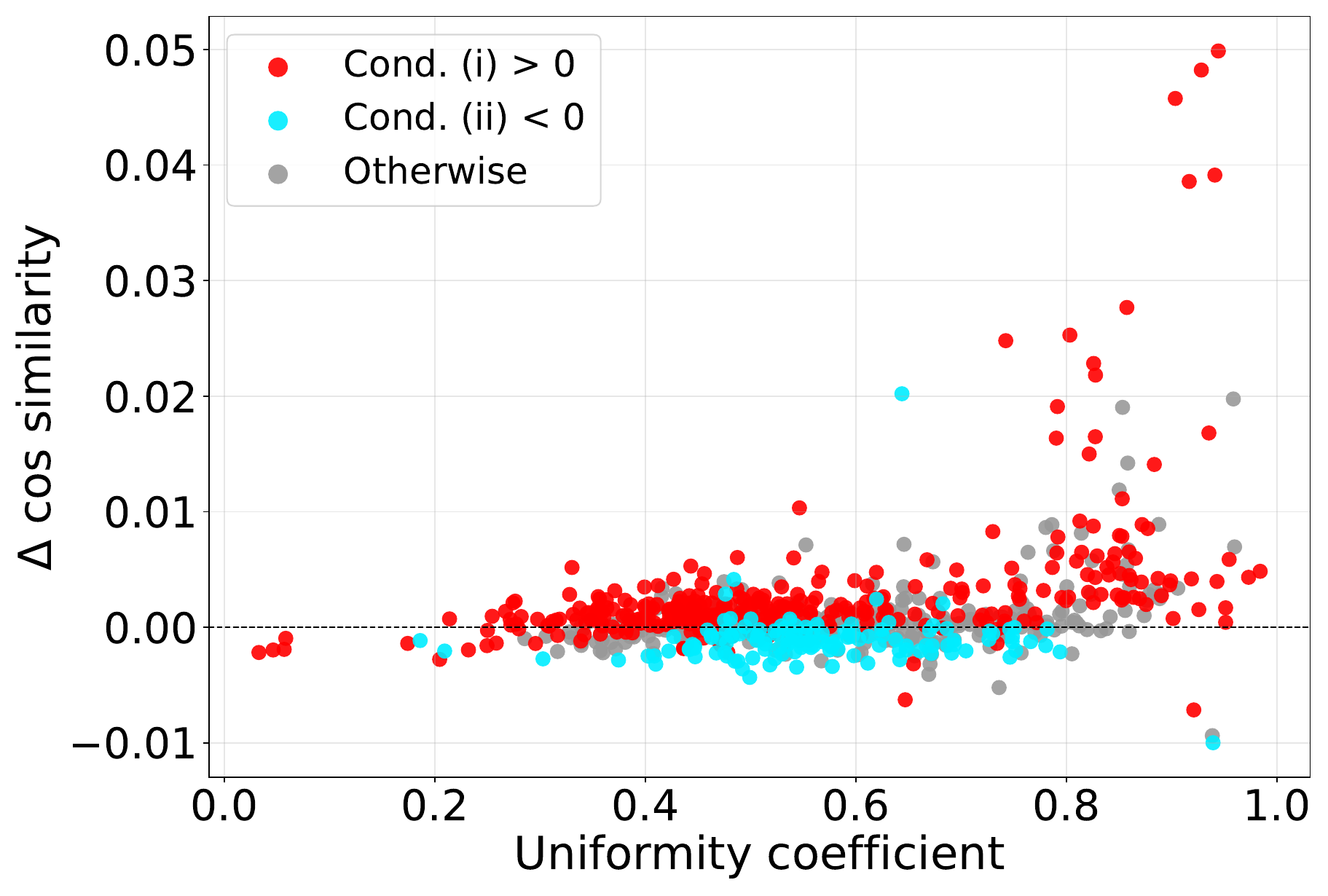}
    \end{subfigure}
     \hfill 
    \begin{subfigure}[b]{0.49\textwidth}
        \centering
        {\small Gemma-7B, WikiText\par}
        \vspace{2pt}
        \includegraphics[width=\textwidth]{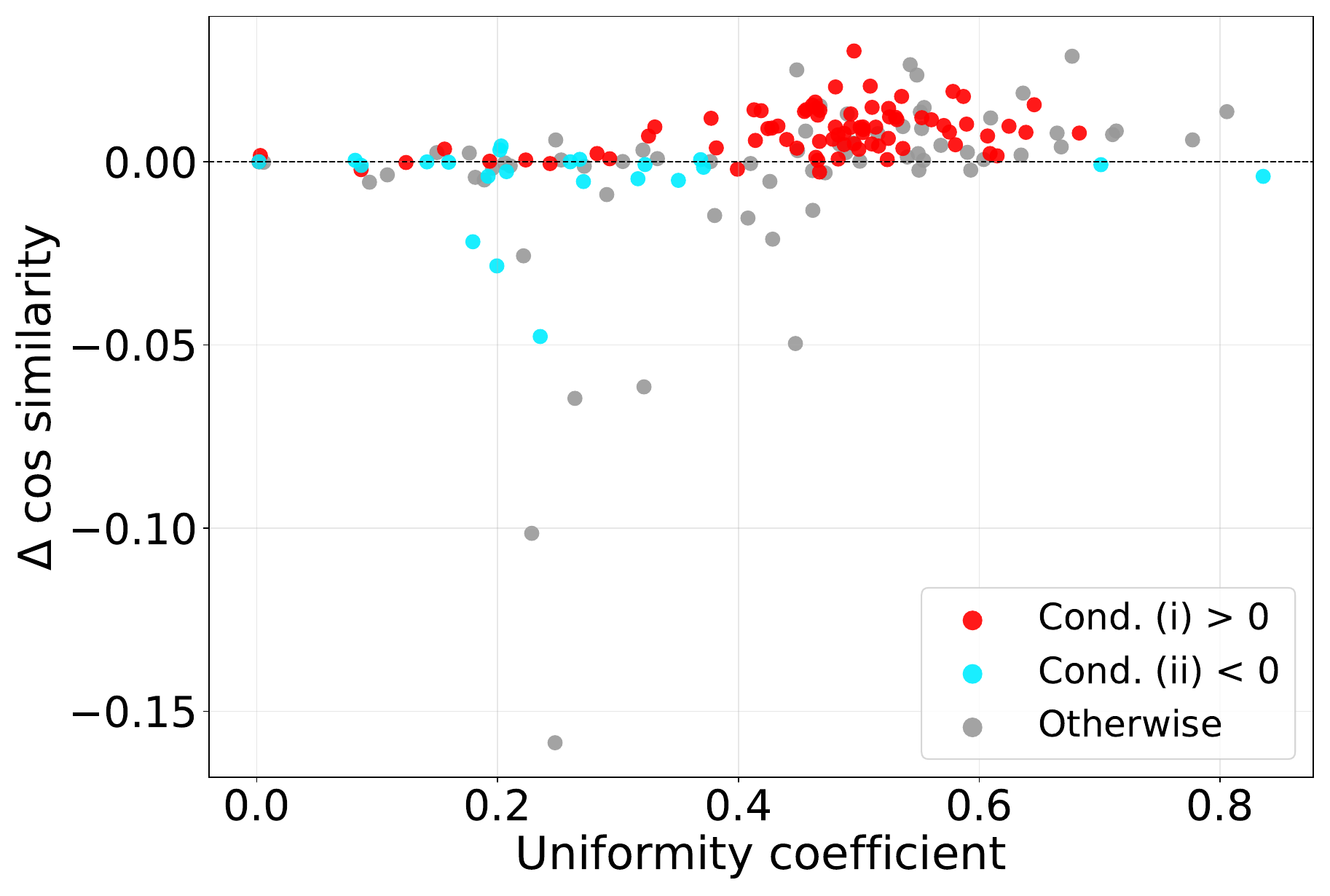}
    \end{subfigure}
    \hfill 
    \begin{subfigure}[b]{0.49\textwidth}
        \centering
        {\small Mistral-7B, C4\par}
        \vspace{2pt}
        \includegraphics[width=\textwidth]{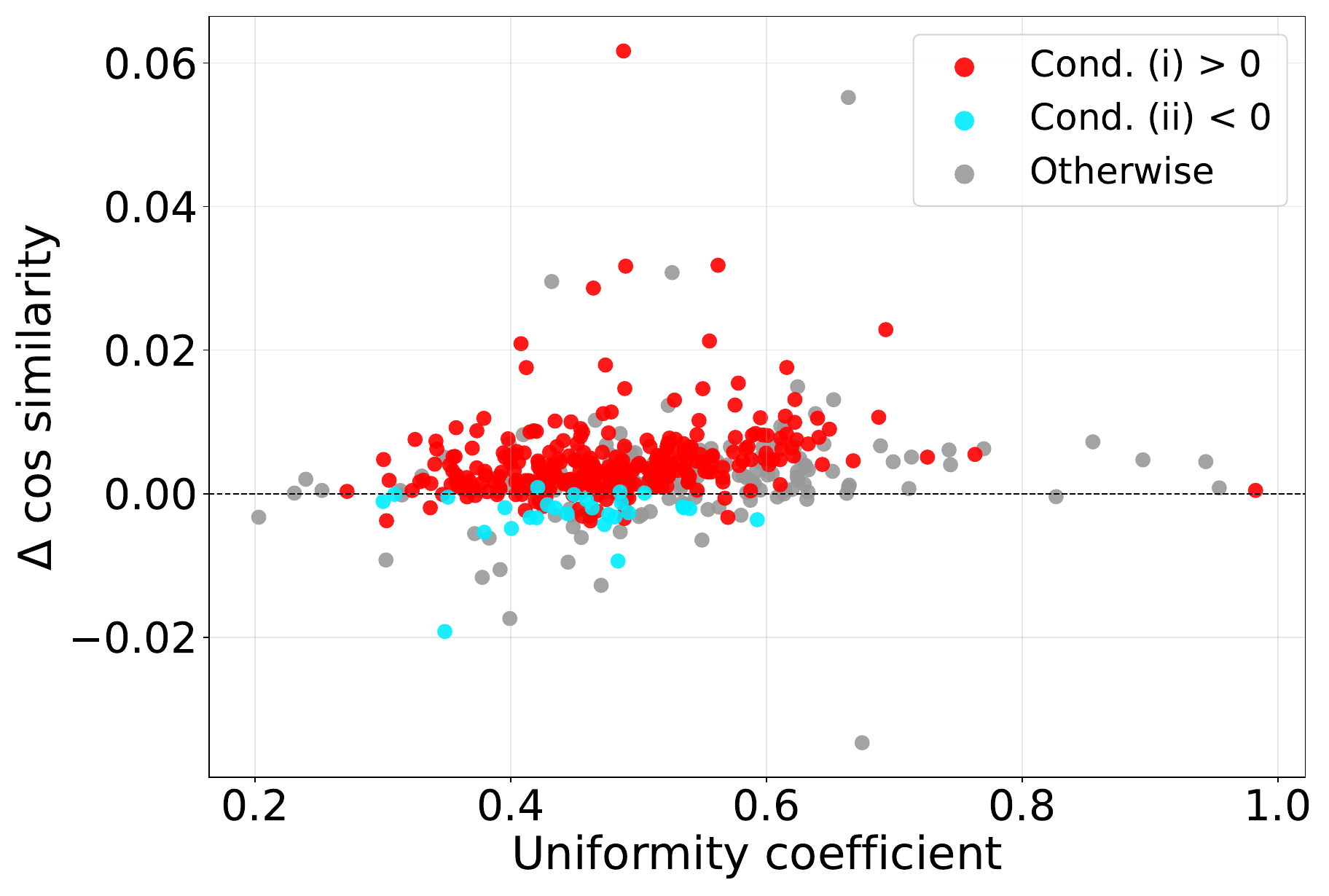}
    \end{subfigure}

   \caption{\(\operatorname{avg.\ cos\ sim}(X+WZA^\top)-\operatorname{avg.\ cos\ sim}(X+WZ)\) for each head of different models and datasets at context length 512. Heads satisfying \(\beta\Tr(BW)>0\) (red) tend to yield positive values, whereas heads satisfying \(\beta\Tr(BW)+\Tr(BW^\top W)<0\) (blue) tend to have negative values. Heads satisfying neither condition are shown in gray.}
    \label{fig:trace-cond-theorem}
\end{figure}

\paragraph{Approximation is accurate for individual heads.} 

We verify that the approximation provided by~\eqref{eq:approx-cos-sim} remains close to the empirical average cosine similarity not only after averaging across heads, but also for most individual heads. Figure~\ref{fig:cossim-approx-ind-heads} shows, for different models and datasets and a context length of 512, the absolute difference between the empirical value and its theoretical approximation for each head. Overall, this difference stays close to 0 for many heads and never exceeds $0.1$, indicating that the approximation is accurate for most heads.

\paragraph{How often do the conditions of Theorem~\ref{thm:incr-lambda-skip} hold?}

Here, we examine to what extent Theorem~\ref{thm:incr-lambda-skip} explains the oversmoothing behavior observed in trained transformers. The theorem provides sufficient conditions under which the approximation of the average cosine similarity after adding a head output to the residual stream, $\hat\rho(Y(\lambda))$, is strictly increasing or decreasing in $\lambda$. We study how often these conditions are satisfied in practice, and whether they are consistent with the token mixing observed empirically.

To this end, we consider the changes in average cosine similarity between different attention steps: (i) $\operatorname{avg. cos sim}(X+WZA^T)-\operatorname{avg. cos sim}(X+WZ)$ (Figure~\ref{fig:trace-cond-theorem}); and (ii) $\operatorname{avg. cos sim}(X+WZA_u^\top)-\operatorname{avg. cos sim}(X+WZA^\top)$ (Figure~\ref{fig:trace-cond-theorem-2}). 

To isolate the effect of sparsity, in our theoretical model, the attention matrix is parametrized as $(1-\lambda)I+\lambda A_{u}$ where $A_u$ is the $T\times T$ uniform causal attention matrix and $\lambda\in[0,1]$. Since a general attention matrix need not lie on this interpolation path, it is not immediate that the theorem should explain the empirical behavior. Nevertheless, one may heuristically interpret $\operatorname{avg. cos sim}(X+WZA^T)-\operatorname{avg. cos sim}(X+WZ)$ as the change in cosine similarity between $\lambda=0$ and some $\lambda_A\in[0,1]$, and $\operatorname{avg. cos sim}(X+WZA_u^\top)-\operatorname{avg. cos sim}(X+WZA^\top)$ as the difference between $\lambda=\lambda_A\in[0,1]$ and $\lambda=1$. In this interpretation, attention matrices with larger uniformity scores should correspond to larger values of $\lambda_A$. Since sink attention patterns are far from this interpolation, we exclude them from this analysis.

Under this interpretation, Theorem~\ref{thm:incr-lambda-skip} predicts the following behavior:\begin{itemize}
    \item Both quantities should be positive when $\beta\Tr(BW)> 0$ ($\hat\rho(Y(\lambda))$ strictly increases with $\lambda$).
    \item Both quantities should be negative when $\beta\Tr(BW)+\Tr(BW^\top W)<0$ ($\hat\rho(Y(\lambda))$ strictly decreases with $\lambda$).
\end{itemize}
The figures show that this prediction is consistent with the behavior of most heads in trained transformers. In particular, most heads satisfying $\beta\Tr(BW)> 0$ (red) exhibit an increase in cosine similarity after applying attention, and this increase remains smaller than under uniform attention. Likewise, for most heads satisfying $\beta\Tr(BW)+\Tr(BW^\top W)<0$ (blue), uniform attention induces less token mixing than the head’s actual attention matrix, which in turn induces less mixing than no attention. For heads that satisfy neither condition (gray), both behaviors are observed, as expected since the theorem makes no prediction in that case. The main exception is Figure~\ref{fig:trace-cond-theorem-2} for GPT2-XL on CodeParrot, where many red heads exhibit less token mixing under uniform attention than under the learned attention matrix; however, these heads are already highly uniform, with uniformity coefficient above 0.8.

Finally, heads satisfying the second condition are rare in all models except GPT2-XL, and even there they remain a minority. This helps explain the dominant empirical trend: in trained transformers, denser attention patterns typically lead to more token mixing.

\paragraph{A closer look into uniformity and oversmoothing.}

Figure~\ref{fig:unif-vs-relChange} shows the relative change in cosine similarity between token representations before each head update ($X$) and after it ($X+WZA^\top$), with heads ordered by uniformity coefficient. Overall, we observe a positive correlation between the increase in cosine similarity and the uniformity of the attention heads. Across all models, this relative change becomes noticeably more positive once the uniformity coefficient exceeds a threshold of roughly 0.6. For some models, however, heads with very high uniformity exhibit smaller changes than heads with lower, though still high, uniformity. Since there are only few such heads, it is difficult to draw  conclusions about this behavior.
\paragraph{Heads where greater uniformity does not increase oversmoothing.}

Finally, Figure~\ref{fig:uncommon-heads} shows several examples of heads that do not follow the typical behavior described in this work. In particular, it includes cases where uniform attention yields less or equal token mixing than the learned head attention, even at relatively small uniformity coefficients (plots 1, 2, and 4), as well as a case where uniform attention produces less token mixing than applying no attention at all (plot 3).

\begin{figure}[t]
    \centering
    \begin{subfigure}[b]{0.49\textwidth}
        \centering
        {\small LLaMA3-8B, C4\par}
        \vspace{2pt}
        \includegraphics[width=\textwidth]{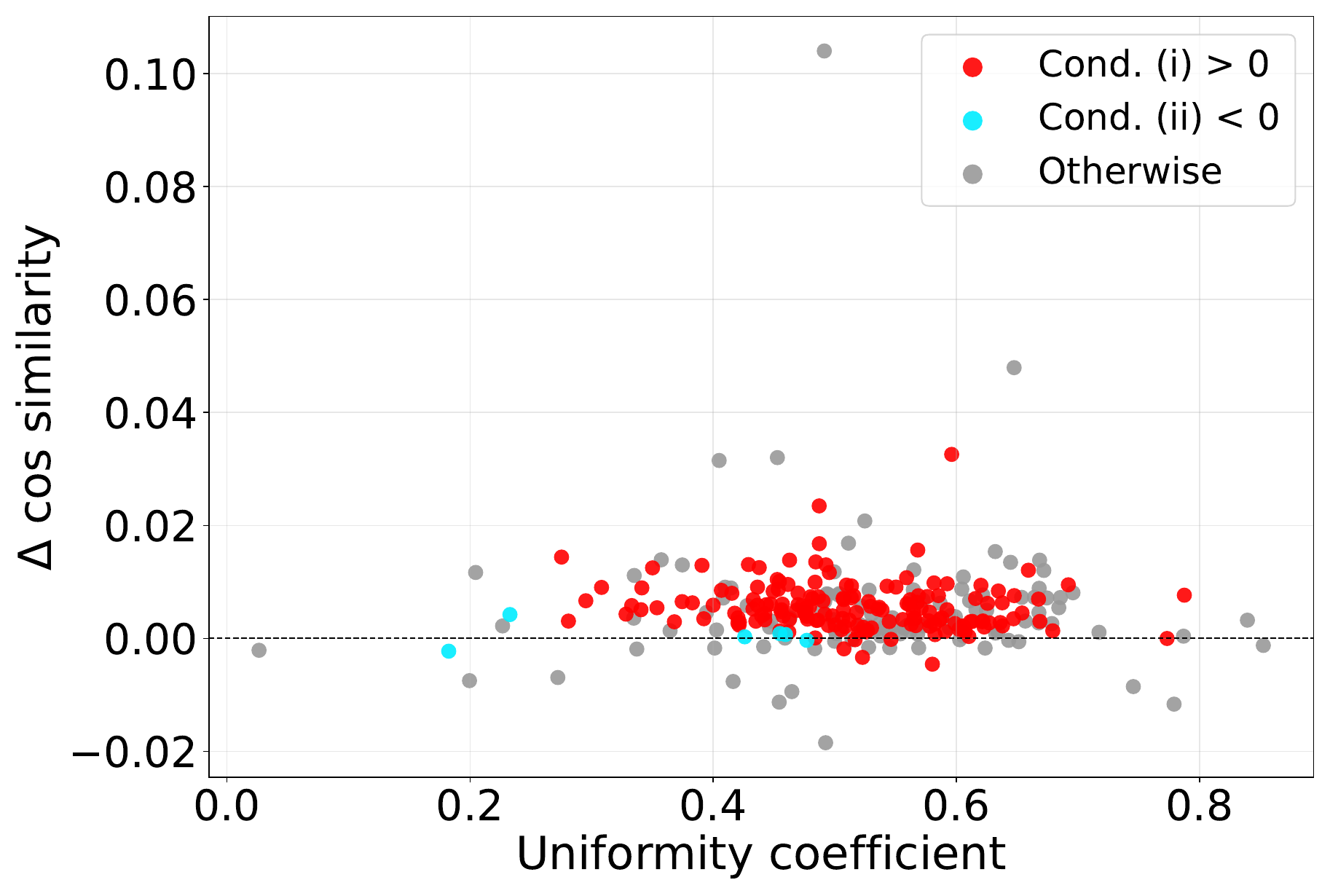}
    \end{subfigure}
    \hfill 
    \begin{subfigure}[b]{0.49\textwidth}
        \centering
        {\small GPT2-XL, CodeParrot\par}
        \vspace{2pt}
        \includegraphics[width=\textwidth]{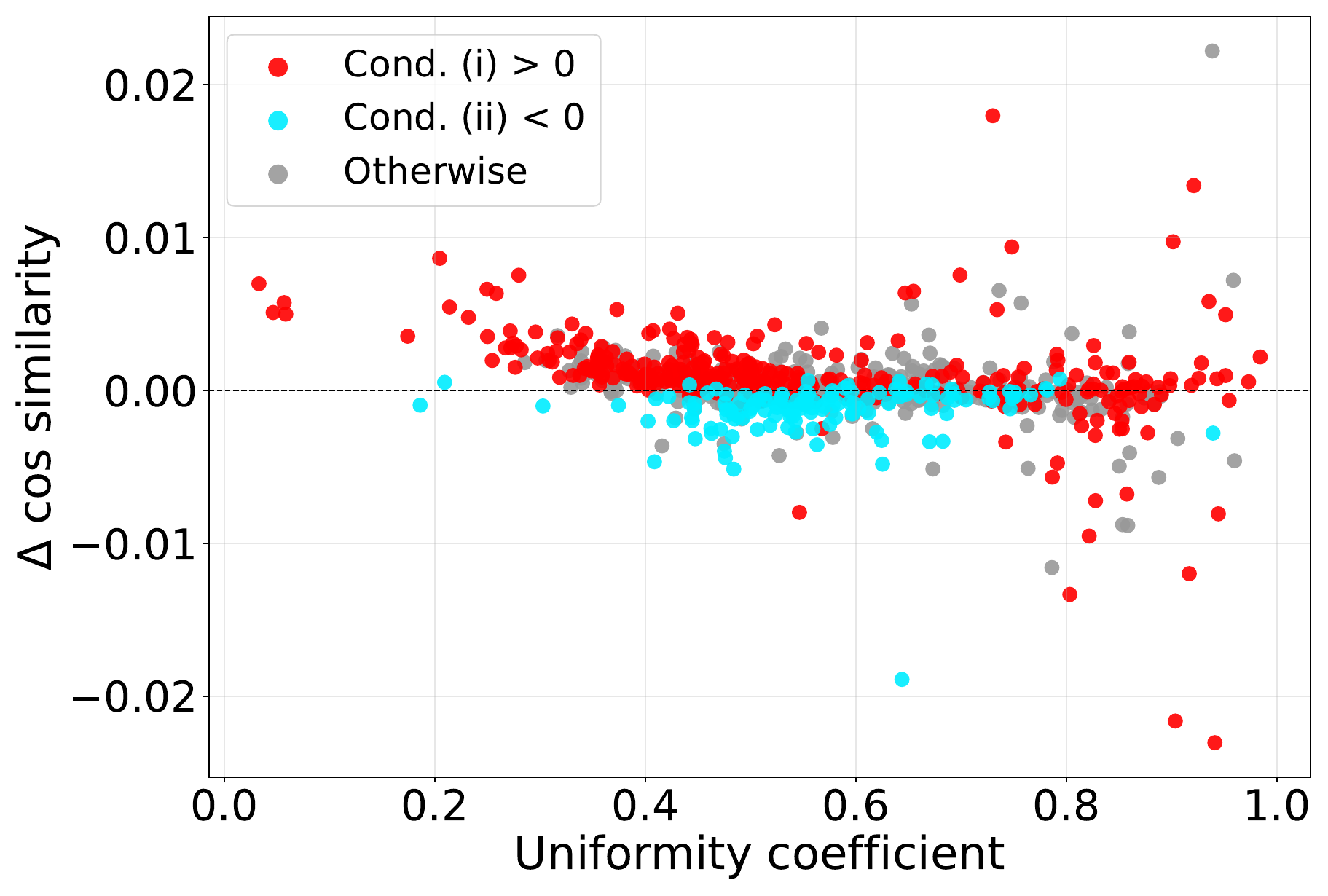}
    \end{subfigure}
     \hfill 
    \begin{subfigure}[b]{0.49\textwidth}
        \centering
        {\small Gemma-7B, WikiText\par}
        \vspace{2pt}
        \includegraphics[width=\textwidth]{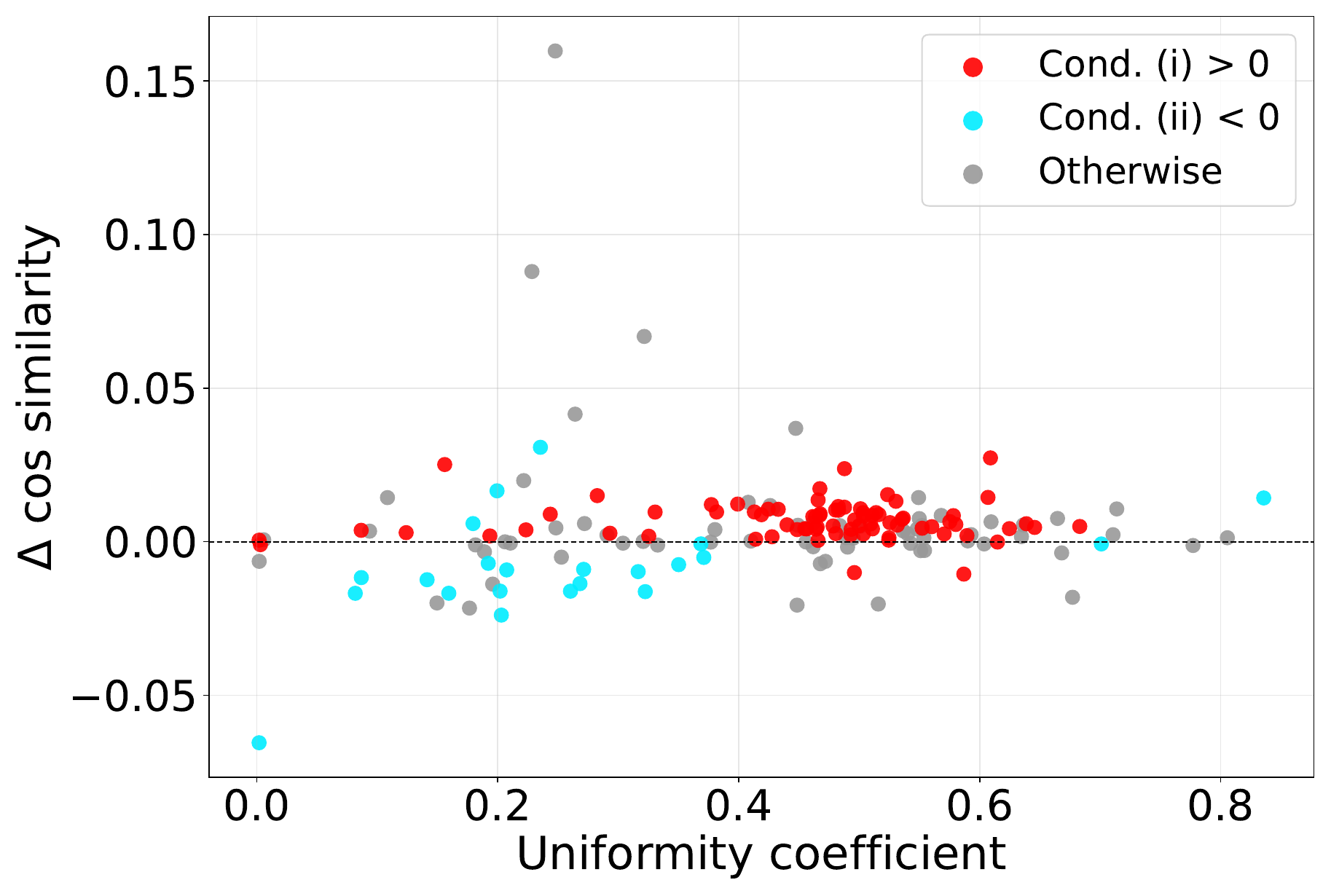}
    \end{subfigure}
    \hfill 
    \begin{subfigure}[b]{0.49\textwidth}
        \centering
        {\small Mistral-7B, C4\par}
        \vspace{2pt}
        \includegraphics[width=\textwidth]{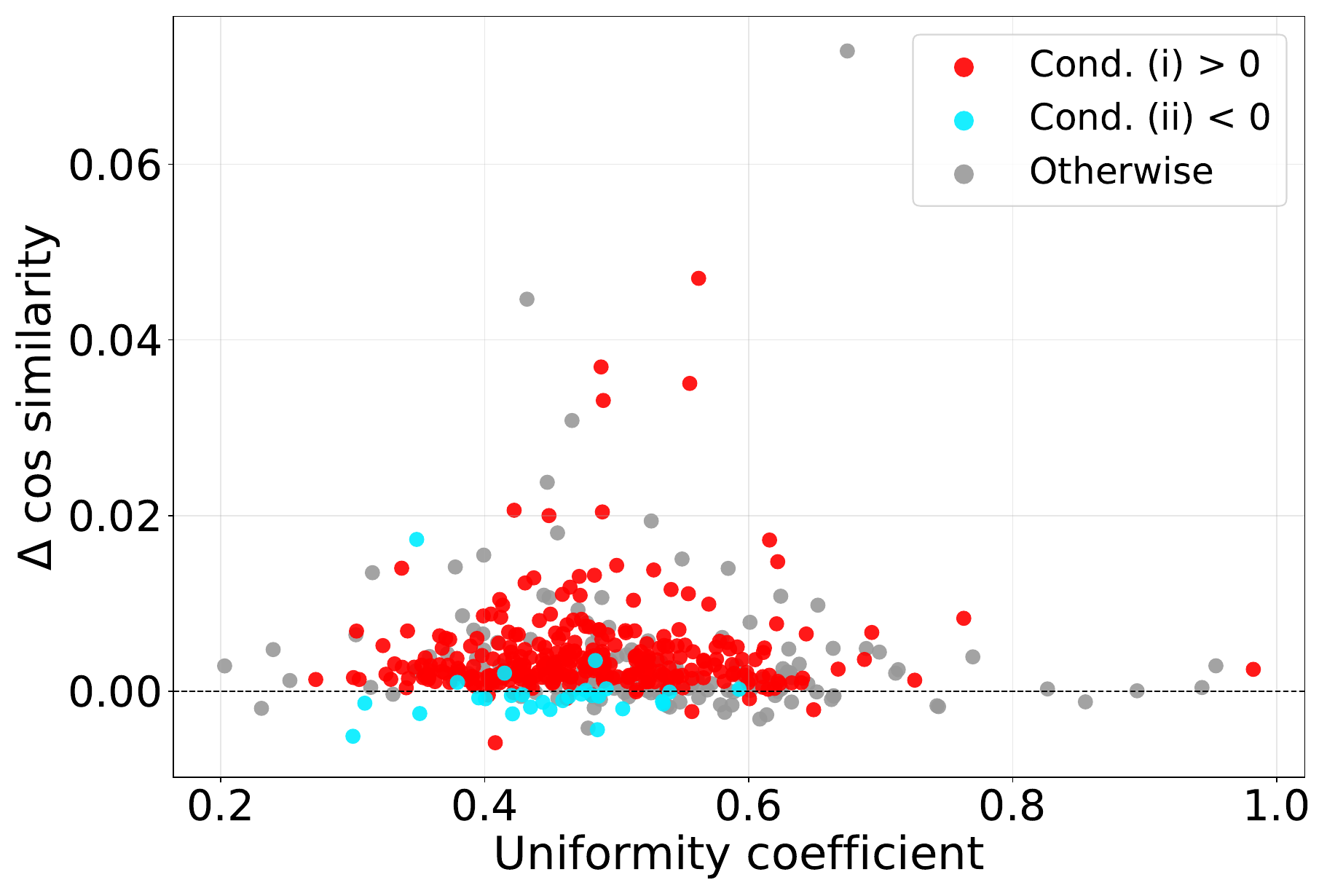}
    \end{subfigure}

    \caption{\(\operatorname{avg.\ cos\ sim}(X+WZA_u^\top)-\operatorname{avg.\ cos\ sim}(X+WZA^\top)\) for each head of different models and datasets at context length 512. Heads satisfying \(\beta\Tr(BW)>0\) (red) tend to yield positive values, whereas heads satisfying \(\beta\Tr(BW)+\Tr(BW^\top W)<0\) (blue) tend to have negative values. Heads satisfying neither condition are shown in gray.}
    \label{fig:trace-cond-theorem-2}
\end{figure}

\begin{figure}[t]
    \centering
    \begin{subfigure}[b]{0.49\textwidth}
        \centering
        {\small LLaMA3-8B, C4\par}
        \vspace{2pt}
        \includegraphics[width=\textwidth]{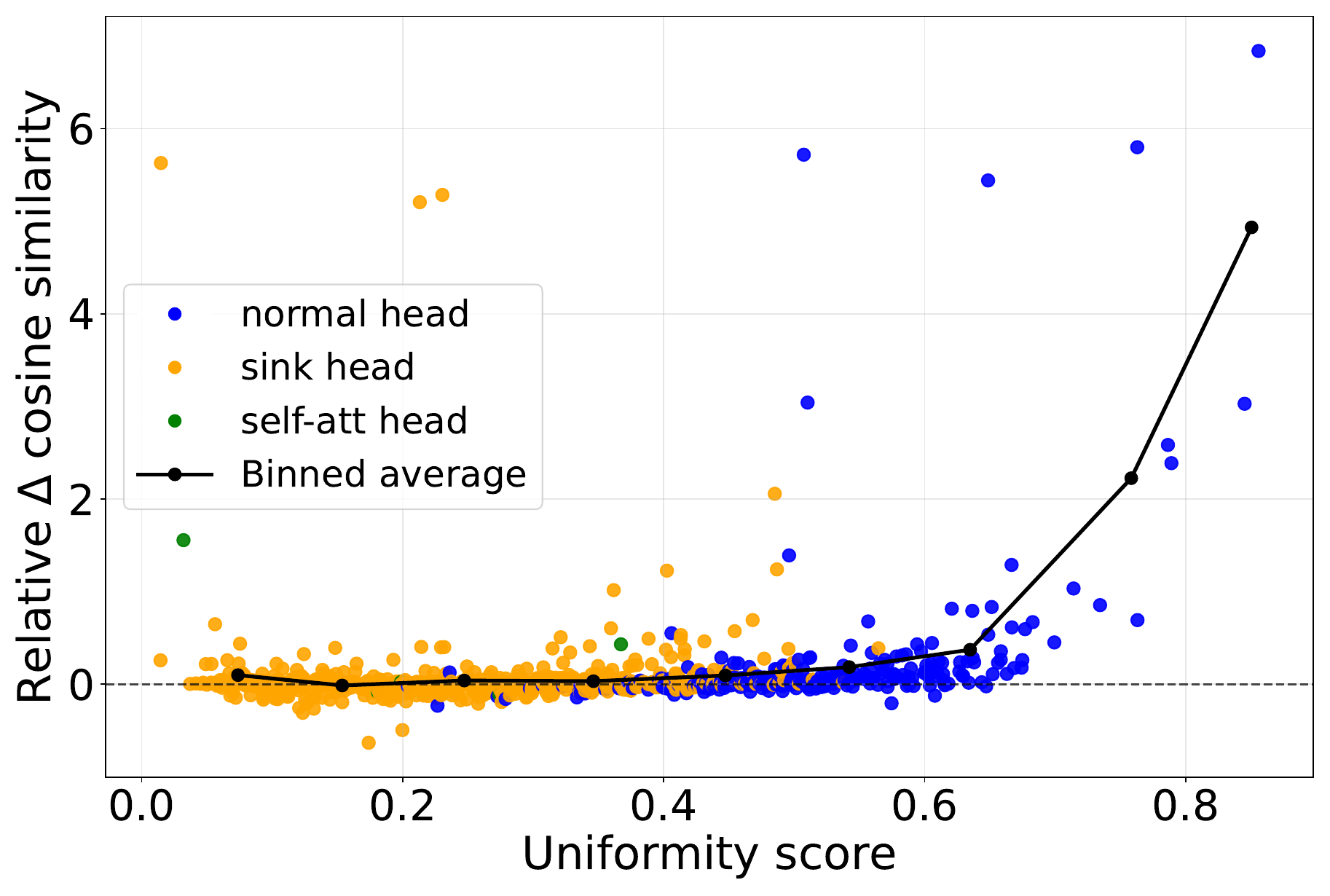}
    \end{subfigure}
    \hfill 
    \begin{subfigure}[b]{0.49\textwidth}
        \centering
        {\small GPT2-XL, CodeParrot\par}
        \vspace{2pt}
        \includegraphics[width=\textwidth]{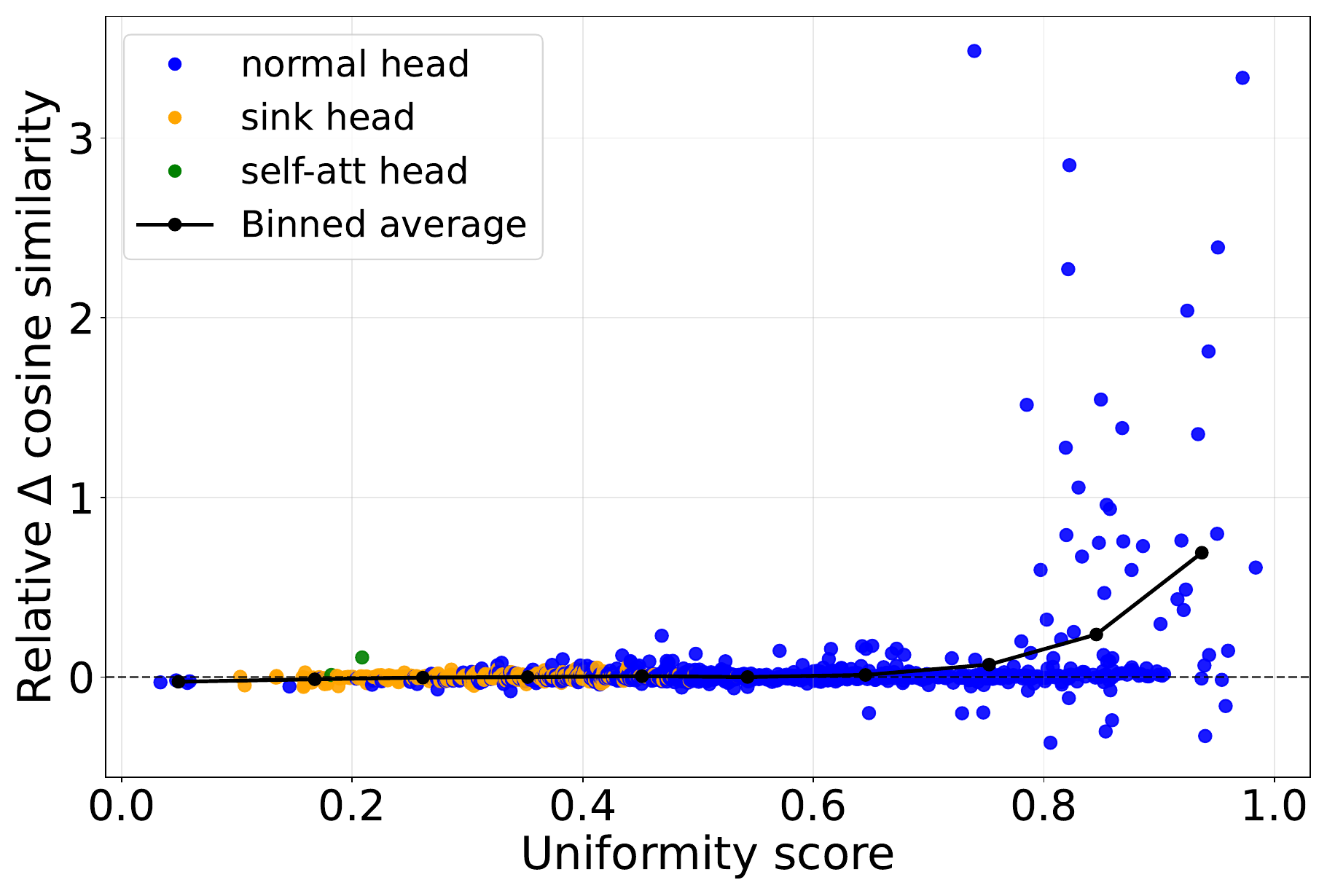}
    \end{subfigure}
     \hfill 
    \begin{subfigure}[b]{0.49\textwidth}
        \centering
        {\small Gemma-7B, WikiText\par}
        \vspace{2pt}
        \includegraphics[width=\textwidth]{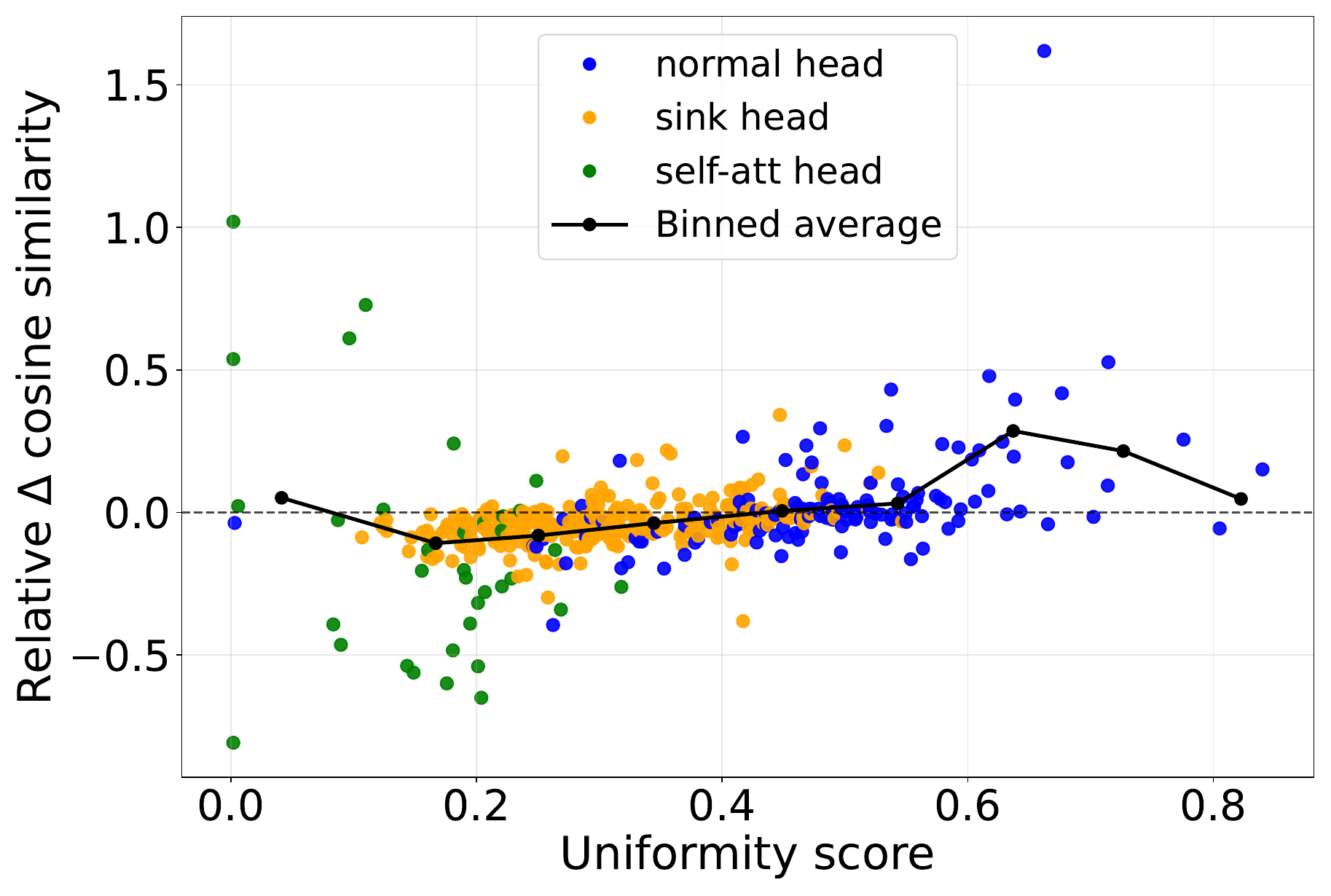}
    \end{subfigure}
    \hfill 
    \begin{subfigure}[b]{0.49\textwidth}
        \centering
        {\small Mistral-7B, C4\par}
        \vspace{2pt}
        \includegraphics[width=\textwidth]{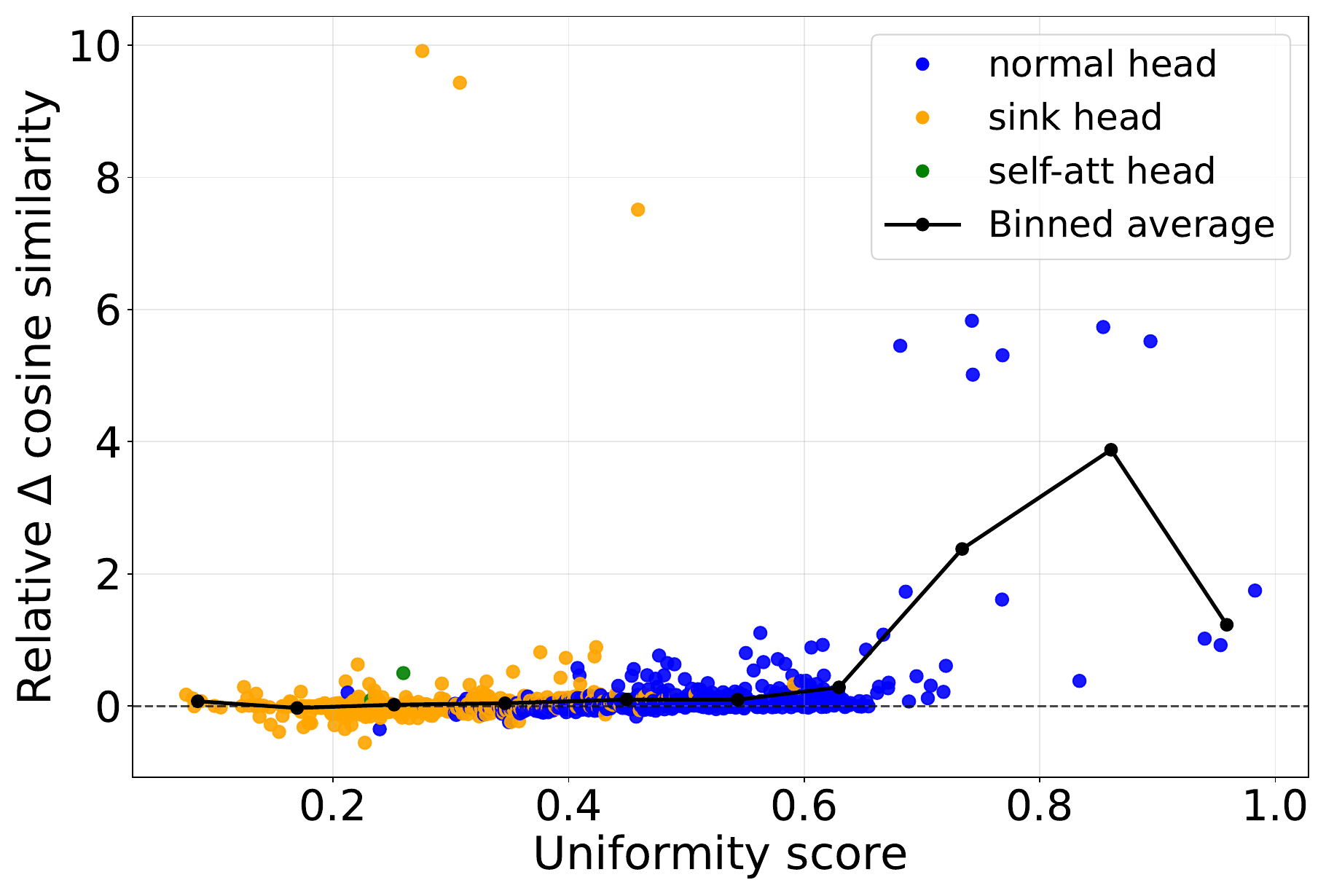}
    \end{subfigure}
\caption{Relative change in average cosine similarity, $(\operatorname{avg.\ cos\ sim}(X+WZA^\top)-\operatorname{avg.\ cos\ sim}(X))/\left|\operatorname{avg.\ cos\ sim}(X)\right|$,
for each head of different models and datasets at context length 512, plotted against the head's uniformity coefficient. Heads are classified as sink heads (yellow), self-attention heads (green), or neither (blue), using a threshold of \(0.5\) for both categories.  The black line shows the average value across heads within bins of width \(0.1\).}
    \label{fig:unif-vs-relChange}
\end{figure}

\begin{figure}[p]
    \centering
    \begin{subfigure}[b]{0.49\textwidth}
        \centering
        {\small LLaMA3-8B, C4, Head 13, Layer 17, Unif. $0.59$\par}
        \vspace{2pt}
        \includegraphics[width=\textwidth]{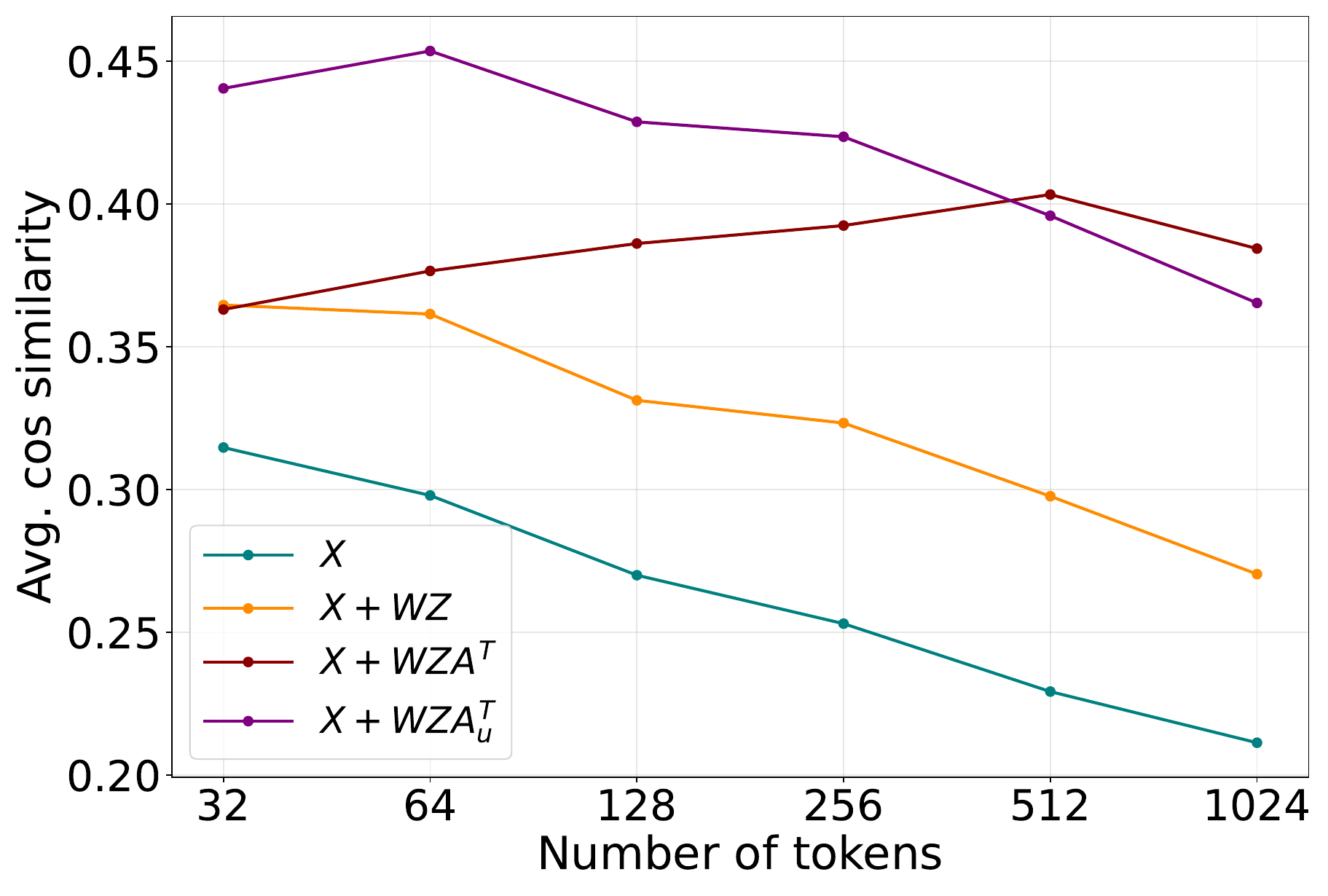}
    \end{subfigure}
    \hfill 
     \begin{subfigure}[b]{0.49\textwidth}
        \centering
        {\small GPT2-XL, CodeParrot, Head 10, Layer 6, Unif. $0.54$\par}
        \vspace{2pt}
        \includegraphics[width=\textwidth]{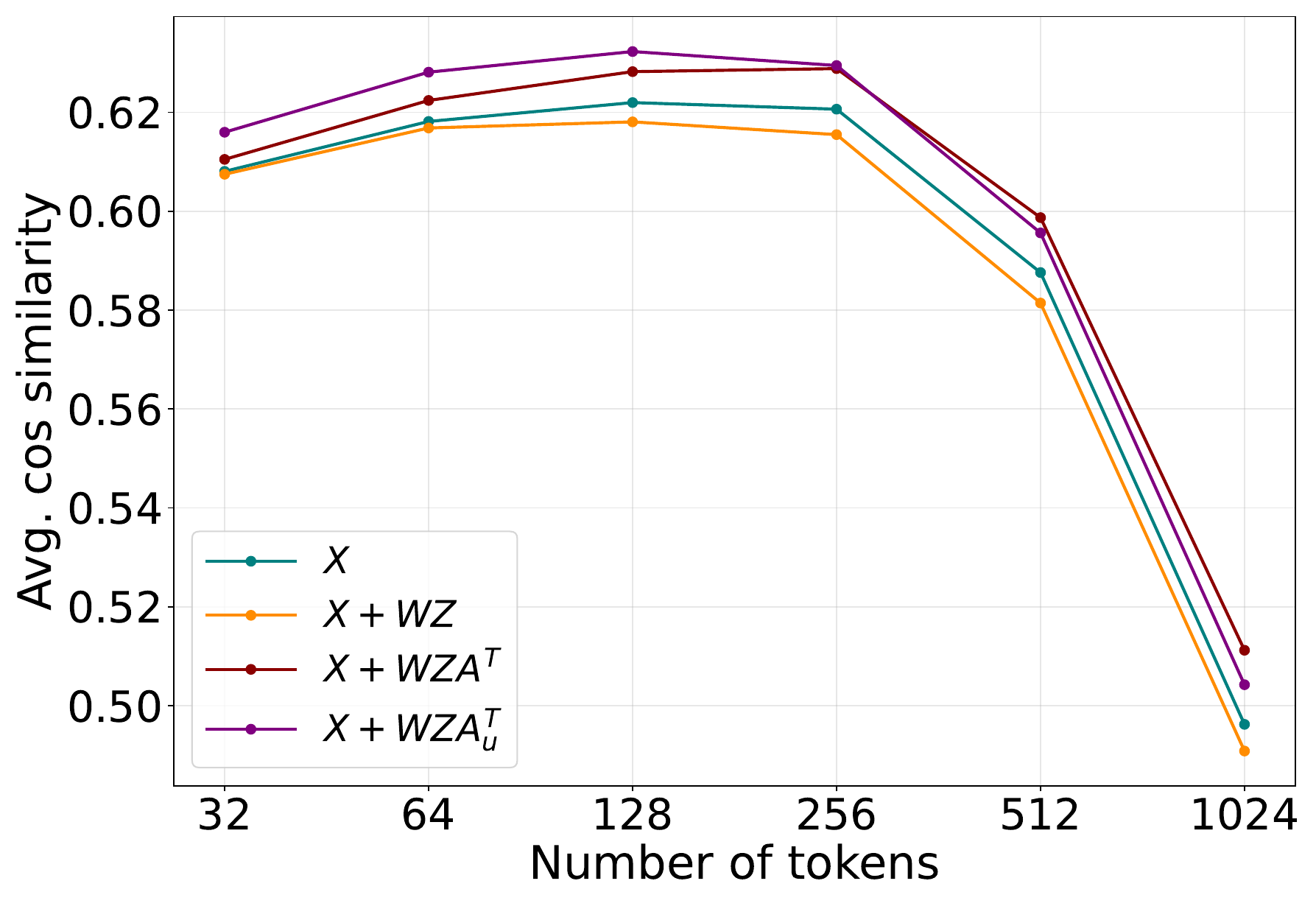}
    \end{subfigure}
     \hfill 
    \begin{subfigure}[b]{0.49\textwidth}
        \centering
        {\small Gemma-7B, WikiText, Head 8, Layer 17, Unif. $0.30$\par}
        \vspace{2pt}
        \includegraphics[width=\textwidth]{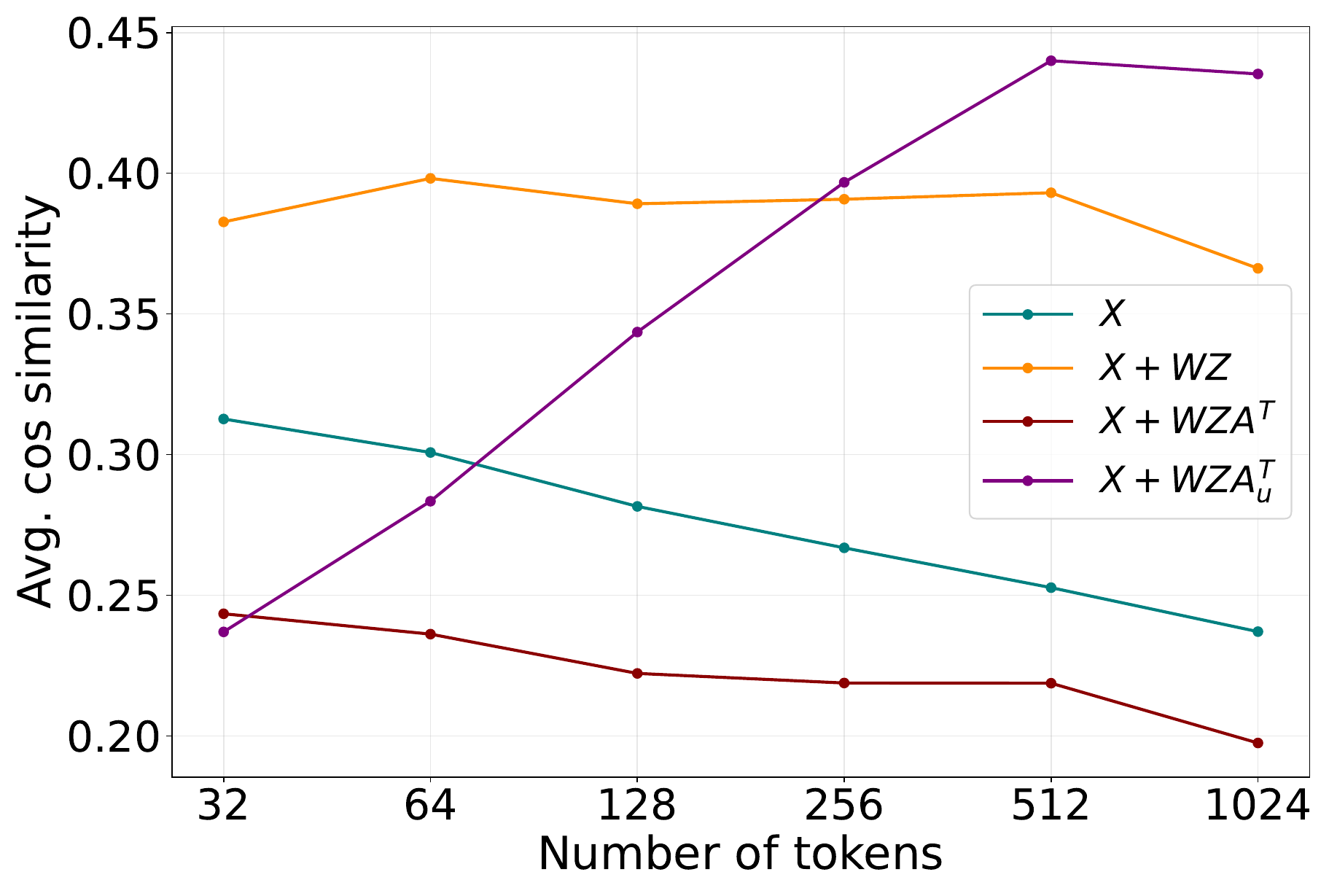}
    \end{subfigure}
    \hfill 
    \begin{subfigure}[b]{0.49\textwidth}
        \centering
        {\small Mistral-7B, C4, Head 24, Layer 30, Unif. $0.27$\par}
        \vspace{2pt}
        \includegraphics[width=\textwidth]{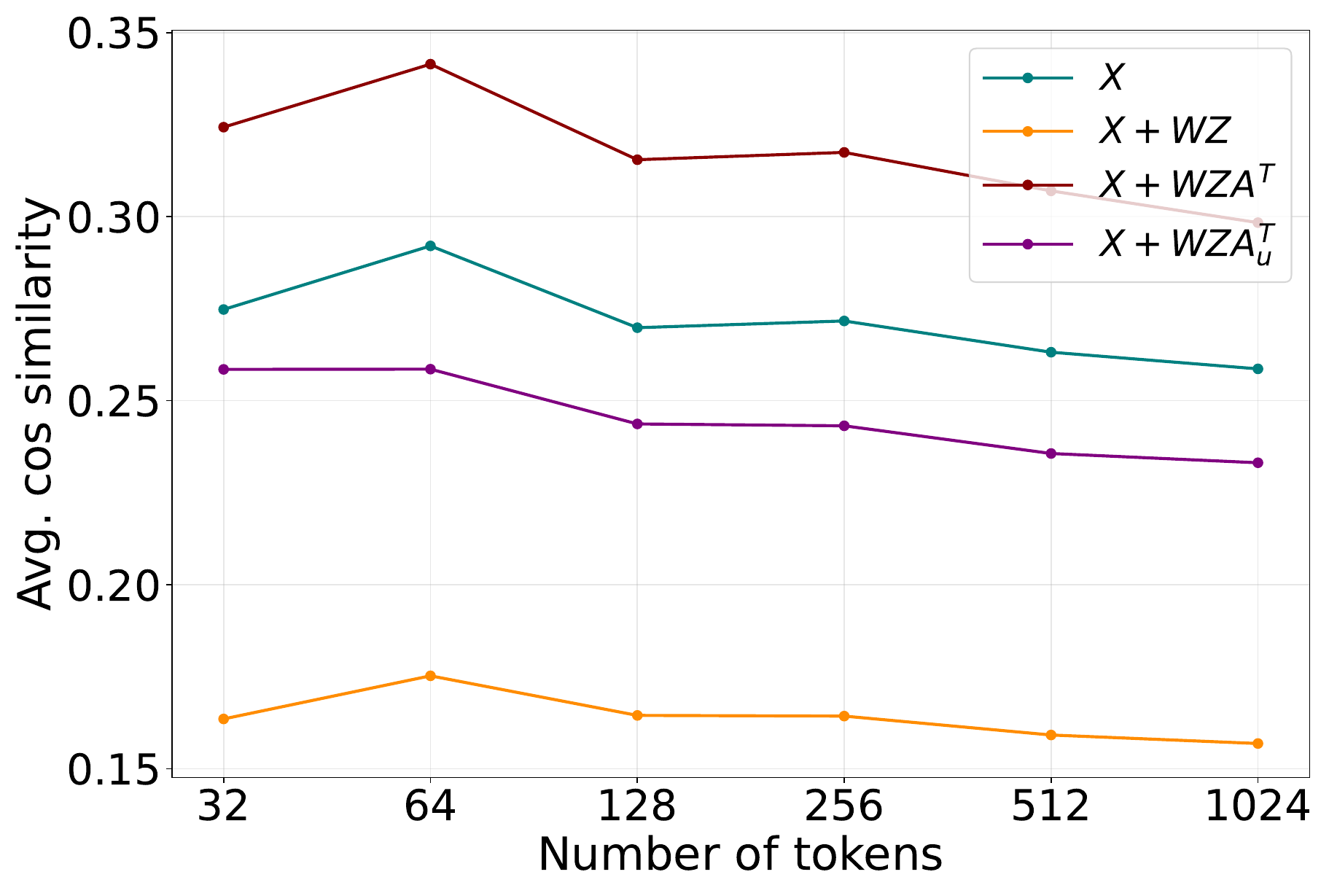}
    \end{subfigure}
\caption{Empirical average cosine similarity for selected heads from different models and datasets, shown for different components of the attention step. These panels illustrate heads that deviate from the typical behavior described in the main text. The reported uniformity coefficient is averaged across context lengths.}
    \label{fig:uncommon-heads}
\end{figure}

\subsection{Value mapping ranks}\label{appx:exp-value_ranks}
Figure~\ref{fig:value_ranks} shows the distributions of ranks of the matrices of value vectors per layer, pooled across heads, 50 random sequences and four datasets (\texttt{TinyStories} \cite{eldan2023tinystories}, \texttt{tinyshakespeare} \cite{karpathy2015char-rnn}, \texttt{WikiText} \cite{merity2017pointer}, and \texttt{CodeSearchNet-Python} \cite{husain2019codesearchnet}). Instead of raw ranks, we plot the rank of the matrix divided by $\min\{m,n\},$ where $m, n$ are the dimensions of the matrix. For each model, we choose the context length to be as large as the inner dimension of the value mappings. As can be seen from the plots, models exhibit very high ranks in the middle layers. The only dataset on which outliers exist is \texttt{CodeSearchNet-Python} on gpt2, because the samples contain a lot of repeated tokens. Note that Llama and Gemma models are immune to repeated tokens and the ranks do not suffer deficiency even on those samples. 

\begin{figure}[p]
    \centering
    \includegraphics[width=\linewidth]{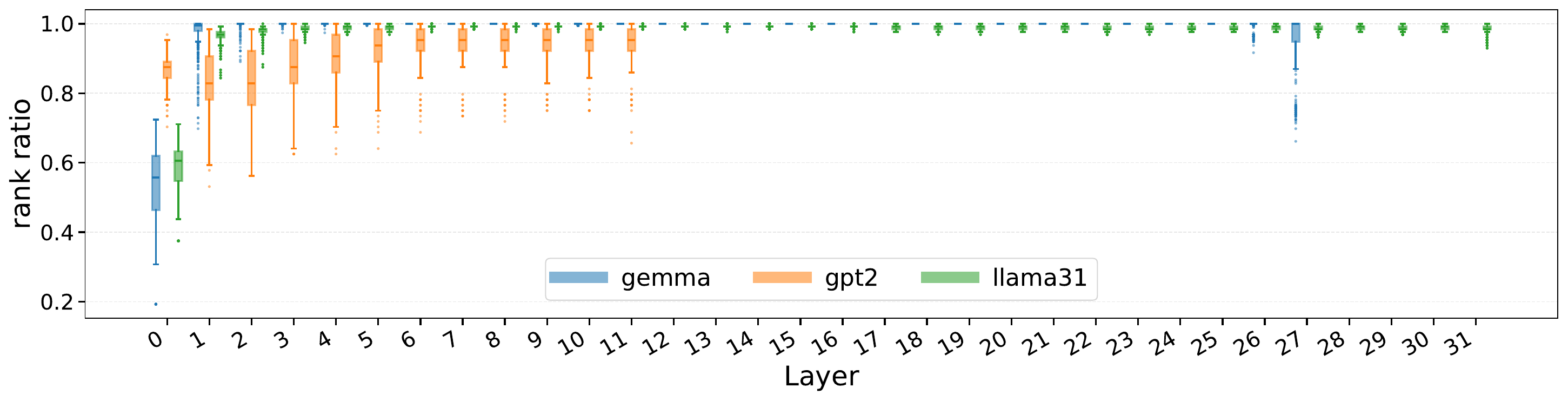}
    \includegraphics[width=\linewidth]{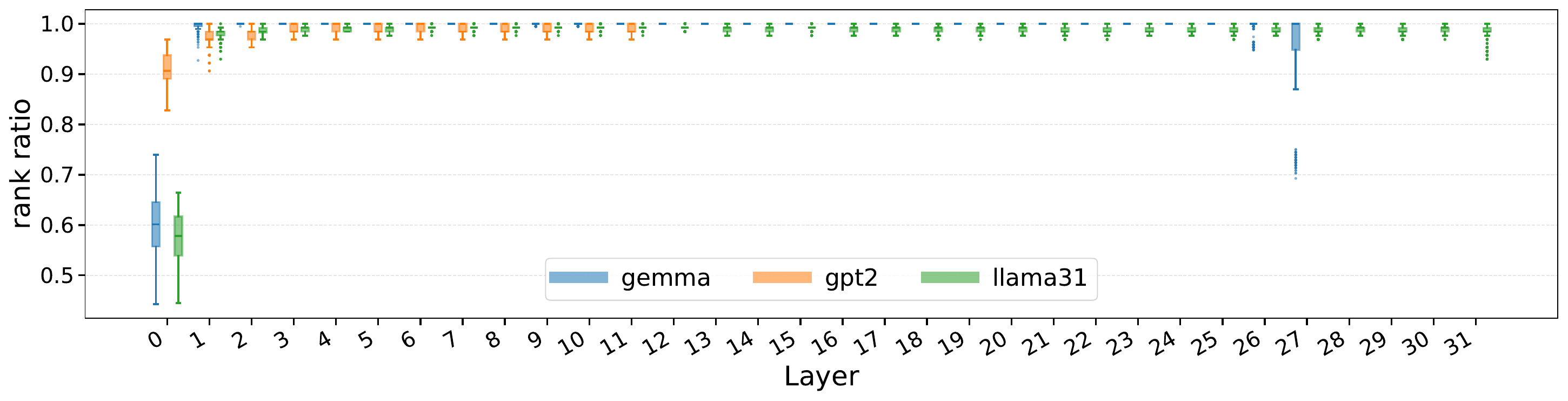}
    \includegraphics[width=\linewidth]{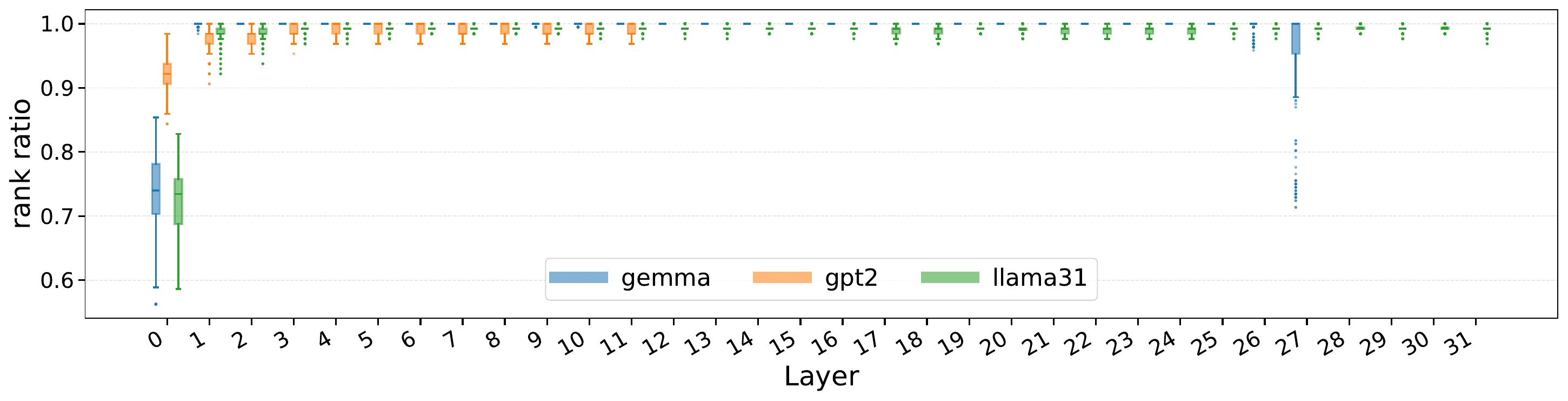}
    \includegraphics[width=\linewidth]{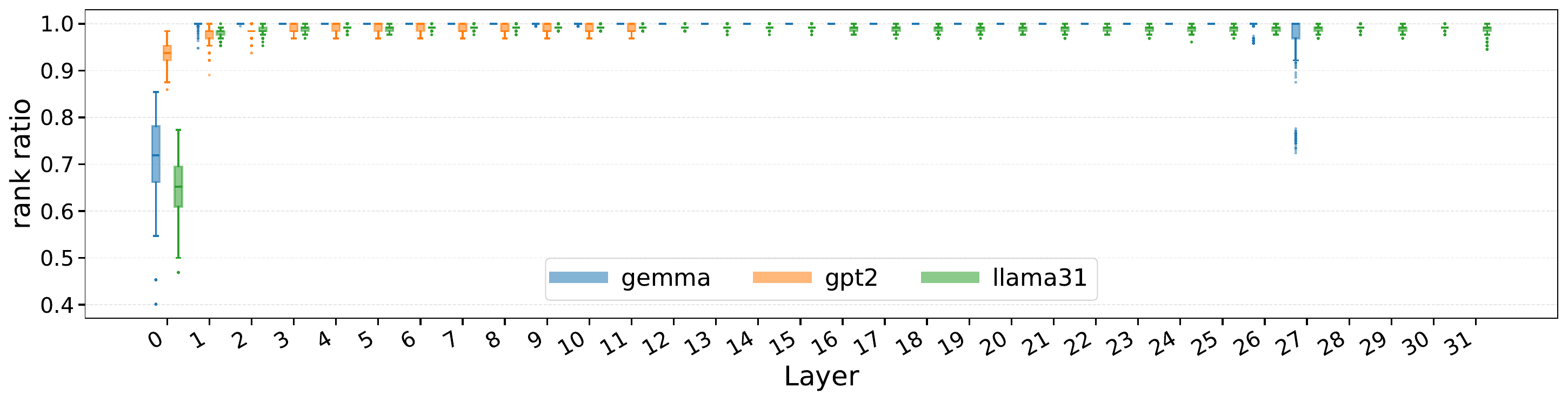}
    \includegraphics[width=\linewidth]{figures/peter/code_search_net_python_value_rank_multi_model_boxplot.pdf}
    \caption{Rank distributions across layers, pooled over heads, 50 sequences and 4 datasets. The whiskers of the box plot point to absolute minimum and maximum. The relative rank tolerance is $0.001.$ \textbf{Upper:} Average for all four datasets. \textbf{Lower four:} Individual results for \texttt{TinyStories}, \texttt{tinyshakespeare}, \texttt{WikiText} and  \texttt{CodeSearchNet-Python} (in this order).}
    \label{fig:value_ranks}
\end{figure}

\subsection{Attention heads following the patterns from Section~\ref{sec:sink_diag}}\label{app:patterns}

Figure~\ref{fig:sink_diag_attention_maps} displays heads whose patterns follow the assumptions in Theorems~\ref{thm:backcopy-frame} and~\ref{thm:sink_diag_generic}. As a source, we use the online tool \texttt{https://attention-motifs.github.io/v2/index.html} developed in \cite{ivanitskiy2025motifs}. The sequence we tested this on is: 

\textit{love planning. It is one of my obsessions. After making numerous planning mistakes, however, I would like to share with you some tips I’ve learned along the way. If you are planning your child’s. One of my favorite drinks when I return home is plantation iced tea. Last year, when I spent a week back home in Hawaii, I ended up drinking it as often as I found it on the menu. Now that the heat of summer is here, I’m dreaming of returning home to visit again. I really […]}

We note that the variance of the type of pattern as a function of the sequence is practically non-existent.

\begin{figure}
    \centering
    \includegraphics[width=0.35\linewidth]{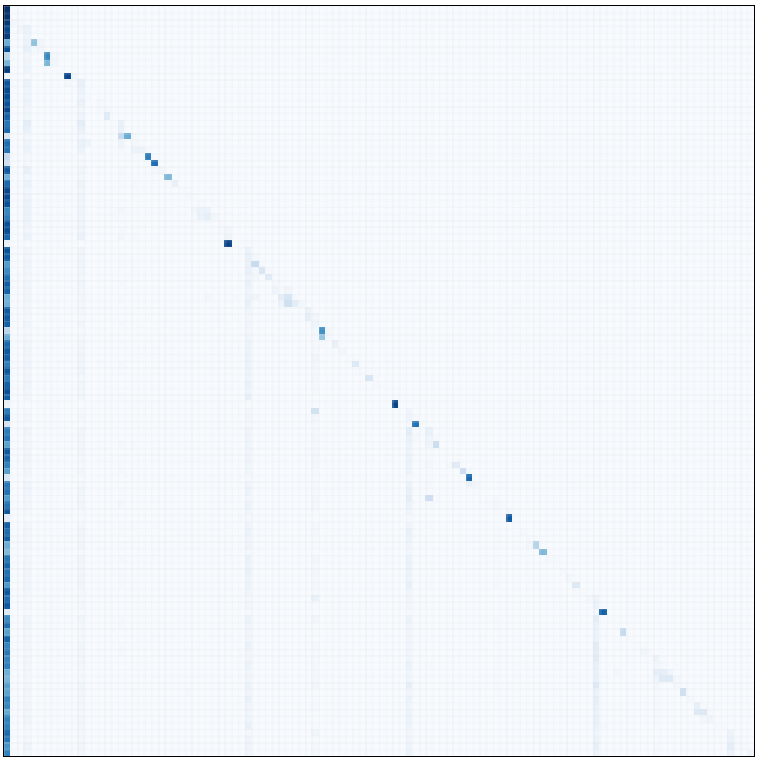}
    \includegraphics[width=0.35\linewidth]{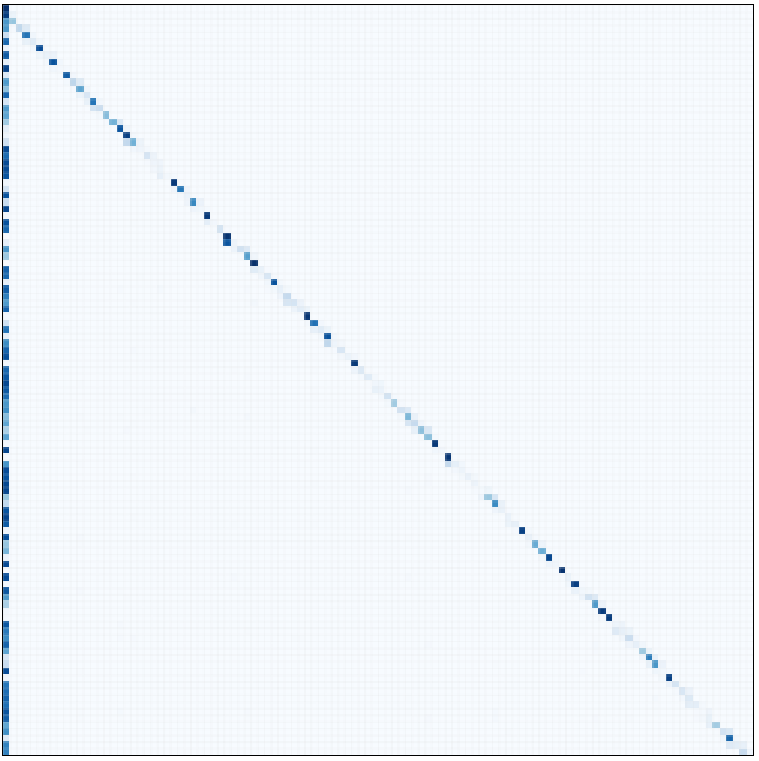}
    \includegraphics[width=0.35\linewidth]{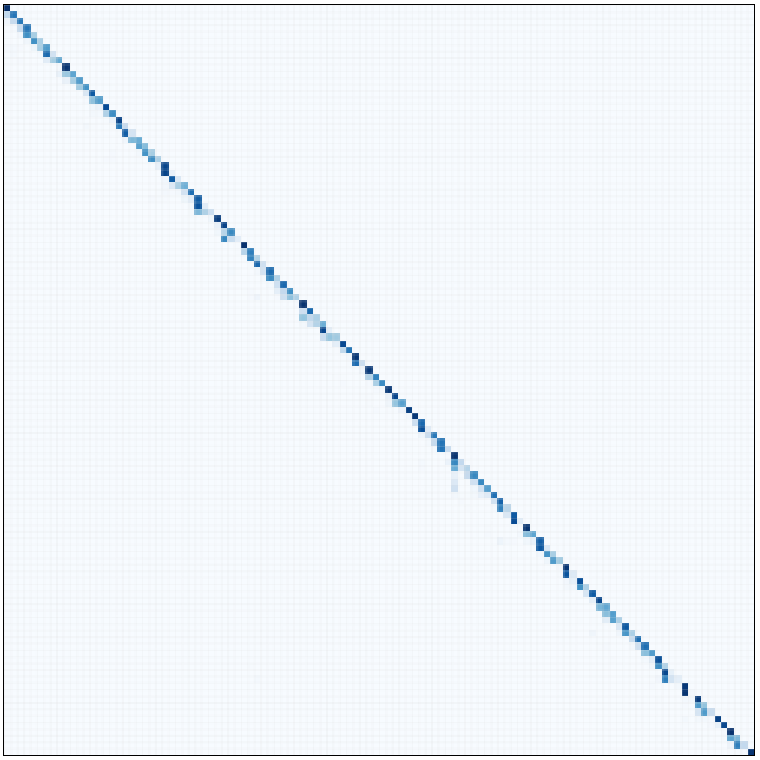}
    \includegraphics[width=0.35\linewidth]{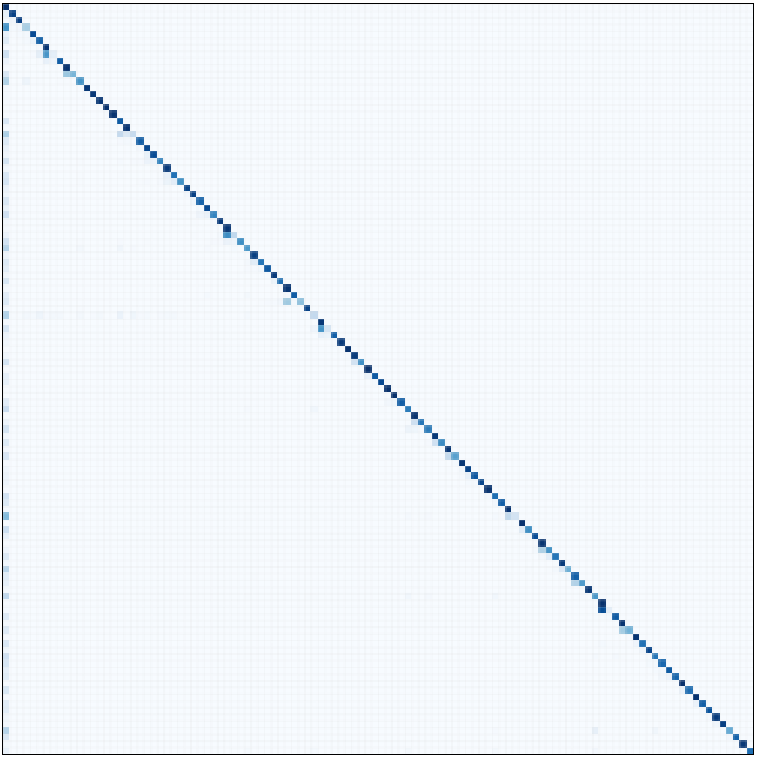}
    \includegraphics[width=0.35\linewidth]{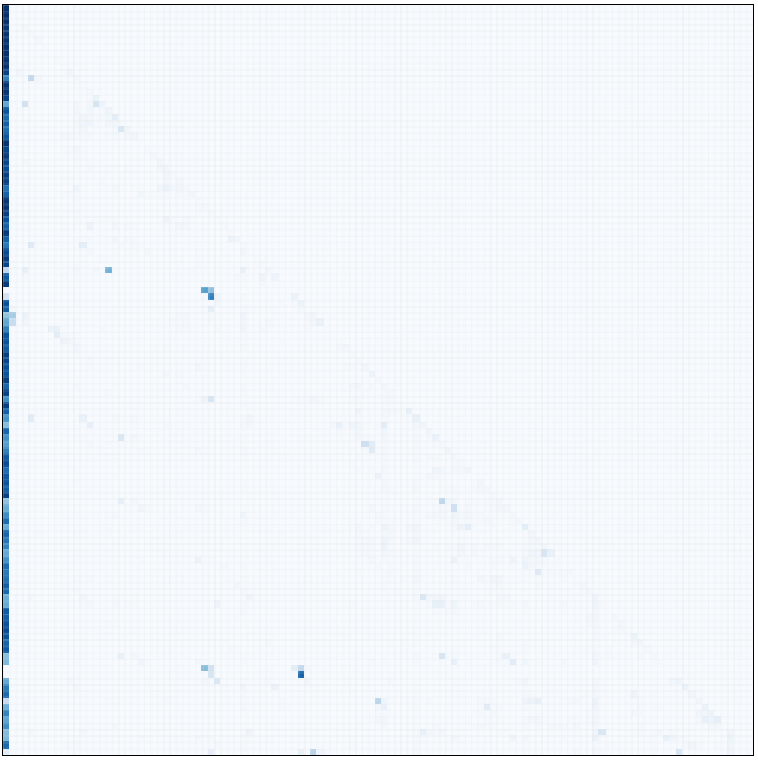}
    \includegraphics[width=0.35\linewidth]{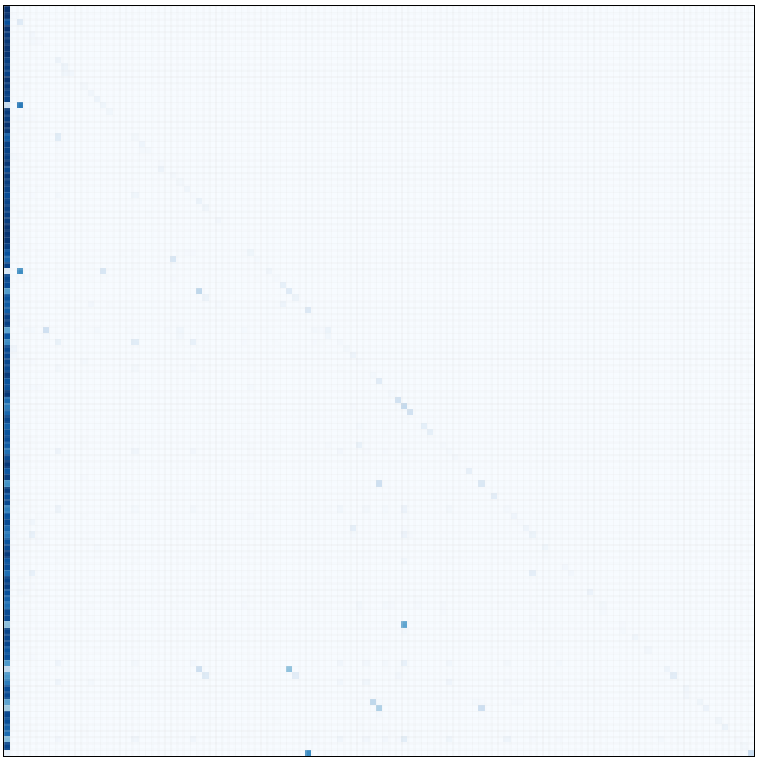}
    \includegraphics[width=0.35\linewidth]{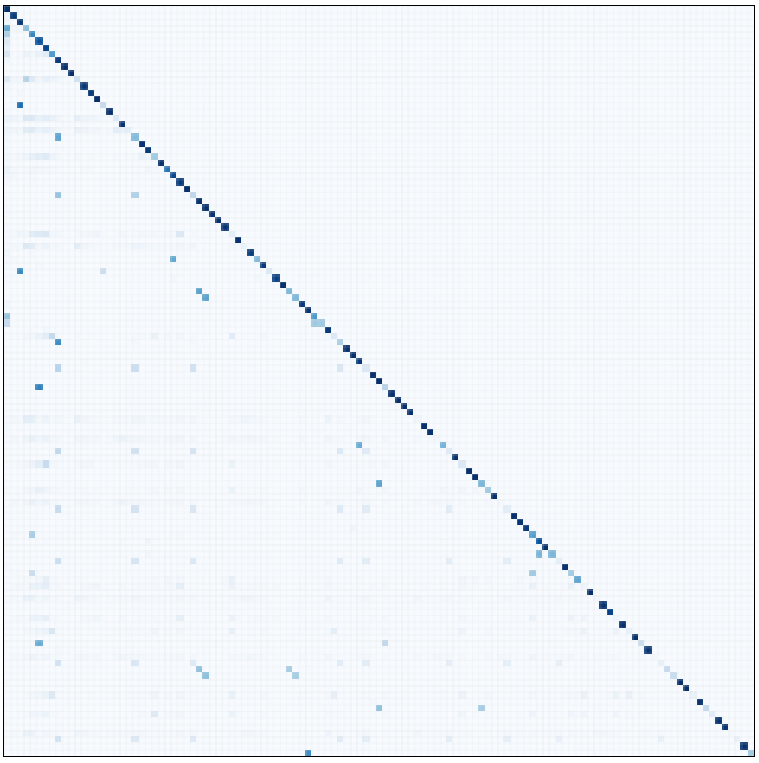}
    \includegraphics[width=0.35\linewidth]{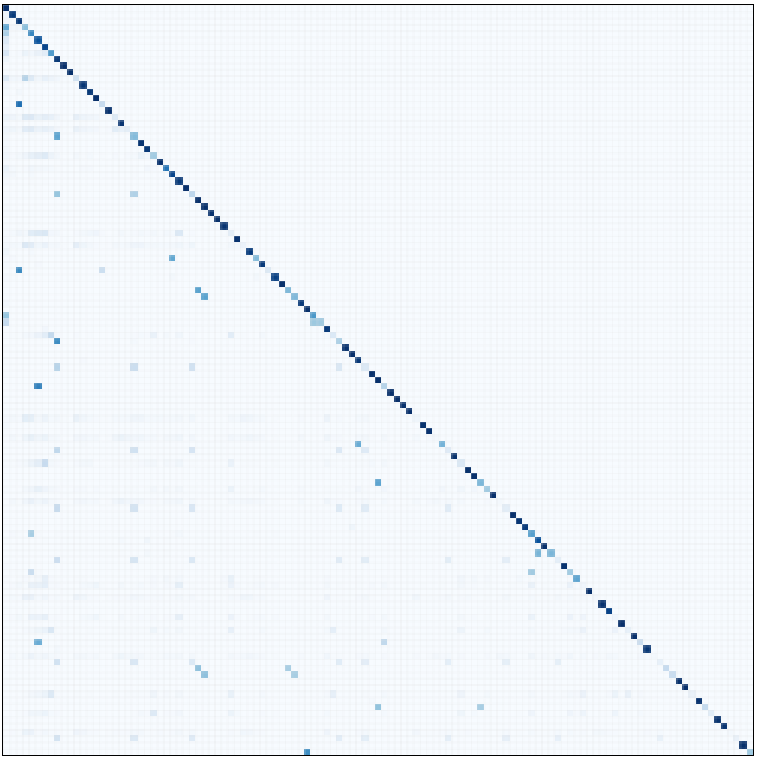}
    \caption{Attention patterns in pretrained transformers (taken randomly) from a pool of models available at \texttt{https://attention-motifs.github.io/v2/index.html}. The patterns satisfy the description of our distributions as copy-paste distribution sink (first row), copy-paste distribution diagonal (second row), copy clusters sink (Theorem~\ref{thm:sink_diag_generic}) (third row) and copy clusters diagonal (fourth row).}
    \label{fig:sink_diag_attention_maps}
\end{figure}




\end{document}